\documentclass[american,letterpaper]{article}

\usepackage[utf8]{inputenc}
\usepackage[top=3cm,bottom=3cm,left=3cm,right=3cm,marginparwidth=1.75cm]{geometry}
\usepackage{graphicx}
\usepackage{amstext}
\usepackage{amsfonts,amsthm}
\usepackage{hhline}
\usepackage{wrapfig}
\usepackage{amssymb}
\usepackage{xcolor}

\usepackage{fullpage}

\usepackage[affil-it]{authblk}

\usepackage{ulem}
\usepackage{footmisc}
\usepackage{mathtools}
\usepackage{bm}
\usepackage{algorithm}
\usepackage{algpseudocode}
\RequirePackage[colorlinks,citecolor={blue!60!black},urlcolor={blue!70!black},linkcolor={red!60!black},breaklinks,hypertexnames=false]{hyperref}
\usepackage{mathrsfs}
\usepackage{amsmath}
\usepackage{thmtools,thm-restate}
\usepackage{float}
\usepackage{verbatim}
\usepackage{cleveref}
\usepackage{booktabs}
\usepackage{multirow}
\usepackage{enumitem}
\usepackage{setspace}
\newtheorem{theorem}{Theorem}
\newtheorem{proposition}{Proposition}
\newtheorem{lemma}{Lemma}
\newtheorem{corollary}{Corollary}

\newtheorem{assumption}{Assumption}
\newtheorem{remark}{Remark}

\newcommand{\customparagraph}[1]{\noindent\textbf{#1:}}

\newcommand{\norm}[1]{\lVert#1\rVert}

\newcommand{\E}{{\rm I}\kern-0.18em{\rm E}}
\newcommand{\p}{{\rm I}\kern-0.18em{\rm P}}
\newcommand{\R}{\mathbb{R}}
\newcommand{\B}{\boldsymbol}
\newcommand{\CS}{S}
\newcommand{\CN}{\mathcal{N}}

\newcommand{\CE}{\mathcal{E}}
\newcommand{\CR}{\mathcal{L}}
\newcommand{\1}{{\rm 1}\kern-0.24em{\rm I}}
\newcommand{\btX}{\tilde{\bm X}}
\newcommand{\btx}{\tilde{\bm x}}
\newcommand{\BS}{\mathbb{S}}

\newcommand{\ourmethod}{\texttt{GraphL0BnB}}
\newcommand{\ourmethodnobnb}{\texttt{GraphL0}}

\crefalias{prop}{proposition}
\newcommand{\prox}[1]{\mathrm{prox}_{#1}}

\newcommand{\schur}[3]{{#1}/ [{#2},{#3}]}
\DeclareMathOperator*{\sign}{sign}

\usepackage{natbib}
 \bibpunct[, ]{(}{)}{,}{a}{}{,}
 
 \def\bibsep{\smallskipamount}

\title{Sparse Gaussian Graphical Models with Discrete Optimization: Computational and Statistical Perspectives}

\author[1]{Kayhan Behdin\thanks{behdink@mit.edu}}
\author[1]{Wenyu Chen\thanks{wenyu.main@gmail.com}}
\author[1,2]{Rahul Mazumder\thanks{rahulmaz@mit.edu}}

\affil[1]{Operations Research Center, Massachusetts Institute of Technology}
\affil[2]{Sloan School of Management, Massachusetts Institute of Technology}
\date{}

\begin{document}

\maketitle

\begin{abstract}
We consider the problem of learning a sparse graph underlying an undirected Gaussian graphical model, a key problem in statistical machine learning. Given $n$ samples from a multivariate Gaussian distribution with $p$ variables, the goal is to estimate the $p \times p$ inverse covariance matrix (aka precision matrix), assuming it is sparse (i.e., has a few nonzero entries). We propose GraphL0BnB, a new estimator based on an $\ell_0$-penalized version of the pseudo-likelihood function, while most earlier approaches are based on the $\ell_1$-relaxation.  Our estimator can be formulated as a convex mixed integer program (MIP) which can be difficult to compute beyond $p\approx 100$ using off-the-shelf commercial solvers. To solve the MIP, we propose a custom nonlinear branch-and-bound (BnB) framework that solves node relaxations with tailored first-order methods. As a key component of our BnB framework, we propose large-scale solvers for obtaining good primal solutions that are of independent interest. We derive novel statistical guarantees (estimation and variable selection) for our estimator and discuss how our approach improves upon existing estimators. Our numerical experiments on real and synthetic datasets suggest that our BnB framework offers significant advantages over off-the-shelf commercial solvers, and our approach has favorable performance (both in terms of runtime and statistical performance) compared to the state-of-the-art approaches for learning sparse graphical models.
\end{abstract}
\section{Introduction}
Gaussian Graphical Models (GGM), due to~\citet{dempster}, are amongst the most widely used tools in multivariate statistics and machine learning (\citet[Chapter 17]{friedman2017elements} and~\citet[Chapter 11]{wainwright2019high}). 
Formally, in a GGM, we are given $n$ samples $\B{x}^{(1)},\cdots,\B{x}^{(n)}\in\R^p$ from a multivariate normal distribution ${\mathcal N}(\B{0},\B{\Sigma}^*)$
where $\B{\Sigma}^*$ is an unknown $p \times p$ positive definite matrix.
Our goal is to estimate the inverse of the covariance matrix $\B{\Sigma}^*$, known as the precision matrix and denoted as $\B{\Theta}^*$. 
Obtaining a sparse estimate of $\B{\Theta}^*$ (i.e., one with only a few nonzero coordinates) is an important methodological problem 
with an array of applications~\citep{friedman2017elements,wainwright2019high}, and has
garnered significant attention in statistical learning.
Particularly, a zero entry in $\B\Theta^*$ indicates conditional independence: For a pair $(i,j)$, having $\theta^*_{ij} = 0$ implies features $i, j$ are independent when conditioning on the other variables.
Our goal is to estimate a sparse precision matrix $\hat{\B{\Theta}}$ (say), such that it is close to the true precision matrix $\B\Theta^*$ in a suitable metric, as discussed below.

The topic of sparse GGMs is quite vast---we first present an overview of some well-known algorithms and then
summarize our key contributions in this paper.

\subsection{Background and Literature Review}
Numerous algorithms have been proposed for sparse GGMs. Generally, these methods aim to minimize a regularized loss function, where the regularization term encourages sparsity in the precision matrix estimate. A popular approach involves the minimization of an $\ell_1$-regularized negative log-likelihood function, known as Graphical Lasso~\citep{friedman2008sparse}. Graphical Lasso given by a convex semidefinite program enjoys good statistical and computational properties~\citep{ravikumar2008model,mazumder2012graphical}. CLIME~\citep{cai2011constrained} is another approach with strong theoretical underpinnings: it is based on constrained $\ell_1$-norm minimization and is given by a linear program.
Another well-known approach is the node-wise $\ell_1$-regularized regression framework of~\citet{meinshausenbuhlmann}, which involves solving $p$ separate Lasso regression problems. 

The current paper focuses on a pseudo-likelihood based approach for sparse GGMs. 
The pseudo-likelihood approach with origins in spatial analysis~\citep{besag1975statistical} 
approximates the Gaussian likelihood by the  product of conditional likelihood functions of each variable, given the rest. In an early work on sparse GGMs~\citep{peng2009partial}, the authors explored a pseudo-likelihood framework using $\ell_1$ regularization (to promote sparsity), but their estimator involves solving a non-convex optimization problem. They show their method performs well numerically and present an asymptotic analysis of their algorithm when $n,p\to\infty$. 
Symmetric Lasso~\citep{friedman2010applications} and CONCORD~\citep{khare2015convex} are other algorithms based on pseudo-likelihood with strong empirical performance.

A fairly recent and promising, though less explored approach to GGMs is based on discrete optimization. Since the work of~\citet{bertsimas2016best} on subset selection in linear regression, there has been considerable interest in exploring statistical problems with a combinatorial structure using tools from Mixed Integer Programming (MIP)~\citep{nemhauser} and relatives. 
Specialized algorithms have been recently explored to address MIP-based statistical problems in
sparse linear regression~\citep{vanparys,hazimeh2020sparse,hazimeh2020fast,mazumder2017subset}, sparse principal component analysis~(\citet{dey2022using,behdin2021integer}; see also references therein), among others. In contrast, the literature on using MIP approaches for sparse GGMs, remains relatively less explored. \citet{bertsimas2020certifiably} consider a MIP approach for $\ell_0$-constrained maximum likelihood GGM estimation---their specialized algorithm can address problems with $p\approx 100$.  Another approach is the node-wise procedure of~\citet{misra2020information}, which requires solving $p$-many $\ell_0$ regularized linear regression problems. 
Recently, \citet{fattahigomez} explore $\ell_0$ regularization for sparse GGMs in the context of time-series problems.

We mention some existing results on statistical properties of sparse precision matrix estimation. 
Let $k$ denote an upper bound on the number of nonzero coordinates in each row/column of $\B\Theta^*$. To have a consistent estimate of $\B\Theta^*$ (in terms of the Frobenius norm of the estimation error)
we need\footnote{We use the notation $\lesssim,\gtrsim$ to show an inequality holds up to a universal constant that does not depend upon problem data.} $n\gtrsim kp\log p$ samples~\citep{rothman2008sparse}. 
In the high-dimensional setting where $p\gg n$, one is often interested in estimating the true support of $\B\Theta^*$ with high probability. Under certain non-degeneracy conditions, $n\gtrsim k\log p$ samples are required for a consistent estimation of the support of $\B{\Theta}^*$~\citep{wang2010information}.

\subsection{Outline of our Approach and Contributions} 
We propose a new $\ell_0$-regularized pseudo-likelihood-based estimator, \ourmethod, with good statistical guarantees and computational performance.   
Our estimator is based on a MIP: it can be written as minimizing a convex objective function over a mixed integer second-order cone.  As a result, commercial solvers such as Mosek can be used to solve the problem for small-scale instances $p\leq 100$. We propose specialized exact (and approximate) algorithms for improved computational scalability for our estimator. In addition to computation, we study the statistical properties of the estimator as outlined below.

\noindent \textbf{Optimization Algorithms:} We propose and implement (i) approximate methods, to obtain high-quality feasible solutions quickly (ii) globally optimal methods based on a specialized nonlinear Branch-and-Bound (BnB) solver. Our standalone BnB solver does not rely on commercial MIP solvers. Our BnB framework provides valid lower bounds and upper bounds for the optimal solution to our MIP. Even if we are to terminate the BnB process early with a compute budget, we are still able to obtain feasible solutions with suitable optimality certificates.

We note that the objective function that we are dealing with involves logarithmic and quadratic-over-linear terms (see Section~\ref{sec:proposedest}). This requires proposing new algorithms for solving node relaxation and obtaining incumbents.\footnote{An incumbent here refers to the best integral solution found so far during the BnB procedure.} We also establish novel convergence guarantees for our algorithms 
that extend existing results. Our node relaxation solver uses cyclic coordinate descent~\citep{tseng2001convergence} along with active set updates for computational efficiency. We discuss methods to efficiently generate dual bounds, which are important for our BnB method. 
Our approximate algorithms for primal solutions to the MIP extend the work of~\citet{hazimeh2020fast} by making use of a coordinate descent procedure with local search on the mixed integer program. Such solvers are of independent interest, and additionally, they can play an important role in our BnB framework (e.g., by providing good feasible solutions). Our BnB framework is inspired by the work of~\citet{hazimeh2020sparse}
both proposed for sparse linear regression. 
We note that the specific structure of our objective function and the problem scale present technical difficulties, making our GGM approach different from earlier work.

\noindent \textbf{Statistical Properties:} We study both the estimation and variable selection properties of our proposed estimator. We show that our estimator has an estimation error (Frobenius norm) bound scaling as $\sqrt{kp\log p/n}$, where $k$ is an upper bound on the total number of nonzero coordinates in each row/column of $\B\Theta^*$. In terms of variable selection, we show that under certain regularity conditions, if $n\gtrsim k\log p$, our estimator is able to recover the support of $\B{\Theta}^*$ correctly with high probability. The non-degeneracy condition needed for consistent variable selection for our method is milder than the earlier ones. This is due to certain symmetry structures we enforce on the precision matrix as a part of our estimation criterion. Our non-asymptotic estimation error bounds and support recovery guarantees are a novel contribution in the context of pseudo-likelihood-based sparse GGMs. Moreover, due to the specific structures of our problem, most earlier proof techniques developed do not apply to our estimator directly, and we develop new techniques for our analysis.

\noindent \textbf{Numerical Results: } We compare our approach with other existing methods in terms of both statistical performance and computational efficiency on both synthetic and real datasets. 
For some problem instances with $p=10,000$, our approximate algorithms can compute solutions in approx. 2-3 minutes, while earlier pseudo-likelihood-based estimators appear to be limited to instances with $p \leq 3000$ or so. Our BnB solver can 
solve, with optimality certificates, problem instances with $p\approx 5000$, $n\approx 1000$ (with around $p^2/2\approx 12\times 10^6$ binary variables due to symmetry) when the optimal solution is sufficiently sparse in less than an hour. On the other hand, leading off-the-shelf solvers such as Mosek are limited to $p\approx 100$. We see in our numerical experiments, our BnB framework with early termination often improves
the initial incumbent and delivers solutions with  better statistical performance---this shows the promise of using our BnB solver (with early termination) to obtain high-quality primal solutions.
Moreover, we observe that \ourmethod~enjoys better statistical performance (estimation and variable selection) on synthetic and real datasets compared to popularly used $\ell_1$-based methods such as CLIME and Graphical Lasso.

Our contributions in this paper can be summarized as follows:

\begin{enumerate}
	\item We propose an $\ell_0$-regularized pseudo-likelihood estimator \ourmethod~for sparse GGMs. Our MIP-based estimator can be formulated as minimizing a convex objective with mixed integer second order conic constraints.
	\item We propose a custom branch-and-bound (BnB) method for the MIP.
	Our open-source BnB solver can solve (with optimality certificates) certain problem-instances with $p\approx 5,000$ and  $n\approx 1,000$ (involving $p \times p$ precision matrices) in less than an hour. As a by-product of our framework, we also propose new approximate algorithms that can be much faster than the optimal methods, scaling to $p\approx 10,000$ in a few mins.   
	\item We derive novel statistical (both estimation and variable selection) guarantees for our estimator  and discuss how they can improve upon existing estimators.
	\item Numerical experiments on real and synthetic datasets show the promise of \ourmethod~over popular alternatives for sparse GGMs in terms of both runtime and statistical performance. 
\end{enumerate}

\paragraph{Organization of paper} In Section~\ref{sec:proposedest}, we introduce \ourmethod. In Section~\ref{sec:comp}, we provide a computational framework for our proposed estimator. In Section~\ref{sec:stat}, we analyze the statistical properties of our proposed estimator. Section~\ref{sec:expts} presents numerical experiments on both synthetic and real datasets. The derivations and proofs in the computational and statistical parts are deferred to Appendices~\ref{app:comp-proofs} and~\ref{app:stat-proofs}.
\paragraph{Notations.}
For the data matrix $\B{X}\in\R^{n\times p}$, we let $\B{x}_j\in\R^n$ denote the $j$-th column of $\B{X}$ for $j\in[p]$. For $\B A\in\R^{p_1\times p_2}$ and $S_1\subseteq[p_1],S_2\subseteq[p_2]$, denote by $\B{A}_{S_1,S_2}$ the submatrix of $\B{A}$ with rows in $S_1$ and columns in $S_2$.
$\mathcal{B}(p)$ denotes the unit Euclidean ball of dimension $p$.
Let $\BS^p,\mathbb{S}_+^p$ denote the set of symmetric and positive definite matrices in $\R^{p\times p}$, respectively. We let $\chi\{a\in A\}$ denote the characteristic function, i.e. $\chi\{a\in A\}=0$ if $a\in A$; otherwise, $\chi\{a\in A\}=\infty$. We let $\bm1\{a\in A\}$ denote the indicator function:
equals $1$ if $a\in A$; and $0$ otherwise. We let $\B{I}_n\in\R^{n\times n}$ denote the identity matrix of size $n$. We use the notation $\lesssim,\gtrsim$ to show an inequality holds up to a universal constant that does not depend upon problem data. We note that these notations do not represent asymptotic relationships.

\section{Proposed Estimator}\label{sec:proposedest}
Let $\B{X}\in\R^{n\times p}$ be the data matrix where every row is an independent draw from   $\mathcal{N}(\B{0},(\B{\Theta}^*)^{-1})$ for some $\B{\Theta}^*\in\mathbb{S}_+^p$. For every $j\in [p]$, the conditional distribution of the $j$-th variable, given the rest, follows the normal distribution:
\begin{equation}\label{knormaldists}
	\B{x}_j|\left\{\B{x}_i\right\}_{i\neq j}\sim\mathcal{N}\left(\sum_{i\neq j}\beta^*_{ij}\B{x}_i,(\sigma_j^*)^2 \B{I}_n\right)
\end{equation}
where
\begin{equation}\label{lem_reg}
	\beta^*_{ij}=-\frac{{\theta}^*_{ji}}{{\theta}^*_{jj}}~~i\neq j\in[p],~~({\sigma_j^*})^2=\frac{1}{{\theta}^*_{jj}}~~j\in[p].
\end{equation}
Let $P(\B{\mu},\B{\Sigma};\B{x})$ denote the probability density of a multivariate normal distribution with mean $\B{\mu}\in\R^n$ and covariance $\B{\Sigma}\in\R^{n\times n}$. The pseudo-(log)-likelihood function~\citep{besag1975statistical} is given by the sum over $j \in [p]$ of negative log-likelihoods of the conditional distributions in~\eqref{knormaldists}:
\begin{equation}\label{eqn:pseudo-likhd1}
	-\sum_{j=1}^p \log P\left(\sum_{i\neq j}\beta^*_{ij}\B{x}_i,(\sigma_j^*)^2 \B{I}_n;\B{x}_j\right)=\sum_{j=1}^p \left[\log(\sigma_j^*)+\frac{1}{n}\frac{1}{2(\sigma_j^*)^2}\left\Vert \B{x}_j -\sum_{i:i\neq j}\beta^*_{ij}\B{x}_i \right\Vert_2^2\right].
\end{equation}
The pseudo-likelihood can be considered an approximation to the likelihood function, where the distributions given in~\eqref{knormaldists} are assumed to be independent across $j$.
Additionally, from~\eqref{lem_reg}, $\beta^*_{ij}\neq 0$ if and only if $\theta^*_{ij}\neq0$ and as $\B{\Theta}^*$ is sparse, several values of $\{\beta^*_{ij}\}$ are zero. We consider an $\ell_0$-penalized version of the pseudo-likelihood~\eqref{eqn:pseudo-likhd1}:
\begin{subequations}\label{main-pseudo-1}
	\begin{align}
		\min_{\{\beta_{ij}\},\{\sigma_j\}}  & \sum_{j=1}^p \left[\log(\sigma_j)+\frac{1}{n}\frac{1}{2\sigma^2_j}\left\Vert \B{x}_j -\sum_{i:i\neq j}{\beta}_{ij}\B{x}_i \right\Vert_2^2\right]+\lambda_0 \sum_{i,j:i\neq j}\bm1\{\beta_{ij}\neq 0\}  \\
		\text{s.t.} \quad &
		\beta_{ij}\sigma_i^2 = \beta_{ji}\sigma_j^2,~~\beta_{ii}=0,~~i\neq j\label{symconst}
	\end{align}
\end{subequations}
where  $\lambda_0>0$ is the regularization parameter.  
Constraint~\eqref{symconst} enforces a symmetric structure on the matrix $\{\theta_{ij}\}$ based on the fact $\beta_{ij}^*(\sigma_i^*)^2=\beta_{ji}^*(\sigma_j^*)^2$ from~\eqref{lem_reg}. 
The tuning parameter $\lambda_0$ controls the number of nonzero entries in $\{ \beta_{ij}\}$ (equivalently, the number of nonzeros in the precision matrix $\B\Theta$).
We investigate the statistical properties of this estimator in Section~\ref{sec:stat}. In what follows, we present a convex mixed integer formulation of Problem~\eqref{main-pseudo-1}.

\subsection{A convex mixed integer optimization problem}\label{subsec:mio}
Problem~\eqref{main-pseudo-1} in its current form has a non-convex objective function and involves nonlinear constraints. We consider a reformulation
using the variables: $\theta_{jj}=1/\sigma_j^2$ and $\beta_{ij}=-
\theta_{ji}/\theta_{jj}$ --- with this reformulation, the symmetry constraint~\eqref{symconst} simplifies to the matrix $\B{\Theta}$ being symmetric. 
For our optimization formulation, we consider a minor modification of Problem~\eqref{main-pseudo-1} by including an additional squared $\ell_2$ (ridge) regularization term on the off-diagonals of $\B\Theta$. 
This leads to our reformulation of Problem~\eqref{main-pseudo-1} given as:
\begin{equation}\label{eqn:L0L2}
	\min_{\bm\Theta\in\BS^p}~F_{0}(\bm\Theta)=\sum_{i=1}^p\bigl(-\log(\theta_{ii})+\frac1{\theta_{ii}}\norm{\btX\bm\theta_{i}}^2\bigr)+\sum_{i<j}\bigl(\lambda_0\bm1\{\theta_{ij}\neq 0\}+\lambda_2\theta_{ij}^2\bigr)
\end{equation}
where $\btX =\frac{1}{\sqrt{n}}\B{X}$ and $\lambda_0,\lambda_2\geq 0$ are regularization parameters that are specified a priori. The addition of the ridge penalty helps both in terms of optimization and statistical properties, and is inspired by its usage in earlier work in sparse linear models~\citep{mazumder2017subset,hazimeh2020sparse}. In our numerical experiments, we observe that a nonzero value of $\lambda_2$ can be helpful in terms of statistical performance (see Appendix~\ref{app:l2reg} for a numerical demonstration). Moreover, taking $\lambda_2>0$ enables us to use a perspective formulation. Perspective formulations~\citep{frangioni2006perspective,akturk2009strong,gunluk2010perspective} are favorable from a computational perspective as they result in tighter MIP relaxations, and have been used recently in sparse linear regression~(\citet{hazimeh2020sparse}; see also references therein).
We present a mixed integer formulation of Problem~\eqref{eqn:L0L2}.
To this end, we introduce auxiliary binary variables $\{z_{ij}\}$ that encode sparsity in $\{\theta_{ij}\}$; and consider the following perspective reformulation of Problem~\eqref{eqn:L0L2}:
\begin{align}\label{eqn:mio}
	\min_{\bm\Theta,\bm z,\bm s}&~F_{\mathsf{mio}}(\bm\Theta,\bm z,\bm s)=\sum_{i=1}^p\bigl(-\log(\theta_{ii})+\frac1{\theta_{ii}}\norm{\btX\bm\theta_{i}}^2\bigr)+\sum_{i<j}\bigl(\lambda_0z_{ij}+\lambda_2 s_{ij}\bigr),\\
	\text{s.t.}&~~\theta_{ij}^2 \leq s_{ij}z_{ij},~ |\theta_{ij}|\leq Mz_{ij},~\theta_{ij} = \theta_{ji},~~z_{ij}\in\{0,1\},~s_{ij}\geq 0,\quad \forall j\neq i.\nonumber
\end{align}
Here, we assume that there is a pre-specified positive scalar $M$ (the Big-M parameter), such that there exists an optimal solution $\hat{\B\Theta}$ to \eqref{eqn:L0L2} that satisfies all of its off-diagonal entries have absolute values no than $M$, i.e. for any $i<j\in[p]$, $|\hat\theta_{ij}|\leq M$. 
We note that as long as $\lambda_2>0$, the BnB framework we discuss below can also be applied with $M=\infty$. Additionally, as long as the Big-M value $M$ is finite, our BnB algorithms can be applied for any value of $\lambda_2 \geq 0$.
Our approximate solvers on the other hand, apply directly to formulation~\eqref{eqn:L0L2} for $\lambda_0, \lambda_2 \geq 0$ (in particular, we can have $\lambda_2=0$ and/or $M=\infty$ in Problem~\eqref{eqn:mio}). A practical way to choose a value of $M$ is by using the solution from our approximate solver.
	We discuss the choice of the Big-M parameter $M$ for Problem~\eqref{eqn:mio} in our numerical experiments in Appendix~\ref{app:num-details}. Additionally, in the Appendix~\ref{app:bigm} we perform ablation studies with different choices of $M$.
We also refer to \cite{bertsimas2016best,xie2020scalable,hazimeh2020sparse} for additional discussions on how to estimate $M$ in practice, in the context of sparse regression. 
In Section~\ref{sec:comp}, we discuss our custom algorithms (both approximate and exact) for solving Problem~\eqref{eqn:mio}.

\section{Computational Framework}\label{sec:comp}
We present \ourmethod, a custom branch-and-bound (BnB) framework for Problem~\eqref{eqn:mio}. In Section~\ref{subsec:overview-BnB}, we discuss related work on nonlinear BnB and provide an overview of our BnB framework. In Section~\ref{subsec:formulations}, we study the formulations of node relaxations of Problem~\eqref{eqn:mio} in the BnB. We 
present algorithms for the node relaxations and primal heuristics in Sections~\ref{subsec:ASCD},~\ref{subsec:relaxation-solving} and~\ref{subsec:incumbents-solving}. In Section~\ref{subsec:dual-bounds}, we show how to obtain dual bounds for the node relaxations.

\subsection{Related work and overview of BnB framework}\label{subsec:overview-BnB}
At a high level, \ourmethod~extends the BnB framework for $\ell_0$-penalized least squares regression~\citep{hazimeh2020sparse} to the
pseudo-likelihood problem~\eqref{eqn:mio}.
There are important differences in these problems that pose challenges for Problem~\eqref{eqn:mio}: First,  
Problem~\eqref{eqn:mio} involves a $p \times p$ matrix $\B\Theta$ involving $\mathcal{O}(p^2)$ variables---in sparse regression, in contrast, we have $\mathcal{O}(p)$-many regression coefficients. The objective in problem~\eqref{eqn:mio} involves additional non-linearities (due to the extra logarithm term, quadratic-over-linear structure), and symmetry constraints arising from the pseudo-likelihood function---these require modifications to our 
algorithm, including obtaining dual bounds and establishing computational guarantees for our method.

\smallskip 

\customparagraph{Our strategies} In \ourmethod, we use the following algorithm choices:
\begin{itemize}
	\item \textbf{Node relaxations:} 
	We consider and solve node relaxation reformulations of Problem~\eqref{eqn:mio} in the original $\B\Theta$-space instead of the extended $(\bm\Theta,\bm z,\bm s)$-space. These formulations are 
	studied in Section~\ref{subsec:formulations}. 
	\item \textbf{Convex relaxation solver:} To solve the node relaxations, we develop a scalable coordinate descent (CD) algorithm with active set updates. The algorithm exploits and shares warm starts and active set information across the BnB tree to further improve computational efficiency. Our algorithm is described in Section~\ref{subsec:ASCD}; additional computational details and convergence guarantees are presented in Section~\ref{subsec:relaxation-solving}. 
	\item \textbf{Dual bounds:} Dual bounds of the node relaxation problem are useful for search space pruning. We develop a novel method to compute dual bounds from the primal solutions (cf Section~\ref{subsec:dual-bounds}).
	\item \textbf{Approximate solver and primal solutions:} Good primal solutions can lead to aggressive pruning in the search tree
	and can reduce the overall runtime for BnB. At each node of the BnB tree, we attempt to improve the upper bound based on a solution $\hat{\bm\Theta}$ (say) 
	from the current node's relaxation problem. Specifically, let $\mathcal{S}$ 
	denote the support (i.e., nonzero indices) of the current solution $\hat{\bm\Theta}$. 
	Using the framework discussed in Section~\ref{subsec:ASCD},
	we obtain good solutions (primal solutions) for the following problem: 
	\begin{align}
		\min_{\bm\Theta\in\mathbb{S}^p}&~ \sum_{i=1}^p\bigl(-\log\theta_{ii}+\frac{1}{\theta_{ii}}\norm{\btX\bm\theta_i}^2\bigr)+\sum_{(i,j)\in\mathcal{S}}\bigl(\lambda_0\bm1\{\theta_{ij}\neq 0\}+\lambda_2\theta_{ij}^2\bigr)\label{eqn:node-incumbent}\\
		\text{s.t.~}&~|\theta_{ij}|\leq M,~\forall (i,j)\in\mathcal{S};~~ \theta_{ij}=0,~\forall (i,j)\in\mathcal{S}^c, \nonumber 
	\end{align}
	where the constraint $\bm\Theta\in\mathbb{S}^p$ enforces $\bm\Theta$ to be symmetric. 
	Section~\ref{subsec:incumbents-solving} presents algorithms to compute good solutions to Problem~\eqref{eqn:node-incumbent}.
\end{itemize}

\subsection{Optimization Problems at every node of the BnB tree}\label{subsec:formulations}
We study the node relaxations of Problem~\eqref{eqn:mio} as they arise in a typical node of \ourmethod's BnB tree. We start with the root relaxation, where we relax each binary variable $z_i$ to the interval $[0,1]$. While the root relaxation involves the extended variables $(\bm\Theta, \bm z,\bm s)$, we present a reformulation in the original 
$\bm\Theta$-space, as this allows our algorithms to operate on a significantly smaller space. This reformulation is given as
\begin{equation}\label{eqn:root-relaxation}
	\min_{\bm\Theta\in\BS^p}~F_{\mathsf{root}}(\bm\Theta)=\sum_{i=1}^p\bigl(-\log(\theta_{ii})+\frac1{\theta_{ii}}\norm{\btX\bm\theta_{i}}^2\bigr)+\sum_{i<j}\psi(\theta_{ij};\lambda_0,\lambda_2,M), 
\end{equation}
where, as shown by~\citet{hazimeh2020sparse}, the penalty function (aka regularizer) $\psi$ is:
\begin{align}
	\psi(\theta;\lambda_0,\lambda_2,M)&=\min_{z,s}~~\lambda_0z+\lambda_2s~~\text{s.t.}~~sz\geq \theta^2, |\theta|\leq Mz, z\in[0,1]\nonumber\\
	&=\left\{\begin{array}{ll}
		2\sqrt{\lambda_0\lambda_2}|\theta| &~\text{if}~ |\theta|\leq \sqrt{\lambda_0/\lambda_2}\leq M  \\
		\lambda_0+\lambda_2\theta^2 &~\text{if}~ \sqrt{\lambda_0/\lambda_2}\leq |\theta|\leq M\\
		(\lambda_0/M+\lambda_2M)|\theta|&~\text{if}~|\theta|\leq M\leq \sqrt{\lambda_0/\lambda_2}\\
		\infty&~\text{if}~|\theta|>M.
	\end{array}\right.\label{eqn:psi}
\end{align}

	\begin{figure}[ht]
	\centering
	\rotatebox{90}{~~~~~~~~~~~~~~~~~~$\psi(\theta,\lambda_0,1-\lambda_0,1)$}\includegraphics[width=0.5\linewidth,trim =.8cm 0.2cm 0.8cm 0cm, clip = true]{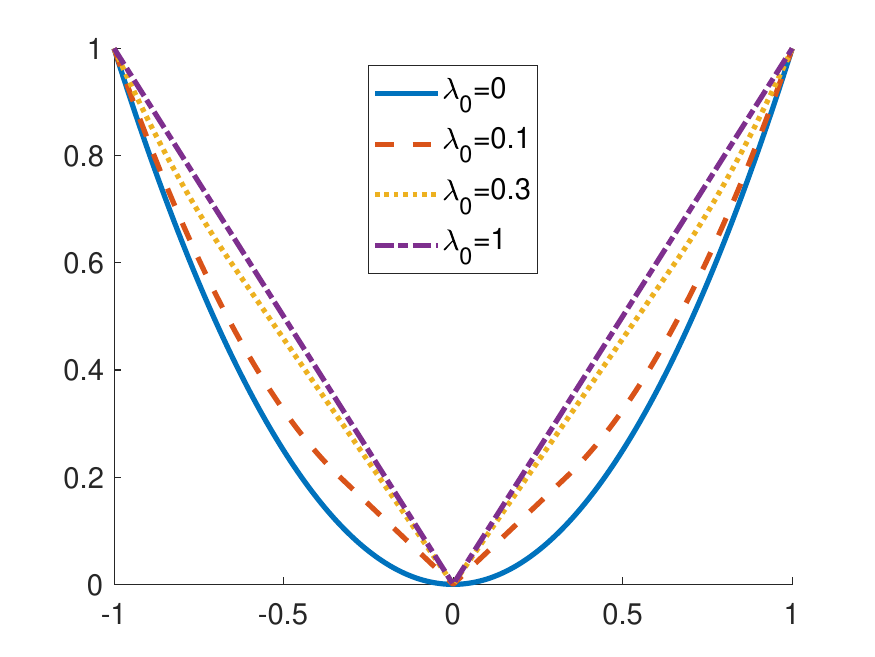}\\  
	{~~~$\theta$}
		\caption{Illustration of $\psi(\theta,\lambda_0,1-\lambda_0,1)$} 
	\label{fig:psiillust}
\end{figure}

In a special case where $M\to\infty,\lambda_2\to 0$, we have that $\psi(\theta;\lambda_0,\lambda_2,M)\propto |\theta|$, while under $M\to\infty,\lambda_0= 0$, we have that $\psi(\theta;\lambda_0,\lambda_2,M)\propto \theta^2$. For other values of $\lambda_0,\lambda_2$, the function $\psi$ is a hybrid of $\ell_1$ and $\ell_2$ penalties---see Figure~\ref{fig:psiillust} for an illustration.
The function $\psi$ is closely related to the reverse Huber penalty~\citep{owen2007robust,dong2015regularization}.

\customparagraph{Node relaxation problem}
For each node within the BnB tree, the node relaxation is similar to the root relaxation, except that some of $z_{ij}$'s are fixed to $0$ and some to $1$. We let $[\underline z_{ij}, \bar z_{ij}]$ denote the range of $z_{ij}$ at each node relaxation.\footnote{For example, if $z_{ij}$ is relaxed to $[0,1]$, then $\underline z_{ij}=0$ and $\bar z_{ij}=1$; if $z_{ij}$ is fixed to 0 (or 1), then $\underline z_{ij}=\bar z_{ij}=0$ (or $\underline z_{ij}=\bar z_{ij}=1$).} 
Using this notation, similar to derivation of~\eqref{eqn:root-relaxation}, we write the corresponding node relaxation problem as:
\begin{equation}\label{eqn:relaxation}
	\min_{\bm\Theta\in\BS^p}~F_{\mathsf{node}}(\bm\Theta)=\sum_{i=1}^p\bigl(-\log(\theta_{ii})+\frac1{\theta_{ii}}\norm{\btX\bm\theta_{i}}^2\bigr)+\sum_{i<j}g(\theta_{ij};\lambda_0,\lambda_2,M, \underline{z}_{ij},\bar{z}_{ij}),
\end{equation}
where 
\begin{equation}\label{eqn:relax-reg}
	g(\theta; \lambda_0,\lambda_2,M, \underline{z},\bar{z})=\left\{\begin{array}{ll}
		\psi(\theta;\lambda_0,\lambda_2,M) & \text{if}~ \underline{z}=0,\bar{z}=1 \\
		\varphi(\theta; z,\lambda_0,\lambda_2,M) &\text{if}~\underline{z}=\bar{z}=z, 
	\end{array}\right.
\end{equation}
and
\begin{equation}
	\varphi(\theta;z,\lambda_0,\lambda_2,M)=\left\{\begin{array}{ll}
		\chi\{\theta=0\} &~\text{if}~ z = 0 \\
		\chi\{|\theta|\leq M\}+\lambda_0+\lambda_2\theta^2& ~\text{if}~ z=1.
	\end{array}\right.\label{eqn:phi}
\end{equation}
Intuitively, if $\underline{z}_{ij}=0,\bar{z}_{ij}=1$, we have not branched on $z_{ij}$ yet and it is relaxed to be in $[0,1]$, hence, $g(\theta_{ij},\cdot,\cdot,\cdot,\underline{z}_{ij},\bar{z}_{ij})=\psi(\theta_{ij},\cdot,\cdot,\cdot)$. If $\underline{z}_{ij}=\bar{z}_{ij}=0$, then the penalty $g(\theta_{ij},\cdot,\cdot,\cdot,\underline{z}_{ij},\bar{z}_{ij})$ enforces $\theta_{ij}=0$. Otherwise, if $\underline{z}_{ij}=\bar{z}_{ij}=1$, 
$g(\theta_{ij},\cdot,\cdot,\underline{z}_{ij},\bar{z}_{ij})$ adds the penalty $\lambda_0+\lambda_2\theta_{ij}^2$ on $\theta_{ij}$ and enforces the constraint $|\theta_{ij}|\leq M$.
For notational convenience, we present below a unified formulation that encompasses the original pseudo-likelihood formulation, its restriction at every node, and the node relaxations:
\begin{equation}\label{eqn:unified}
	\min_{\bm\Theta\in\mathbb{S}^p} ~F(\bm\Theta):=\sum_{i=1}^p\left(-\log(\theta_{ii})+\frac{1}{\theta_{ii}}\norm{\btX\bm\theta_i}^2\right)+\sum_{i<j}h_{ij}(\theta_{ij}),
\end{equation}
where $h_{ij}$ is a penalty function (aka regularizer). In particular, depending upon the choice of $h_{ij}$, Problem~\eqref{eqn:unified} specializes to the original problem~\eqref{eqn:L0L2}, the root relaxation problem \eqref{eqn:root-relaxation}, the node relaxation problem \eqref{eqn:relaxation} and the problem for incumbent solving \eqref{eqn:node-incumbent}.

In what follows, we present a scalable active-set coordinate descent algorithm for solving (or approximately solving) Problem~\eqref{eqn:unified},
when $F$ is convex (or, non-convex).

\subsection{Active-set Coordinate Descent}\label{subsec:ASCD}
Due to the separability of the (nonsmooth) regularizers $h_{ij}$ for $i=1,\ldots,p, j>i$, Problem~\eqref{eqn:unified} is amenable to cyclic CD~\citep{tseng2001convergence} where we perform full minimization for every coordinate in the lower triangular part of $\bm\Theta$. CD-type methods are commonly used for solving large-scale structured optimization problems in statistical learning due in part to their inexpensive iteration updates and capability of exploiting problem structure. 
They have been used 
with success in various settings~(\citet{mazumder2012graphical,friedman2010regularization,hazimeh2020fast}; see also references therein). 

As presented in Algorithm~\ref{algo:CD}, at each step, cyclic CD optimizes the objective with respect to one coordinate (with other variables remaining fixed). 
We cycle through all coordinates according to a fixed ordering of the indices. In Algorithm~\ref{algo:CD}, $\bm E_{ij}\in\R^{p\times p}$ denotes a standard basis matrix where the $(i,j)$-th entry is one, and others are zero. 
\begin{algorithm}[H]
	\begin{algorithmic}[1]
		\Require An initialization $\hat{\bm\Theta}$
		\While {not converged}
		\For {each pair $i<j$}
		\State $\hat\theta_{ij}=\hat\theta_{ji}\gets \arg\min_{\theta_{ij}} F(\hat{\bm\Theta}-\hat \theta_{ij}\bm E_{ij}-\hat\theta_{ij}\bm E_{ji}+\theta_{ij}\bm E_{ij}+\theta_{ij}\bm E_{ji})$
		\EndFor
		\For {$i=1,2,\ldots p$}
		\State $\hat\theta_{ii}\gets \arg\min_{\theta_{ij}} F(\hat{\bm\Theta}-\hat \theta_{ii}\bm E_{ii}+\theta_{ii}\bm E_{ii})$
		\EndFor
		\EndWhile
	\end{algorithmic}
	\caption{Cyclic CD for Problem~\eqref{eqn:unified}}
	\label{algo:CD}
\end{algorithm}

Algorithm~\ref{algo:CD} differs from earlier work on CD for $\ell_0$-penalized regression problems~\citep{hazimeh2020fast,hazimeh2020sparse}: 
(a)~In problem~\eqref{eqn:unified}, the optimization variable is a sparse symmetric matrix. The CD algorithm needs to handle the diagonal and (symmetric) off-diagonal entries differently, as shown in lines 3 and 6 in Algorithm~\ref{algo:CD}; (b) Even when $h_{ij}$'s are convex, the convergence guarantee of Algorithm~\ref{algo:CD} is unknown to our knowledge. The convex node relaxation problem considered in \citet{hazimeh2020sparse} is a special case of~\eqref{eqn:unified}---
the problem in~\eqref{eqn:unified} has additional logarithmic and quadratic-over-linear terms.
Later in Section~\ref{subsec:relaxation-solving}, we will provide a convergence guarantee for Algorithm~\ref{algo:CD} for the root/node relaxation subproblems~\eqref{eqn:relaxation}. 

\customparagraph{Coordinate updates} The coordinate updates in lines 3 and 6 of Algorithm~\ref{algo:CD} can be performed analytically. For any $i<j$ and $h_{ij}$, the update in line 3 of Algorithm~\ref{algo:CD} is given by\footnote{In both updates \eqref{eqn:offdiag-update} and \eqref{eqn:diag-update}, we use the superscript `$+$' to distinguish the entries before and after the coordinate update. To be more specific, the symbols $\hat\theta_{ij}$ and $\hat\theta_{ii}$ in $a_{ij},b_{ij},\bm r_i$ and $\bm e_i$ are the values before the update, while $\hat\theta_{ij}^+$ and $\hat\theta_{ii}^+$ are the ones after the update.}
\begin{equation}\label{eqn:offdiag-update}
	\hat\theta_{ij}^+=\arg\min_\theta~a_{ij}\theta^2+b_{ij}\theta+h_{ij}(\theta)
\end{equation}
where 
$$a_{ij}=\frac{v_j}{\hat\theta_{ii}}+\frac{v_i}{\hat\theta_{jj}},\quad\quad b_{ij}=\frac{2\btx_j^\top(\bm r_i-\hat\theta_{ij}\btx_j)}{\hat\theta_{ii}}+\frac{2\btx_i^\top(\bm r_j-\hat\theta_{ij}\btx_i)}{\hat\theta_{jj}}$$
with $v_i=\btx_i^\top\btx_i$ and ${\bm{r}}_i=\btX\hat{\bm\theta}_i$. 
The solution to~\eqref{eqn:offdiag-update} can be computed in closed form---see Appendix~\ref{subsec:oracles-reg} for details.

For the diagonal entries $\theta_{ii}$, the update in line 6 of Algorithm~\ref{algo:CD} is given by
\begin{equation}
	\hat\theta_{ii}^+=\arg\min_{\theta}-\log\theta+v_i\theta+\frac{\norm{\bm e_i}^2}{\theta}=\frac{1+\sqrt{1+4v_i\norm{\bm e_i}^2}}{2v_i},\label{eqn:diag-update}
\end{equation}
where $\bm e_i=\bm r_i-\hat\theta_{ii}\btx_i$.
In the implementation of CD, instead of computing $\bm r_i$'s from scratch for every update, to improve efficiency we keep track of these values and update them after each coordinate update. This is often referred to as the residual update in sparse regression~\citep{friedman2010regularization,hazimeh2020fast}.

\customparagraph{Active sets}
The cost of computing $a_{ij},b_{ij}$ and $\bm e_i$ in the updates \eqref{eqn:offdiag-update} and \eqref{eqn:diag-update} are $\mathcal{O}(n)$, and each full pass across all coordinates 
involves updating $\mathcal{O}(p^2)$ variables. 
Hence, each iteration of Algorithm~\ref{algo:CD} costs $\mathcal{O}(np^2)$, which is quite expensive when $n$ or $p$ becomes large. To reduce the computation cost, we propose an active-set method:  we run Algorithm~\ref{algo:CD} restricted to the diagonal variables $\mathcal{D}$ and a small subset of the off-diagonal variables denoted by 
$\mathcal{A}$ that is, 
$\mathcal{A}\subseteq \{(i,j):i<j, i,j\in[p]\}$ and  $\bm\Theta|_{\mathcal{A}^c\backslash \mathcal{D}}=0$. 
After solving the restricted problem, we augment the active set with the off-diagonal variables $(i,j)\in \mathcal{A}^c$ that violate the coordinate-wise optimality conditions, and re-solve the problem on the new active set. We repeat this process and terminate the algorithm until there are no more violations. Similar active-set updates have been used earlier in other problems~\citep{hazimeh2020fast,hazimeh2020sparse,chen2020multivariate}. Our proposed active-set method is summarized in Algorithm~\ref{algo:ASCD}.

\begin{algorithm}[H]
	\begin{algorithmic}[1]
		\Require An initial active set $\mathcal{A}$ and initial solution $\hat{\bm\Theta}$
		\While {not converged}
		\State Get a solution for $\min_{\bm\Theta\in\mathbb{S}^p}~F(\bm\Theta)~~\text{s.t.}~~ \bm\Theta|_{\mathcal{A}^c\cap \mathcal{D}^c}=0$ using Algorithm~\ref{algo:CD}
		\State Find indices that violate coordinate-wise optimality conditions:
		$\mathcal{V}\gets\{(i,j):i<j, i,j\in[p], \hat\theta_{ij}=0, 0\notin\arg\min_{\theta_{ij}} F(\hat{\bm\Theta}-\hat \theta_{ij}\bm E_{ij}-\hat\theta_{ij}\bm E_{ji}+\theta_{ij}\bm E_{ij}+\theta_{ij}\bm E_{ji})\}$
		\State If $\mathcal{V}$ is empty then \textbf{Terminate}; otherwise, $\mathcal{A}\gets\mathcal{A}\cup\mathcal{V}$
		\EndWhile
	\end{algorithmic}
	\caption{Active set method for Problem~\eqref{eqn:unified}}
	\label{algo:ASCD}
\end{algorithm}

In what follows, we discuss some details of Algorithm~\ref{algo:ASCD} when we use it to solve the root or node relaxation~\eqref{eqn:relaxation} which are convex problems (cf Section~\ref{subsec:relaxation-solving}).
Section~\ref{subsec:incumbents-solving} discusses how to compute good solutions for the non-convex problem~\eqref{eqn:L0L2} and its restricted version problem~\eqref{eqn:node-incumbent} --- these help us obtain incumbents for the BnB procedure.

\subsection{Solving the node relaxations}\label{subsec:relaxation-solving}
We discuss computational details and convergence guarantees of Algorithms~\ref{algo:CD} and~\ref{algo:ASCD} for the convex Problem~\eqref{eqn:relaxation}.

\customparagraph{Coordinate updates} Recall that as a special case of the general formulation~\eqref{eqn:unified}, the node relaxation subproblem~\eqref{eqn:relaxation} has regularizers $h_{ij}(\theta_{ij})=g(\theta_{ij};\lambda_0,\lambda_2,M,\underline{z}_{ij},\bar{z}_{ij})$ where $g(\theta;\lambda_0,\lambda_2,M,\underline{z},\bar{z})$ is defined in~\eqref{eqn:relax-reg}. Its corresponding off-diagonal updates~\eqref{eqn:offdiag-update} in line~3 of Algorithm~\ref{algo:CD} has a closed-form solution --- see Appendix~\ref{subsubsec:g-oracles}.

\customparagraph{Computational guarantee} As we mentioned earlier, due to the presence of the logarithmic terms and quadratic-over-linear structure of the pseudo-likelihood, there is no known convergence guarantee for the CD algorithm (See also~\citet{khare2015convex}). The following theorem provides such a convergence guarantee and presents the sublinear rate of convergence for Algorithm~\ref{algo:CD} applied to the relaxation subproblem~\eqref{eqn:relaxation}. 
\begin{theorem}\label{thm:convergence-relaxation}
	Given any initialization $\bm\Theta^{(0)}$, let $\bm\Theta^{(t)}$ be the $t$-th iterate generated by Algorithm~\ref{algo:CD} (end of $t$-th iteration of while-loop) for the convex Problem~\eqref{eqn:relaxation}. Then there is a constant $C$ that depends on $\bm\Theta^{(0)}$, for any $t\geq 1$,
	\begin{equation}
		F_{\mathsf{node}}(\bm\Theta^{(t)})-F_{\mathsf{node}}^*\leq \frac{C}{t},
	\end{equation}
	where $F_{\mathsf{node}}^*=\min_{\bm\Theta\in\mathbb{S}^p}F_{\mathsf{node}}(\bm\Theta)$.
\end{theorem}
Proof of Theorem~\ref{thm:convergence-relaxation} can be found in Appendix~\ref{subsec:convergence-guarantee-proof} where we also derive convergence guarantees for the unified formulation~\eqref{eqn:unified} encompassing a larger family of regularizers.

\customparagraph{Initializations} The number of iterations taken by Algorithm~\ref{algo:ASCD} depends upon the initial active set $\mathcal{A}$. Due to the similarity between the parent node and its two child nodes, we take the initial active set to be the same as the support of the relaxation solution at the parent node. For the root relaxation problem, we initialize the active set to the support of the warm start obtained by the approximate solver, which is discussed in Section~\ref{subsec:incumbents-solving}.

\customparagraph{Approximate solution} For practical purposes, we usually solve the restricted problem in line 2 of Algorithm~\ref{algo:ASCD} up to some numerical tolerance: we terminate Algorithm~\ref{algo:CD} when the relative change in the objectives is small.\footnote{ We refer to Appendix~\ref{app:num-details} for more details on the value of tolerance we use in practice.} Such approximate solutions would still result in a convergent BnB procedure as long as we have dual bounds for the convex problem (for search-space pruning) which is discussed next.

\subsection{Dual bounds}\label{subsec:dual-bounds}
We discuss how to compute dual bounds for Problem~\eqref{eqn:relaxation} based on an approximate primal solution to the node relaxation~\eqref{eqn:relaxation} as obtained from Algorithm~\ref{algo:ASCD}.
First, we present the Lagrangian dual of the convex program \eqref{eqn:relaxation}:
\begin{theorem}\label{thm:dual-bound}
	The dual of Problem~\eqref{eqn:relaxation} is given by
	\begin{equation}\label{eqn:dual-relaxation}
		\max_{\bm \nu}~D(\bm\nu):=p+\sum_{i=1}^p\log(-\norm{\bm\nu_i}^2/4-\btx_i^\top\bm \nu_i)-\sum_{i<j}g^*(\btx_j^\top\bm \nu_i+\btx_i^\top\bm\nu_j;\lambda_0,\lambda_2,M,\underline{z}_{ij},\bar{z}_{ij}),
	\end{equation}
	where $g^*(\cdot; \lambda_0,\lambda_2,M,\underline{z},\bar{z})$ is the convex conjugate of $g(\cdot; \lambda_0,\lambda_2,M,\underline{z},\bar{z})$. Strong duality holds,  $$\min_{\bm\Theta\in\mathbb{S}^p}F_{\mathsf{node}}(\bm\Theta)=\max_{\bm\nu}D(\bm\nu).$$ Furthermore, if $\bm\Theta$ is an optimal primal solution to \eqref{eqn:relaxation}, let $\bm r_i^*=\btX\bm\theta_i$ for all $i\in[p]$, then $\bm\nu_i^*=-2\bm r_i^*/\theta_{ii}$ is an optimal dual solution to \eqref{eqn:dual-relaxation}.
\end{theorem}

Given any $\alpha$, the convex conjugate $g^*(\alpha;\lambda_0,\lambda_2,M,\underline{z},\bar{z})$ can be computed explicitly (see Appendix~\ref{subsec:dual-bound-appendix} for details).

\customparagraph{Dual bounds} Let $\hat{\bm\Theta}$ be an approximate primal solution to the node relaxation~\eqref{eqn:relaxation}, as available 
from Algorithm~\ref{algo:CD} or \ref{algo:ASCD}. We can construct a dual solution based on $\hat{\bm\Theta}$ as follows
\begin{equation}\label{eqn:dual-solution}
	\hat{\bm\nu}_i=-2\btX\hat{\bm\theta}_i/\hat\theta_{ii} ~~ \forall i \in [p].
\end{equation}
Notice that when $-\norm{\hat{\bm\nu}_i}^2/4-\btx_i^\top\bm\nu_i\leq 0$, 
the dual solution is infeasible, and thus $D(\hat{\bm\nu})=-\infty$. This indicates the optimization error of the current inexact solution $\hat{\bm\Theta}$ is still not small enough. In such a case, we run a few more iterations of Algorithm~\ref{algo:CD} (or Algorithm~\ref{algo:ASCD}) to improve the solution accuracy. 

\customparagraph{Efficient computation of the dual bounds} A direct computation of the dual bound $D(\hat{\bm\nu})$ costs $\mathcal{O}(np^2)$. This can be reduced to $\mathcal{O}(nk)$ if $k$ denotes the number of nonzero off-diagonal entries in the precision matrix estimate 
obtained from Algorithm~\ref{algo:ASCD}. 
As shown by~\citet{hazimeh2020sparse} in the sparse regression setting, if $\hat\theta_{ij}=0$, then $$\psi^*(\btx_j^\top\hat{\bm\nu}_j+\btx_i^\top\hat{\bm\nu}_i;\lambda_0,\lambda_2,M)=0.$$
This means we only need to compute $\psi^*$ (as a special case of $g^*$ in root relaxation) for any $(i,j)$ such that $i\neq j$ and $\hat{\theta}_{ij}\neq0$, which reduces the computation cost to $\mathcal{O}(nk)$.

We can also consider the node relaxation setting --- the only difference is that if $\hat\theta_{ij}=0$ and $\underline{z}_{ij}=\bar{z}_{ij}=1$, then $g^*(\btx_j^\top\hat{\bm\nu}_j+\btx_i^\top\hat{\bm\nu}_i;\lambda_0,\lambda_2,M,\underline{z}_{ij},\bar{z}_{ij})=-\lambda_0$. Since for every node we can easily store the number of $z_{ij}$'s that are fixed to 1 with cost $\mathcal{O}(1)$, the cost of computing dual bounds remains $\mathcal{O}(nk)$. The formal statement is presented in Proposition~\ref{prop:eff-comp-dual-bounds}.

\begin{proposition}\label{prop:eff-comp-dual-bounds}
	Let $\hat{\bm\Theta}$ be a solution obtained by Algorithm~\ref{algo:ASCD} applied to the node relaxation~\eqref{eqn:relaxation}, and $\hat{\bm\nu}$ is a dual feasible solution obtained by \eqref{eqn:dual-solution}. Denote by
	$$\hat{\mathcal{S}}=\{(i,j):i<j, \hat{\theta}_{ij}\neq 0\}, \quad\text{and}\quad \mathcal{F}_1= \{(i,j):i<j, \underline z_{ij}=\bar z_{ij}=1\}.$$If for any $i\in[p]$, we have $-\norm{\hat{\bm\nu}_i}^2/4-\btx_i^\top\hat{\bm\nu}_i>0$, then
	$$D(\hat{\bm\nu})=p+\sum_{i=1}^p\log(-\norm{\hat{\bm\nu}_i}^2/4-\btx_i^\top\hat{\bm\nu}_i)-\sum_{(i,j)\in\hat{\mathcal{S}}}g^*(\btx_j^\top\hat{\bm\nu}_i+\btx_i^\top\hat{\bm\nu}_j;\lambda_0,\lambda_2,M,\underline{z}_{ij},\bar{z}_{ij})+\lambda_0|\mathcal{F}_1\backslash\hat{\mathcal{S}}|.$$Otherwise, $D(\hat{\bm\nu})=-\infty$.
\end{proposition}
As shown in Proposition~\ref{prop:eff-comp-dual-bounds}, calculating the dual bound requires calculating the values of $g^*(\btx_j^\top\hat{\bm\nu}_j+\btx_i^\top\hat{\bm\nu}_i;\lambda_0,\lambda_2,M,\underline{z}_{ij},\bar{z}_{ij})$ for $(i,j)\in\hat{\mathcal{S}}$.
In practice, we always make sure $\mathcal{F}_1$ is a subset of the active set $\mathcal{A}$, and thus both $\hat{\mathcal{S}}$ and $\mathcal{F}_1$ are subsets of $\mathcal{A}$. We can compute the convex conjugate terms restricted to $\mathcal{A}$, and the corresponding computational cost is $\mathcal{O}(n|\mathcal{A}|)$. Since we anticipate a sparse solution, we expect $|\mathcal{A}|$ to be small, leading to efficient calculations of the dual bound.

\subsection{Approximate solver and incumbents}
\label{subsec:incumbents-solving}

We discuss how to obtain approximate solutions 
to~\eqref{eqn:node-incumbent} using Algorithms~\ref{algo:CD} and \ref{algo:ASCD} (cf Section~\ref{subsec:ASCD}). These algorithms are used at every node of our BnB search framework to obtain new incumbents. As the CD algorithms can get stuck in a local solution when applied to the non-convex problem~\eqref{eqn:node-incumbent}, we also discuss a local search method that can improve solution quality by applying them on top of Algorithms~\ref{algo:CD} and \ref{algo:ASCD}. We note that the combination of Algorithms~\ref{algo:CD},~\ref{algo:ASCD} and local search to get good solutions (without optimality certificates) is of independent interest.

\customparagraph{Coordinate updates} The objective in \eqref{eqn:node-incumbent} is a special case of the unified formulation~\eqref{eqn:unified} with $$h_{ij}(\theta_{ij})=\left\{\begin{array}{ll}
	\lambda_0\bm1\{\theta_{ij}\neq 0\}+\lambda_2\theta_{ij}^2+\chi\{|\theta_{ij}|\leq M\},&  \text{if}~(i,j)\in\mathcal{S} \\
	\chi\{\theta_{ij}=0\}, &\text{if}~(i,j)\in\mathcal{S}^c. 
\end{array}\right.$$
The corresponding off-diagonal update in \eqref{eqn:offdiag-update} in 
line~3 of Algorithm~\ref{algo:CD} has a closed-form expression---see Appendix~\ref{subsubsec:L0-oracles} for details.

\customparagraph{Choosing $\mathcal{S}$} For computing the initial incumbent in the BnB tree, or when using our algorithm as a standalone approximate solver, we take $\mathcal{S}$ to be the set of all upper triangular indices: $\mathcal{S}=\{(i,j):1\leq i<j\leq p\}$, and use Algorithm~\ref{algo:ASCD} to obtain a good solution to the problem.
At every node of the BnB tree we attempt to obtain a better feasible solution (resulting in an improved incumbent) based on the solution $\hat{\bm\Theta}$ available from the current node's relaxation. 
In this case, we set $\mathcal{S}$ based on 
the sparsity pattern of $\hat{\bm\Theta}$. We propose two options: (i) directly taking the support of $\bm z$, i.e. $\mathcal{S}=\{(i,j):i<j,~\text{and}~z_{ij}>0\}$; (ii) taking the support of rounded $\bm z$, i.e. $\mathcal{S}=\{(i,j):i<j,~\text{and}~z_{ij}\geq 0.5\}$. In this case, due to the sparsity of $\hat{\bm\Theta}$, we expect $\mathcal{S}$ to be small 
and CD can efficiently solve the reduced problem.

\customparagraph{Initializations} Note that~\eqref{eqn:node-incumbent} is a discrete optimization problem, 
the number of iterations in Algorithm~\ref{algo:CD} or \ref{algo:ASCD} and the quality of the approximate solution given by the algorithms are affected by the quality of the initial solution $\hat{\bm\Theta}$ and/or the quality of the initial active set $\mathcal{A}$. 

For obtaining the initial incumbent solution in the BnB tree, or when using our algorithm as a standalone approximate solver, as we do not have any prior knowledge, we initialize Algorithm~\ref{algo:ASCD} with a diagonal matrix: $\hat{\bm\Theta}^{(0)}=\mathrm{diag}(v_1^{-1},\dots,v_p^{-1}),$
which is optimal when all the off-diagonal entries are forced to be 0. We obtain the initial active set $\mathcal{A}$ by correlation screening~\citep{hazimeh2020fast} --- computing the correlation matrix of $\bm X$ and taking a small portion of coordinates $(i,j)$ that have the highest correlations in each row. 

In order to update the incumbent at every node of the BnB tree, we initialize Algorithm~\ref{algo:CD} with the current relaxation solution $\hat{\bm\Theta}$ restricted on $\mathcal{S}$.

\customparagraph{Local Search} The CD-based solver for Problem~\eqref{eqn:node-incumbent} can get stuck in a suboptimal solution due to the non-convexity of the objective. 
To improve the quality of the solution, we use a local search procedure where roughly speaking we change the support of the solution followed by running CD on the new support. 
The details of the local search are discussed in Appendix~\ref{subsec:localsearch}.

\subsection{Additional BnB Details}\label{sec:bnbdetails}
Finally, we discuss some details from our BnB implementation below.

\customparagraph{Branching Strategy} We use a maximum fractional branching strategy~\citep{belotti2013mixed}, where we branch on the variable $z_{ij}$ that is the furthest from integrality. This choice is mostly due to computational efficiency as there are $\mathcal{O}(p^2)$-many binary variables in our problem and hence, a more sophisticated branching can be computationally expensive.

\customparagraph{Search Strategy} We use a Breadth-First Search (BFS) for the BnB procedure. When we do early termination of the BnB tree, the BFS strategy offers us the flexibility to exit the BnB tree after processing nodes up to a certain depth.

\section{Statistical Properties}\label{sec:stat}
We investigate the statistical properties of the estimator~\eqref{main-pseudo-1}. We consider two different metrics to quantify the quality of our estimator. First, we present estimation error bounds of the form $\|\B{\Theta}^*-\hat{\B{\Theta}}\|_F$ where $\B{\Theta}^*$ is the underlying precision matrix and $\hat{\B{\Theta}}$ is its estimate. Next, we consider the variable selection properties of our estimator. Throughout this section, to simplify our proofs, we study a slightly modified version of Problem~\eqref{main-pseudo-1}. In particular, instead of considering the matrix $\B\Theta$ to be symmetric, we consider symmetric support, that is $\beta_{ij}\neq 0\Leftrightarrow \beta_{ji}\neq 0$. 
For technical reasons, we include a bound constraint on $\{\sigma_j\}$; and consider
\begin{align}\label{main-sup}
	\min_{\{{\beta_{ij}},\sigma_j,z_{ij}\}} \quad  & \sum_{j=1}^p\left[\log(\sigma_j)+ \frac{1}{2n\sigma_j^2}\left\Vert \B{x}_j -\sum_{i:i\neq j}{\beta}_{ij}\B{x}_i \right\Vert_2^2\right] +\lambda \sum_{i\neq j}z_{ij} \\
	\text{s.t.} \quad & z_{ij}\in\{0,1\},~~z_{ij}=z_{ji},~~(1-z_{ij})\beta_{ij}=0,~~\beta_{ii}=0~~ i\neq j\in[p]\nonumber\\
	& \sqrt{\ell}\leq\sigma_j\leq \sqrt{L},~~j\in[p] \nonumber
\end{align}
for some $0\leq\ell\leq L$. In practice, we observe the variances obtained from the optimization of~\eqref{eqn:L0L2} are bounded, so imposing a boundedness constraint is not restrictive. We note that the binary variables $z_{ij}$ here encode sparsity, similar to Problem~\eqref{eqn:mio}. Similar to Section~\ref{subsec:mio}, by taking  ${\theta}_{jj}={1}/{{\sigma}_j^2}$ and ${\theta}_{ji} = -{{\beta}_{ij}}/{{\sigma}_j^2}$, Problem~\eqref{main-sup} is equivalent to the convex mixed integer problem
\begin{align}\label{main-sup-convex}
	\min_{\{\B{\Theta},z_{ij}\}} \quad  &\sum_{i=1}^p\left[-\log(\theta_{ii})+\frac1{\theta_{ii}}\norm{\btX\bm\theta_{i}}^2\right]+\lambda\sum_{i\neq j}z_{ij}, \\
	\text{s.t.} \quad & z_{ij}\in\{0,1\},~~z_{ij}=z_{ji},~~(1-z_{ij})\theta_{ij}=0,~~ i\neq j\in[p]\nonumber\\
	& \frac{1}{L}\leq\theta_{jj}\leq \frac{1}{\ell},~~j\in[p]. \nonumber
\end{align}

\subsection{Estimation Error Bound}\label{est-error-stat} Before proceeding with our results in this section, we state our assumptions on the model.

\begin{assumption}
	We have $n$ independent samples $\B{x}^{(1)},\ldots,\B{x}^{(n)}\in\R^p$ from ${\mathcal N}(\B{0},\B{\Sigma}^*)$
	with $\B\Theta^* = (\B\Sigma^*)^{-1}$. Let $\{\beta^*_{ij}\},\{\sigma_j^*\}$ be as defined in~\eqref{lem_reg}. We assume:
	\begin{enumerate}[leftmargin=*,label=\textbf{(A\arabic*)}]
		\item There exist $l_{\sigma},u_{\sigma}\geq 0$ such that for any $j\in[p]$, $l_{\sigma}\leq \sigma_j^*\leq u_{\sigma}$ with $u_{\sigma}\geq 1$.\label{assum1-1}
		\item For any $i\neq j$, the values $|\beta^*_{ij}|$ are uniformly bounded by a universal constant, $|\beta^*_{ij}|\lesssim 1$.\label{assum1-2}
		\item We have $l_{\sigma}^2\geq \frac{6}{25}u_{\sigma}^4+\frac{2}{5}u_{\sigma}^2.$\label{assum1-3}
		\item For $j\in [p]$, $\left\vert\left\{i\in [p]: i\neq j, \beta^*_{ij}\neq 0\right\}\right\vert\leq k$. \label{assum1-4}
		\item For the matrix $\B{\Theta}^*$, we assume 
		$\min_{\substack{S\subseteq[p]}}\lambda_{\min}(\B{\Sigma}^*_{S,S})\geq \kappa^2 \gtrsim 1$
		where $\kappa$ is a universal constant.\label{assum1-5}
	\end{enumerate}
\end{assumption}
Assumptions~\ref{assum1-1} to~\ref{assum1-3} ensure that the entries of the matrix $\B{\Theta}^*$ are not too large or small---such assumptions are common in the literature~\citep{cai2011constrained,ravikumar2008model}. Assumption~\ref{assum1-4} states that each column of $\B{\Theta}^*$ is sparse and off-diagonals of each column have at most $k$ nonzeros---a standard assumption in the GGM literature \cite[Chapter 11]{wainwright2019high}. Note that our estimator does not assume that $k$ is known, and our results are adaptive to the sparsity level $k$. Assumption~\ref{assum1-5} states that the sub-matrices of $\B{\Sigma}^*$ are not badly conditioned---required in our analysis to derive estimation error bounds. Our analysis considers $\kappa$ to be a fixed universal constant while other parameters can vary. Additional discussions on our assumptions are presented in Appendix~\ref{sec:theory-discussion}.

\begin{theorem}\label{lowdimthm}
	Let $\{\hat{\beta}_{ij}\},\{\hat{\sigma}_j\}$ be an optimal solution to Problem~\eqref{main-sup} with $\lambda$ taken as
	\begin{equation}\label{thm1-lambda}
		\lambda=c_{\lambda}\frac{u_{\sigma}^2\log (2p/k)}{l_{\sigma}^2n} 
	\end{equation}
	for some sufficiently large universal constant $c_{\lambda}>0$, and $\ell = l_{\sigma}^2,L= u_{\sigma}^2$. Suppose, Assumptions~\ref{assum1-1} to~\ref{assum1-5} hold true with $p/k>5$, and $n\gtrsim kp\log p$. Then, with high probability\footnote{An explicit expression for probability can be found in~\eqref{thm1prob}\label{probfootnote}}, we have: 
	\begin{equation}
		\sum_{j\in[p]}(\hat{\sigma}_j-\sigma^*_j)^2 + \frac{1}{u_{\sigma}^2}\sum_{j\in[p]} \sum_{i:i\neq j} (\beta_{ij}^*-\hat{\beta}_{ij})^2\lesssim \frac{u_{\sigma}^2kp\log(2p/k)}{l_{\sigma}^2n}.
	\end{equation}
	
\end{theorem}
Theorem~\ref{lowdimthm} establishes an $\ell_2$ error bound in estimating the coefficients $\{\beta^*_{ij}\}$ and variances $\{(\sigma^*_j)^2\}$. 
Theorem~\ref{lowdimthm2}  presents an estimation error bound for the precision matrix $\B{\Theta}$ 
using the equivalence of Problems~\eqref{main-sup} and~\eqref{main-sup-convex}.

\begin{theorem}\label{lowdimthm2}
	Let $\hat{\B\Theta}$ be an optimal solution to Problem~\eqref{main-sup-convex} with $\lambda,\ell,L$ as defined in Theorem~\ref{lowdimthm}. Then, under the assumptions of Theorem \ref{lowdimthm}, with high probability\footref{probfootnote}, we have:
	\begin{equation}
		\left\Vert\hat{\B{\Theta}}-\B{\Theta}^*\right\Vert_F^2 \lesssim  \frac{(u_{\sigma}^6+u_{\sigma}^8)kp\log(2p/k)}{l_{\sigma}^{10}n}.
	\end{equation}
\end{theorem}

\customparagraph{Comparison with previous work}
Theorem~\ref{lowdimthm2} shows that our proposed estimator \ourmethod~achieves an estimation error rate (using Frobenius norm) of $\sqrt{kp\log p/n}$. This rate typically matches the estimation rate of current methods for sparse GGMs and is known to be minimax optimal up to logarithmic factors. See \citet[for example]{rothman2008sparse} for discussion on estimation error rate for GGMs in the context of a regularized maximum likelihood estimator---we are not aware of similar non-asymptotic error bounds for a pseudo-likelihood-based estimator which is what we focus on here.

To our knowledge, Theorem~\ref{lowdimthm2} is a novel result showing 
a non-asymptotic estimation guarantee for pseudo-likelihood-based GGM that results in a symmetric support. GGM estimators based on pseudo-likelihood have been considered in earlier work~\citep{peng2009partial,khare2015convex,friedman2010applications}---however, as far as we can tell, non-asymptotic analysis similar to the one in Theorem~\ref{lowdimthm2} has not appeared in earlier work. (The analysis of~\citet{peng2009partial} is asymptotic as $n,p\to\infty$). 

We note that Problem~\eqref{main-sup} has a non-convex and non-quadratic objective, hence existing proof techniques for sparse linear regression do not directly apply to our case. Therefore, we develop new tools for our proof. At a high level, although the objective of~\eqref{main-sup} is non-convex, we show that the pseudo-likelihood function can be locally lower bounded with a quadratic function, resulting in a bound of the form 
\begin{multline*}
	\sum_{j=1}^p\left[\log(\hat{\sigma}_j)+ \frac{1}{2n\hat{\sigma}_j^2}\left\Vert \B{x}_j -\sum_{i:i\neq j}\hat{\beta}_{ij}\B{x}_i \right\Vert_2^2\right]-\sum_{j=1}^p\left[\log(\sigma_j^*)+ \frac{1}{2n(\sigma_j^*)^2}\left\Vert \B{x}_j -\sum_{i:i\neq j}{\beta}^*_{ij}\B{x}_i \right\Vert_2^2\right]\\ \gtrsim \sum_{i\neq j}(\hat{\beta}_{ij}-\beta_{ij}^*)^2 + \sum_{j\in[p]}(\hat{\sigma}_j-\sigma_j^*)^2~~~~~~~~~~~~~~~~~
\end{multline*}
which we use to derive estimation error bounds.

\subsection{Support Recovery Guarantees}
We now study the variable selection properties 
of our estimator. To this end, we present a new set of assumptions that allow us to derive support recovery guarantees\footnote{Note that here we do not consider the assumptions stated in Section~\ref{est-error-stat}.}. 

\begin{assumption}
	We have $n$ independent samples $\B{x}^{(1)},\ldots,\B{x}^{(n)}\in\R^p$ from ${\mathcal N}(\B{0},\B{\Sigma}^*)$
	with $\B\Theta^* = (\B\Sigma^*)^{-1}$.
	Let $\{\beta^*_{ij}\}$ be as defined in~\eqref{lem_reg}. We assume:
	\begin{enumerate}[leftmargin=*,label=\textbf{(B\arabic*)}]
		\item \label{sigmaboundedassum} There exist $u_{\sigma}\geq l_{\sigma}> 0$ such that for any $j\in[p]$, $l_{\sigma}\leq\sigma_j^*\leq u_{\sigma}$ and $u_{\sigma}\leq 5l_{\sigma}$.
		\item\label{beta-bound} For $i,j\in[p]$, $i\neq j$, we have $|\beta^*_{ij}|\leq 1/\sqrt{k}$.
		\item \label{assum2-max} For $i,j\in[p]$, we have 
		$$(\B{\Sigma}^*)_{ij}\lesssim 1,~~~~\max_{j\in[p]} \frac{(\B{\Sigma}^*)_{jj}}{(\sigma_j^*)^2}\leq \frac{400}{7}.$$
		\item \label{betaminassum} There is a value $\beta_{\min}$ such that $\beta_{\min}\geq \sqrt{(\eta\log p)/{n}}$ for some 
		numerical constant $\eta\gtrsim u_{\sigma}^2$; and
		every nonzero $\beta^*_{ij}$ satisfies $|\beta^*_{ij}| \geq \beta_{\min}$ for all $i>j$.
		\item \label{assum2degree} For $j\in [p]$, $\left\vert\left\{i\in [p]: i\neq j, \beta^*_{ij}\neq 0\right\}\right\vert \leq k$ for some $k>0$. 
		\item \label{kappaassum} For the matrix $\B{\Theta}^*$, we assume
		$$3\geq \max_{\substack{S\subseteq[p] \\ |S|\leq 2k}}\lambda_{\max}(\B{\Sigma}^*_{S,S})\geq\min_{\substack{S\subseteq[p] \\ |S|\leq 2k}}\lambda_{\min}(\B{\Sigma}^*_{S,S})\geq \kappa^2 > 0.3$$
		where $\kappa$ is a universal constant.
	\end{enumerate}
\end{assumption}
In Assumptions~\ref{sigmaboundedassum} to~\ref{assum2-max}, we generally assume that $\B{\Theta}^*,\B{\Sigma}^*$ are bounded. Assumption~\ref{betaminassum} is a non-degeneracy condition generally needed to achieve support recovery. Such assumptions are common in the literature~\citep{wang2010information}. Assumption~\ref{assum2degree} is the sparsity assumption on the underlying model. Note that the value of $k$ does not appear in Problem~\eqref{main-sup}. Finally, Assumption~\ref{kappaassum} is a condition number assumption that also appears in earlier work. Additional discussions on our assumptions are presented in Appendix~\ref{sec:theory-discussion}.
Theorem~\ref{supthm} (see Appendix~\ref{app:proof-supthm} for proof) presents support recovery guarantees. 
\begin{theorem}\label{supthm}
	Suppose Assumptions~\ref{sigmaboundedassum} to~\ref{kappaassum} hold. Let $\{\hat{z}_{ij}\}$ be the optimal support for Problem~\eqref{main-sup} (or equivalently~\eqref{main-sup-convex}) with $\ell=l_{\sigma}^2/{3},L=\infty$. Moreover, let $\{z^*_{ij}\}$ be the binary matrix corresponding to the correct support, such that $z^*_{ij}=1\Leftrightarrow \theta^*_{ij}\neq 0$ for $i\neq j$. Then, $\hat{z}_{ij}=z^*_{ij}$ for $i\neq j\in[p]$  with high probability\footnote{An explicit expression for the probability can be found in~\eqref{supthmprob}} if
	$n= c_n k\log p$ and $\lambda=c_{\lambda} {\log p}/{n}$ for some sufficiently large universal constants $c_n, c_{\lambda}>0$.
\end{theorem}

\customparagraph{Comparison with prior work}
We note that the number of samples $n\gtrsim k\log p$ required in Theorem~\ref{supthm} for correct support recovery is minimax optimal up to logarithmic factors. To see this, note that the second term in Theorem~1 of~\citet{wang2010information} can be lower bounded as $n\gtrsim k\log(p/k)/\log k$, showing our bound is tight up to logarithmic factors. Our results also match or improve upon support recovery results of current popular methods such as Graphical Lasso or CLIME. Particularly, Theorem~1 of~\citet{ravikumar2008model} shows Graphical Lasso requires $n\gtrsim k^2\log p$ samples for correct support recovery. Theorem~7 of~\citet{cai2011constrained} shows that under the assumption $\sum_{i\in[p]}|\theta_{ij}^*|\leq \sqrt{k}$ for $j\in [p]$ (which is similar to Assumption~\ref{beta-bound} in our case), $n\gtrsim k\log p$ samples would be required for correct support recovery by CLIME. However, the correct support is recovered after post-processing the CLIME solution by a thresholding operator, while our estimator does not need such a step. This shows a useful benefit of our $\ell_0$ regularized estimator over $\ell_1$-based CLIME.

Next, we discuss how imposing symmetry constraints on the solution improves the statistical properties of our estimator. A popular approach for the sparse GGM problem is the node-wise sparse linear regression approach~\citep{meinshausenbuhlmann,misra2020information}. Here, 
one solves $p$-many sparse linear regression problems, where in the $j$-th problem, one performs a sparse least squares regression of $\B{x}_j$ on the features $\{\B{x}_i\}_{i\neq j}$. Here, the variances $\sigma_j^2$ are taken to be equal.
However, as these $p$ problems are solved independently, imposing symmetry during optimization can be tricky. Comparing our results to those of the node-wise methods allows us to quantify the benefits of using a pseudo-likelihood objective with symmetric support.
As seen in Theorem~\ref{supthm}, to achieve perfect support recovery, we require a non-degeneracy condition as given by the $\beta_{\min}$ assumption~\ref{betaminassum}. However, we note that as Problem~\eqref{main-sup} results in a symmetric support, we need non-degeneracy conditions only on half of the values $\{\beta^*_{ij}\}$ as stated in Assumption~\ref{betaminassum}. Intuitively, an error in estimating the support propagates to at least one other location (due to the symmetric support)---this means that only half of the 
$\beta^*_{ij}$ coefficients need to be non-degenerate.
On the other hand, node-wise methods~\citep{meinshausenbuhlmann,misra2020information} require non-degeneracy on all entries of $\B{\Theta}^*$, showing that our assumptions are milder. 

Finally, we note that Problem~\eqref{main-sup} involves logarithmic and quadratic over linear terms. Hence, standard techniques used to analyze sparse linear models do not apply here, requiring us to develop new proof techniques.

\section{Numerical Experiments}\label{sec:expts}
We present various numerical experiments to compare our proposed method against state-of-the-art methods in terms of computational efficiency, statistical performance, and a downstream task of portfolio optimization. In this section, we use \ourmethod~to denote our BnB solver. We  use \ourmethodnobnb~to denote our approximate solver (CD and local search, without any BnB search).
An implementation of \ourmethod~can be found at \href{https://github.com/mazumder-lab/GraphL0Learn}{https://github.com/mazumder-lab/GraphL0Learn}.

\noindent \textbf{Competing Methods:} We compare our method to the following popular and state-of-the-art algorithms for sparse graphical models: \texttt{GLASSO}~\citep{friedman2008sparse} that studies an $\ell_1$-penalized version of the maximum log-likelihood estimator via convex optimization; \texttt{CONCORD}~\citep{khare2015convex} that considers a convex approximation to pseudo-likelihood with $\ell_1$ penalization; and \texttt{CLIME}~\citep{cai2011constrained} that considers the problem of minimizing the $\ell_1$ norm of the precision matrix under $\ell_{\infty}$ data fidelity constraints.

\subsection{Synthetic Data}
We first investigate the computational and statistical performance of our proposed estimator on synthetic datasets. The data points $\B{x}^{(i)}$ for $i=1,\ldots,n$ are drawn independently from the normal distribution $\CN(\B{0},(\B{\Theta}^*)^{-1})$. Our validation set (used for tuning parameter selection) also contains 
$n$ samples from the same distribution. We consider different models for the true precision matrix $\B{\Theta}^*\in\R^{p\times p}$, as follows:
\begin{enumerate}
	\item\label{uinfsc} \textbf{Uniform Sparsity:} We let $\B{\Theta}=\B{B}+\delta \B{I}_p$ where, $\B{I}_p$ is the identity matrix 
	and $\B{B}$ is a symmetric matrix. The 
	entries of $\B{B}$ are independently set to 0.5 with probability $p_0$ and zero with probability $1-p_0$. We then make $\B{B}$ symmetric: $(\B{B}+\B{B}^{\top})/2$. We adjust the value of $\delta$ to control the condition number of $\B{\Theta}$. Finally, $\B{\Theta}^{-1}$ is normalized so that each variable has a unit variance. We set $p_0=k/(2p)$. Note that $\B{\Theta}$ has approximately $kp$ nonzero entries.
	\item \textbf{Banded Precision:} We let $\B{\Theta}=\B{B}+\delta \B{I}_p$ where the $(i,j)$-th entry of $\B{B}$ is $b_{ij}= 0.5^{|i-j|}  \bm 1(|i-j|\leq k/2)$
		where, $\bm{1}(\cdot)$ is the indicator function, and
		$k$ is the bandwidth. We set $\delta$ to control the condition number of $\B{\Theta}$.  Finally, $\B{\Theta}^{-1}$ is normalized, so each variable has a unit variance. Note that $\B{\Theta}$ has $k+1$ nonzeros per column.
	\end{enumerate}
	The results reported here are the averages of 10 independent runs. More details on our experimental setup (including the choice of $M$ in~\eqref{eqn:mio}) can be found in Appendix~\ref{app:num-details}.
	\subsubsection{Timing benchmarks}\label{num-mosek}
	We compare the runtime of our method to other estimators and show the scalability of our framework. We study the uniform sparsity precision matrix case, we set $k=10$ and the condition number to $p/40$. Our experiments are performed on a machine equipped with Intel Xeon 8260 CPU and 32GB of RAM.
	
	\noindent \textbf{Approximate Solvers:} We first study the performance of our approximate algorithm \ourmethodnobnb---we compare this runtime to that of $\ell_1$ penalized estimators \texttt{CONCORD} and \texttt{GLASSO}. We report the overall runtime for computing a path of 16 values of the tuning parameters (in our case, this is a $4\times 4$ grid for $\lambda_0,\lambda_2$) and limit the runtime of all methods to an hour.
	We set the convergence tolerance of  \texttt{CONCORD} and \texttt{GLASSO} to $1\%$. We use the default tolerance of $0.01\%$ for \ourmethodnobnb~as discussed in Appendix~\ref{app:num-details}. The results for this case are shown in Table~\ref{table:times-approx}. We observe that our approximate algorithm is generally faster than \texttt{CONCORD}, and faster than \texttt{GLASSO} for $p\geq 2500$. In fact, our approximate framework can obtain high-quality solutions for problems with $p=10,000$ in less than an hour, while both \texttt{GLASSO} and \texttt{CONCORD} fail to do so. This suggests that good feasible solutions can be found quickly using our approximate framework (see Section~\ref{mediumscale} for statistical performance comparisons).

	\begin{table}[t!]
		\begin{minipage}{0.5\textwidth}
			\footnotesize
			\centering\footnotesize
			\caption{Runtime comparison of our approximate solver \ourmethodnobnb~with \texttt{GLASSO} and \texttt{CONCORD} from Section~\ref{num-mosek}. A dash means the method did not converge in an hour. $\pm$ denotes the standard error. }
			\label{table:times-approx}
			\begin{tabular}{ cc|ccc }
				$p$ & $n$ & \ourmethodnobnb &\texttt{GLASSO}& \texttt{CONCORD} \\
				\hline\hline
				\multirow{2}{*}{$100$} & 500  & $<1$    & $<1$ & $<1$  \\
				& 1000 &$<1$ & $<1$ & $<1$ \\
				\hline
				\multirow{2}{*}{$250$} & 500  & $1.2\pm0.1$  &  $<1$& $1.2\pm0.2$ \\
				& 1000 & $1.1\pm 0.1$  &  $<1$ & $1.2\pm 0.0$ \\
				\hline
				\multirow{2}{*}{$500$} & 500  & $4.0\pm 0.7$  &  $2.1\pm 0.3$& $16.3\pm0.9$  \\
				& 1000 & $4.2\pm0.5$   &  $2.0\pm 0.2$  & $7.2\pm 0.1$ \\
				\hline
				\multirow{2}{*}{$1000$} & 500  & $21.6\pm2.4$   &  $18.4\pm0.7$ & $92.1\pm3.5$  \\
				& 1000 & $20.3\pm1.9$  &  $18.8\pm0.2$ & $151\pm15$\\
				\hline
				\multirow{2}{*}{$2500$} & 500  & $145\pm12$  &  $391\pm58$ &  $1177\pm356$ \\
				& 1000 & $178\pm8$  &  $402\pm39$ &  $2469\pm329$ \\
				\hline
				\multirow{2}{*}{$5000$} & 500  & $319\pm81$  & $1489\pm100$ & - \\
				& 1000 & $452\pm 142$  &  $1763\pm231$& -\\
				\hline
				\multirow{2}{*}{$10000$} & 500  & $1731\pm197$   &- & - \\
				& 1000 & $1948\pm261$   & - & -\\
				\hline

			\end{tabular}
		\end{minipage}
		\begin{minipage}{0.45\textwidth}
			\footnotesize
			\centering\footnotesize
			\caption{Running time benchmark of our exact solver and Mosek to $1\%$ MIP gap. A dash means Mosek did not return any lower bounds in an hour. If $1\%$ gap is not achieved in 1 hour, we report the final MIP gap within parenthesis. $\pm$ denotes the standard error.}
			\label{table:times-exact}
			\begin{tabular}{ cc|cc }
				$p$ & $n$  & \ourmethod & Mosek\\
				\hline\hline
				\multirow{2}{*}{$100$} & 500  & $18.6\pm1.5$ & $207\pm 75$\\
				& 1000  & $19.3\pm0.8$ & $245\pm69$  \\
				\hline
				\multirow{2}{*}{$250$} & 500   &   $89.4\pm6.2$  & - \\
				& 1000 &   $93.2\pm7.6$ & - \\
				\hline
				\multirow{2}{*}{$500$} & 500    & $212\pm14$  & -  \\
				& 1000 &   $203\pm 8$  & -\\
				\hline
				\multirow{2}{*}{$1000$} & 500   & $(1.6\%)$  & -\\
				& 1000 &   $(1.1\%)$ & - \\
				\hline
				\multirow{2}{*}{$2500$} & 500   & $(3.1\%)$& -  \\
				& 1000 &  $(2.4\%)$& - \\
				\hline
				\multirow{2}{*}{$5000$} & 500   & $(4.7\%)$& -  \\
				& 1000 &  $(4.3\%)$& - \\
				\hline
				
			\end{tabular}
		\end{minipage}
	\end{table}

	\noindent \textbf{Exact Solvers:} Next, we study the performance of our exact BnB framework. To this end, we find the best hyper-parameter (across 16 tuning parameters on a validation set) by using our approximate solver. On this hyper-parameter, we run (i) our BnB procedure, and (ii) Mosek. Both these global solvers are run to $1\%$ MIP gap for problem~\eqref{eqn:mio} and we report the runtimes for these two methods (we include the hyper-parameter tuning runtime as well).
	If $1\%$ MIP gap is not achieved after 1 hour, we report the final MIP gap.  
	The results for this case are shown in Table~\ref{table:times-exact}. We observe that Mosek fails to return meaningful results when $p>100$, and for $p=100$ is an order of magnitude slower than \ourmethod. On the other hand, \ourmethod~can certify $1\%$ optimality gap for problems with $p=500$. An optimality gap of less than $3\%$ or so is also achievable for $p=2500$, if we stop the BnB as the time limit is reached.  We recall that our algorithm optimizes over (approximately) $p^2/2$ binary variables encoding the sparsity pattern of the precision matrix: when $p=2500$, we deal with approximately $p^2/2\approx 3\times 10^6$ binary variables. This suggests that \ourmethod~is quite promising
	in terms of speed and efficiency for considerably large problem instances. Furthermore, we observe that the objective values returned by our approximate algorithm~\ourmethodnobnb~are within $1\%$ of the final incumbent from BnB. This shows that our approximate algorithm~\ourmethodnobnb~can serve as an independent method (without BnB search) to obtain high-quality solutions quickly, often faster than \texttt{CONCORD} and \texttt{GLASSO} as we discussed above. The solution from the approximate method can be improved along with optimality certificates using BnB tree search (see below).
	
	Finally, we present additional numerical experiments, including a deeper study of the BnB performance over time in Appendix~\ref{app:numerical}.

	\subsubsection{Statistical benchmarks }\label{mediumscale}
	We use synthetic datasets to compare the statistical performance of our estimator to other algorithms. Since we are considering several competing methods each with varying runtimes, we take a moderate value of $p=200$. In terms of performance metrics, we report the normalized estimation error  $\|\hat{\B{\Theta}}-\B{\Theta}^*\|_F/\|\B{\Theta}^*\|_F$ where $\B{\Theta}^*$ is the true precision matrix and $\hat{\B{\Theta}}$ is the estimated one. Next, we report Matthews Correlation Coefficient (MCC) which is defined as
	\begin{equation*}
		\begin{aligned}
			\text{MCC} = \frac{\text{TP}\times\text{TN} - \text{FP}\times \text{FN}}{\sqrt{(\text{TP}+\text{FP})(\text{TP}+\text{FN})(\text{TN}+\text{FP})(\text{TN}+\text{FN})}}
		\end{aligned}
	\end{equation*}
	where 
	\begin{equation*}
		\begin{aligned}
			\text{TP}=|\{(i,j): \theta^*_{ij},\hat{\theta}_{ij}\neq 0\}|&,\text{FP}=|\{(i,j): \theta^*_{ij}=0,\hat{\theta}_{ij}\neq 0\}|\\
			\text{TN}=|\{(i,j): \theta^*_{ij},\hat{\theta}_{ij}=0 \}|&,\text{FN}=|\{(i,j): \theta^*_{ij}\neq0,\hat{\theta}_{ij}= 0\}|.
		\end{aligned}
	\end{equation*}
	Note that a higher value of MCC implies a better support recovery performance. Finally, we report the support size of each estimator as $\text{NNZ}=|\{(i,j): \hat{\theta}_{ij}\neq 0\}|$. We set the MIP gap stopping criterion for \ourmethod~to $5\%$ and the runtime limit to 2 minutes. We use the same tuning parameter selection procedure discussed in Appendix~\ref{app:num-details}. For competing methods, we use the default convergence tolerance to obtain their best performance.\\
	\customparagraph{Scenario 1, Banded Precision} In this setup, consider different values of $n \in \{50,\ldots,300\}$ and set the condition number to 100, and the row/column-wise sparsity to $k=6$ (i.e, $\| \B{\Theta}^*\|_0 \approx kp$). We compare the outcomes of different methods---the results are shown in Figure~\ref{fig:synt2}. We observe that \ourmethod~provides the smallest estimation error, and the highest MCC (which implies the best support recovery), while resulting in a sparse solution. Although \texttt{CONCORD} provides good support recovery, it suffers in terms of estimation performance. \texttt{GLASSO} provides good estimation performance, however, similar to \texttt{CLIME}, leads to many false positives and larger support sizes, underperforming in support recovery performance. We also see that in these experiments, \ourmethodnobnb~generally performs well, better that the competing methods, but worse than \ourmethod. This suggests that: Our approximate solver \ourmethodnobnb~can return high-quality solutions, and our BnB search procedure can further improve the quality of the solution from \ourmethodnobnb.
	
	\begin{figure}[ht]
	\centering
	\begin{tabular}{ccc}
		\small Estimation Error & \small MCC & \small NNZ\\
		\includegraphics[width=0.28\linewidth,trim =.8cm 0cm .8cm 0cm, clip = true]{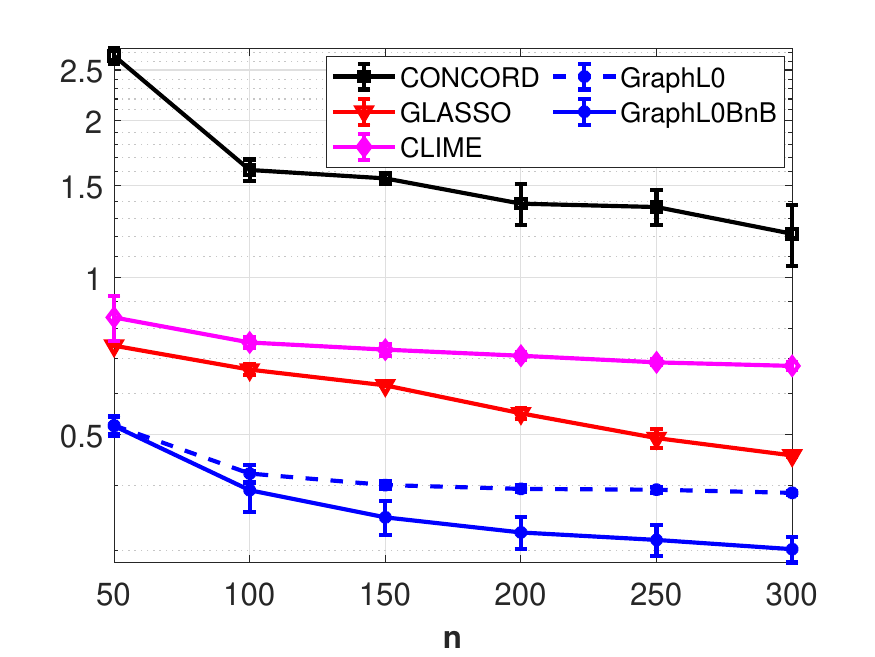}& 
		\includegraphics[width=0.28\linewidth,trim =.8cm 0cm .8cm 0cm, clip = true]{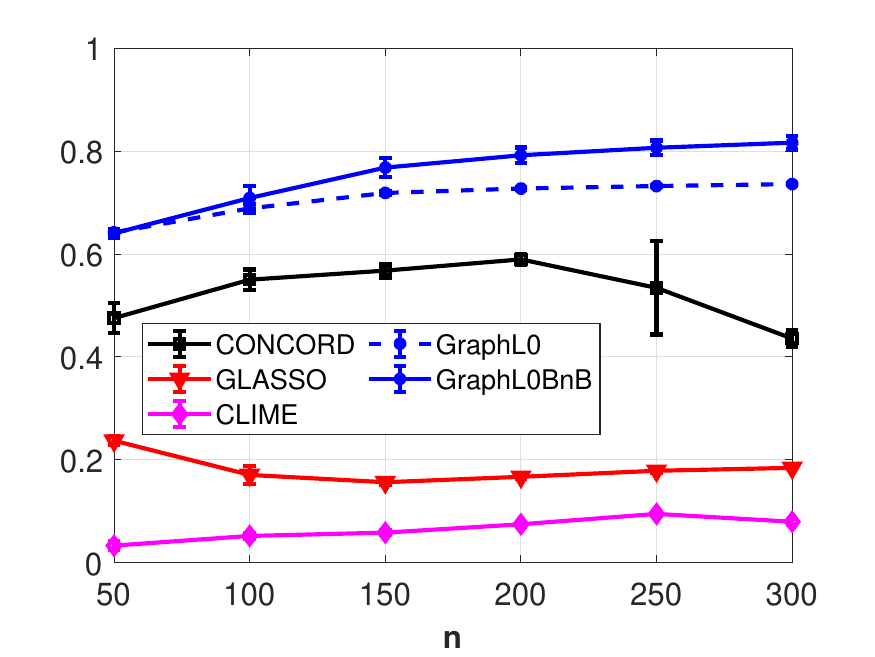}&   
		\includegraphics[width=0.28\linewidth,trim =.8cm 0cm .8cm 0cm, clip = true]{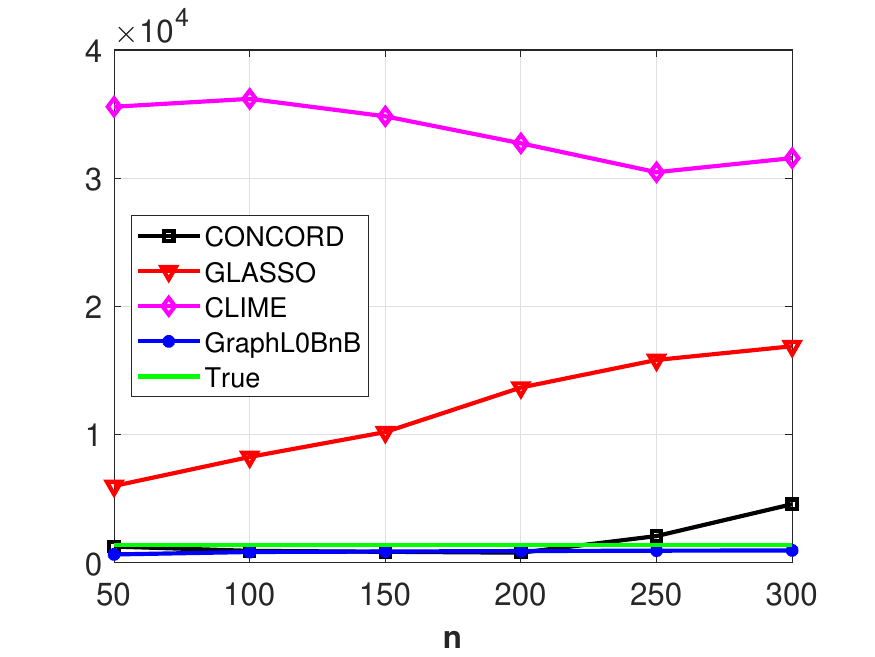} \\
	\end{tabular}
		\caption{Comparison for the banded precision model in Section~\ref{mediumscale} with $k=6$.}
	\label{fig:synt2}
\end{figure}

	\customparagraph{Scenario 2, Uniform Sparsity} Here we choose $n \in \{50,\ldots,300\}$ and set the condition number to 200, and the row/column sparsity to $k \in \{5,10\}$ (i.e., $\| \B{\Theta}^*\|_0 \approx kp$). The results for $k=5$ are shown in Figure~\ref{fig:synt3} and the results for $k=10$ can be found in Figure~\ref{fig:synt4}. Overall, it can be seen that our proposed estimator provides good estimation and support recovery performance. Moreover, our estimator is sparse, specially compared to \texttt{CLIME} and \texttt{GLASSO}. Another observation is that increasing $k$ leads to worse statistical performance, which is expected. We also see that similar to the previous setting, \ourmethodnobnb~performs quite well though \ourmethod~can offer further improvements.
	
	\begin{figure}[ht]
	\centering
	\begin{tabular}{ccc}
		\small Estimation Error & \small MCC & \small NNZ\\
		\includegraphics[width=0.28\linewidth,trim =.8cm 0cm .8cm 0cm, clip = true]{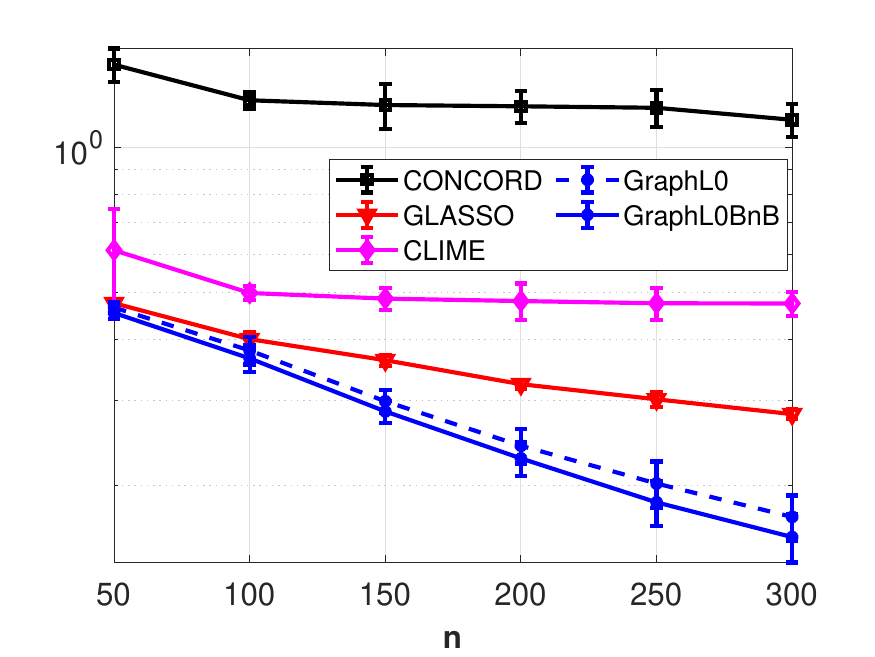}& 
		\includegraphics[width=0.28\linewidth,trim =.8cm 0cm .8cm 0cm, clip = true]{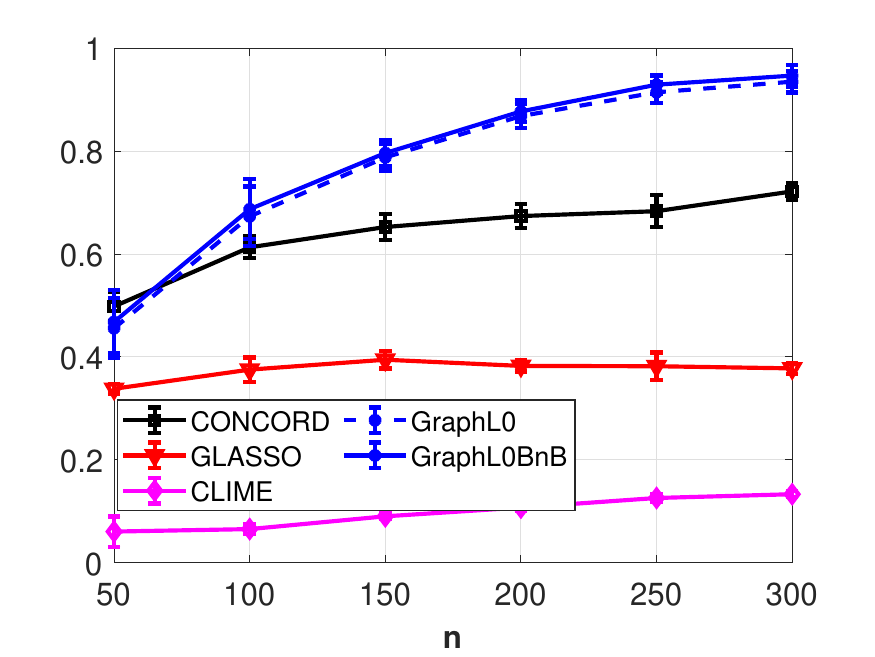}&   
		\includegraphics[width=0.28\linewidth,trim =.8cm 0cm .8cm 0cm, clip = true]{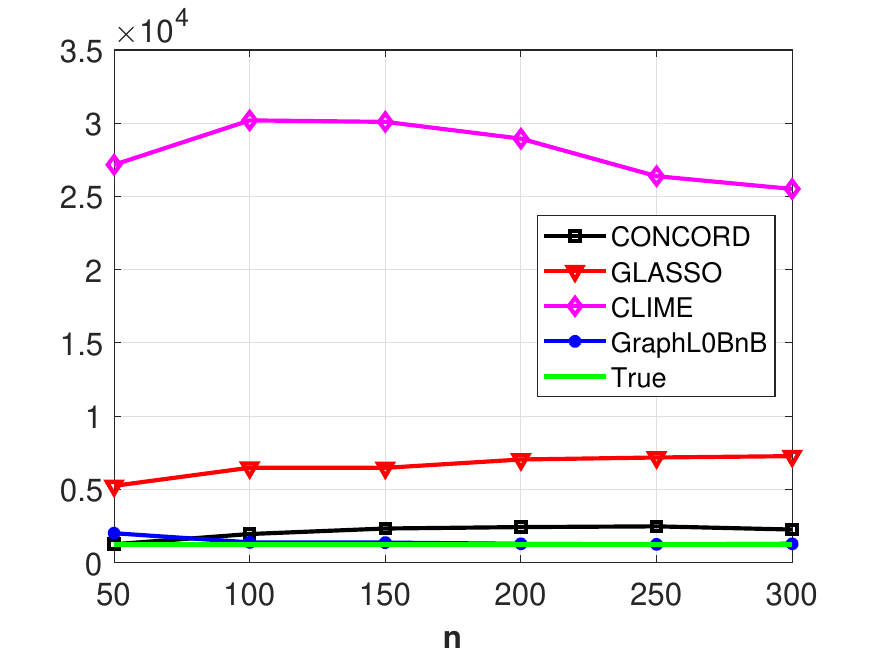} \\
	\end{tabular}
		\caption{Comparison for the uniform sparsity model in Section~\ref{mediumscale} with $k=5$ and $p=200$. }
	\label{fig:synt3}
\end{figure}

	\begin{figure}[ht]
	\centering
	\begin{tabular}{ccc}
		\small Estimation Error & \small MCC & \small NNZ\\
		\includegraphics[width=0.28\linewidth,trim =.8cm 0cm .8cm 0cm, clip = true]{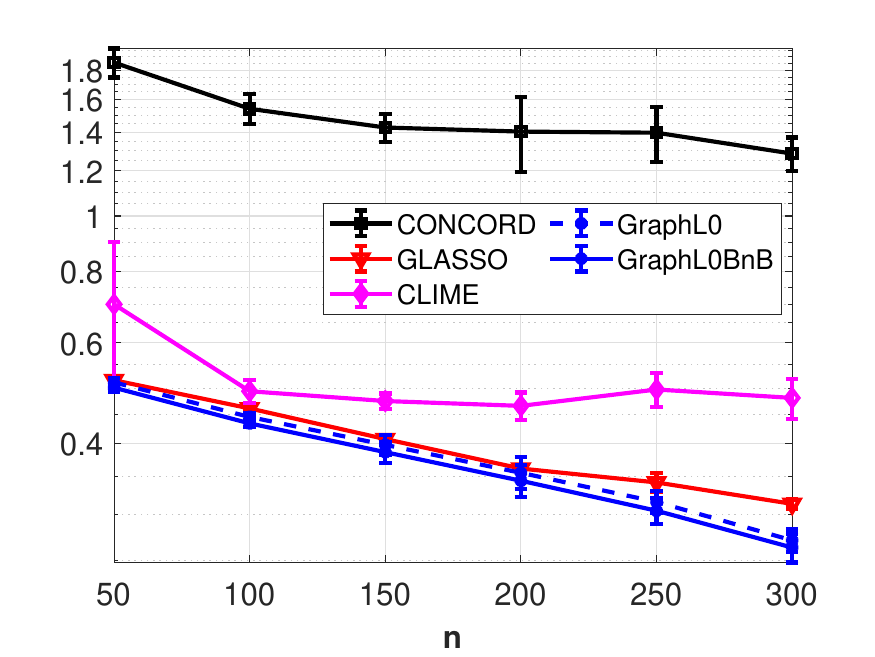}& 
		\includegraphics[width=0.28\linewidth,trim =.8cm 0cm .8cm 0cm, clip = true]{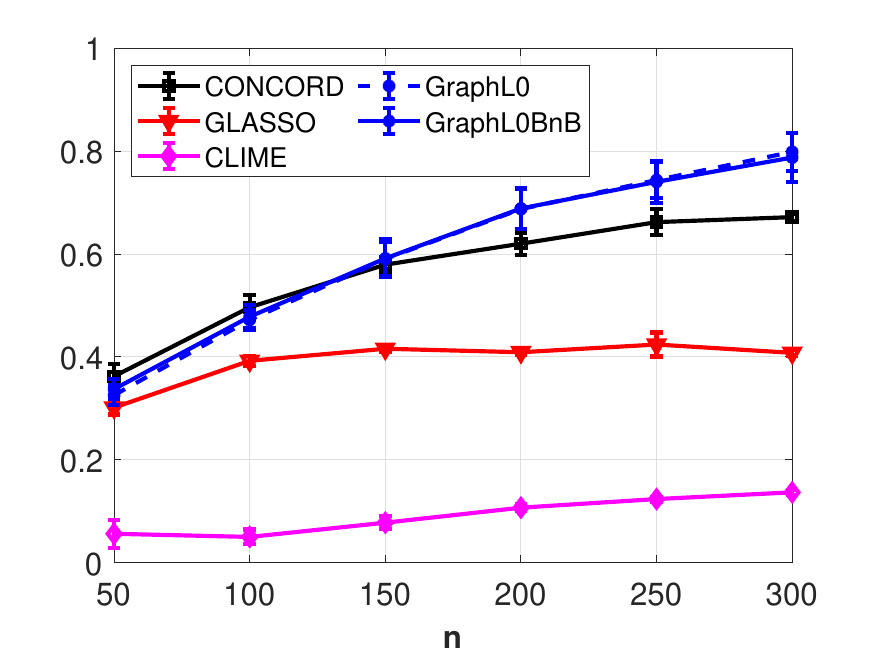}&   
		\includegraphics[width=0.28\linewidth,trim =.8cm 0cm .8cm 0cm, clip = true]{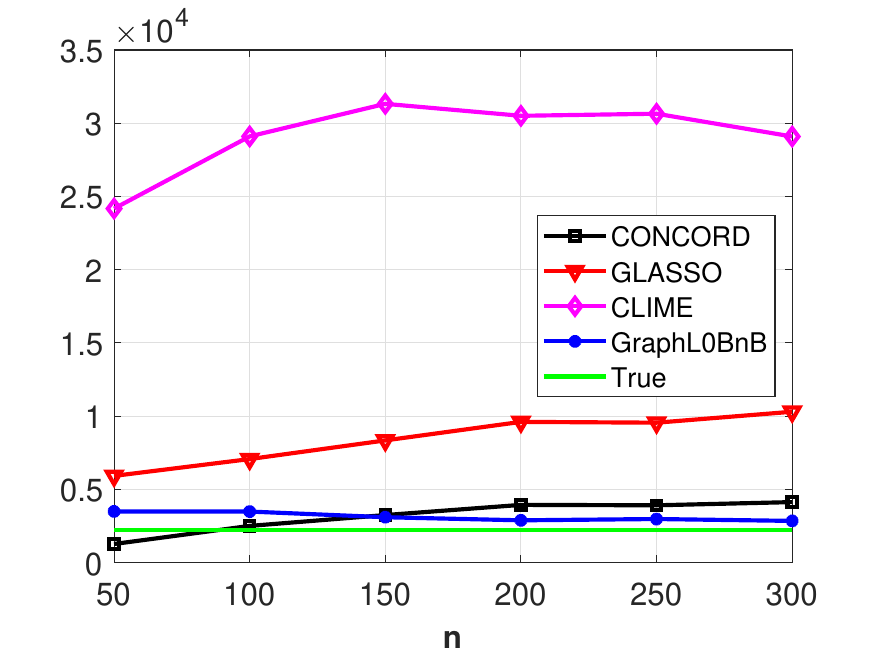} \\
	\end{tabular}
		\caption{Comparisons for the uniform sparsity model in Section~\ref{mediumscale} with $k=10$ and $p=200$. }
			\label{fig:synt4}
\end{figure}
	
	Finally, we consider 
	some high-dimensional settings with larger values of $p$. We set $p=3000$, $k=10$ and let the condition number be 150. Only \ourmethod~and \texttt{GLASSO} seem to scale to these instances. The results for this case are shown in Figure~\ref{fig:synt5}. We see that \ourmethod~leads to almost-perfect support recovery for $n\approx 1000$ while providing better estimation performance compared to \texttt{GLASSO}. Moreover, \texttt{GLASSO} incurs a fairly large number of false positives and has a dense support, as observed before.  
	
	\begin{figure}[ht]
	\centering
	\begin{tabular}{ccc}
		\small Estimation Error & \small MCC & \small NNZ\\
		\includegraphics[width=0.28\linewidth,trim =.8cm 0cm .8cm 0cm, clip = true]{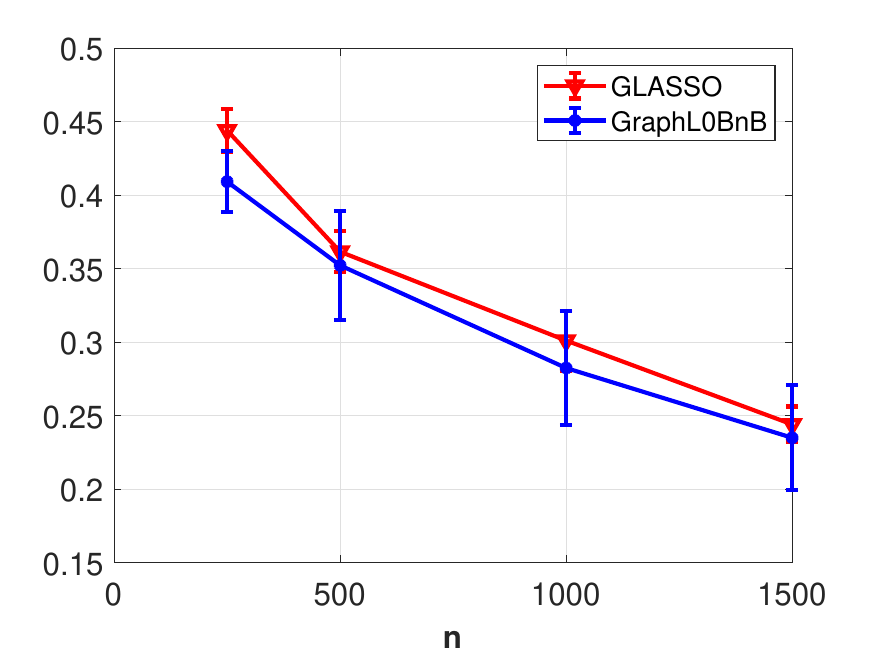}& 
		\includegraphics[width=0.28\linewidth,trim =.8cm 0cm .8cm 0cm, clip = true]{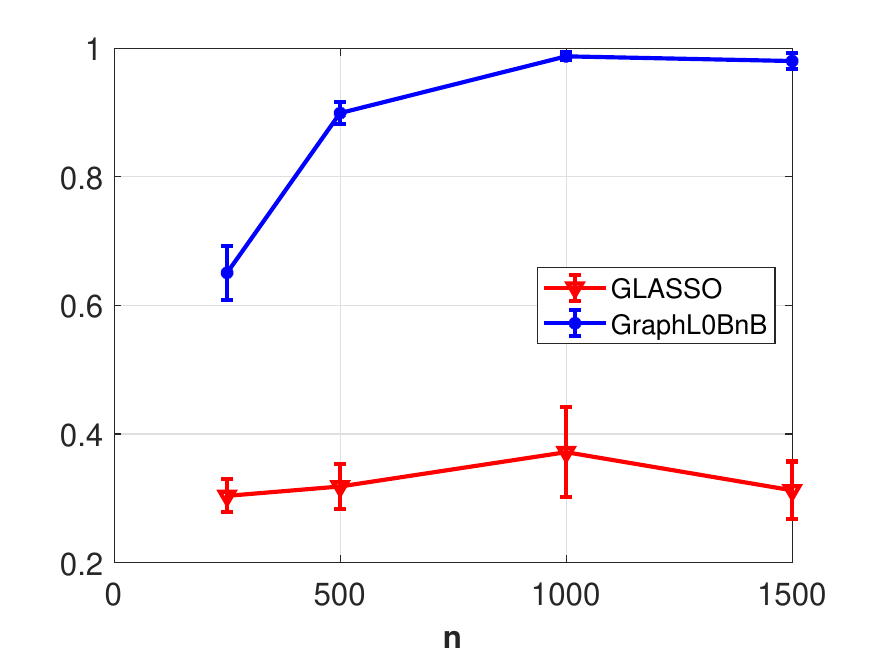}&   
		\includegraphics[width=0.28\linewidth,trim =.8cm 0cm .8cm 0cm, clip = true]{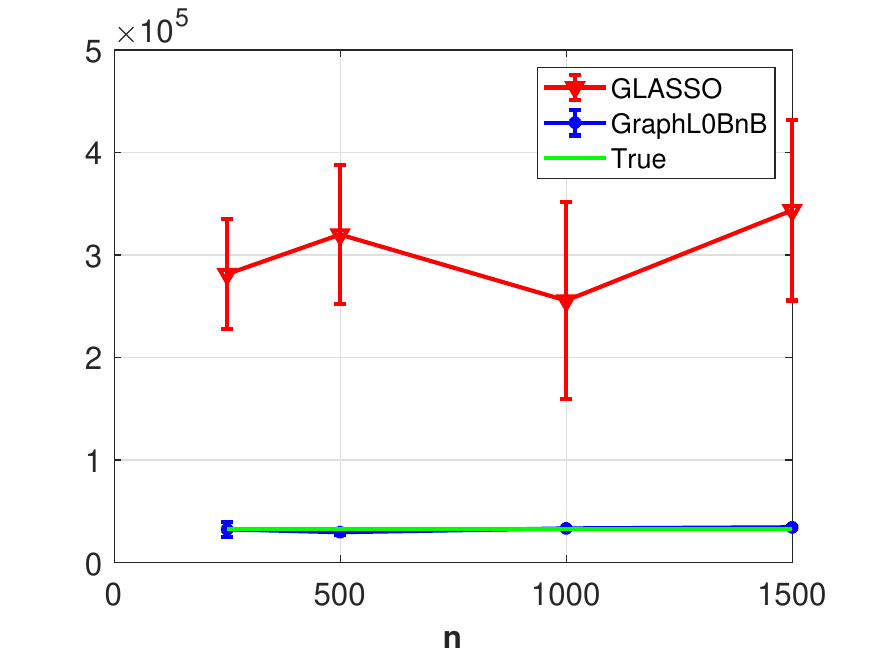} \\
	\end{tabular}
		\caption{ Comparison for the uniform sparsity model in Section~\ref{mediumscale} with $k=10$ and $p=3000$. }
	\label{fig:synt5}
\end{figure}

	\subsection{A downstream application in portfolio optimization}\label{financesec}
	We consider an application of sparse GGM in finance in the context of 
	portfolio optimization. We use data on stock returns extracted from Yahoo! Finance from 2005 to 2019 for 1452 companies. Given the data, we consider the well-known problem in portfolio optimization: we select a portfolio that leads to  maximum returns and minimum risk over the portfolio~\citep{markowitz1952portfolio}. Given the returns data matrix $\B{X}\in\R^{n\times p}$ and portfolio weights $\B{w}\in\R^p_{\geq 0}$ with $\sum_{i=1}^p w_i=1$, the values of returns and risk are defined as 
	\begin{equation}
		\begin{aligned}
			r = \sum_{i=1}^n (\B{Xw})_i,\quad
			\sigma =\sqrt{\text{VAR}(\B{Xw})},
		\end{aligned}
	\end{equation}
	respectively, where $\text{VAR}$ denotes the variance of the vector. To select the optimal portfolio, we solve the quadratic portfolio selection problem:
	\begin{align}
		\min_{\B{w}}  ~ \B{w}^{\top}\B{\Sigma}_X\B{w}~~~\text{s.t.} ~~~\B{w}\in\R_{\geq 0}^p,~~\sum_{i=1}^p w_i=1,~~\sum_{i=1}^n (\B{Xw})_i\geq \bar{r}~~ \label{portfolio}
	\end{align}
	where $\B{\Sigma}_X$ is an estimate of the covariance matrix and $\bar{r}$ is a pre-determined minimum return. 
	We explore different methods to estimate $\hat{\B{\Theta}}$ and use $\hat{\B{\Theta}}^{-1}$ as an estimate of matrix $\B{\Sigma}_X$. In this section, we let $\bar{r}=-\infty$ (there is no minimum return constraint), and additional results with different values of $\bar{r}$ can be found in Appendix~\ref{appendix:portfolio}. Under this setup, we split the data into three randomly selected subsets. We use the training split to calculate a path of solutions (over different hyper-parameters) for each sparse GGM method studied. For each method, we select $\hat{\B\Theta}$ to be the solution from the path with the smallest validation loss calculated on the validation split. We refer to Appendix~\ref{app:num-details} for more details on the validation loss and hyper-parameters. We then let $\B{\Sigma}_{X}=\hat{\B\Theta}^{-1}$ and solve~\eqref{portfolio} to obtain the optimal portfolio weights (note that when $\bar{r}=-\infty$, \eqref{portfolio} does not directly depend on $\B{X}$ so no additional data is required here). Finally, we use the test split to calculate the risk and return of the portfolio. We use 1000 training samples, 500 validation samples and 1000 test samples.

	To be able to experiment with different methods with varying runtimes, we consider two cases. In the first case, we select the top 100 stocks with highest variance over time. 
	As a baseline, we also consider estimating $\B\Sigma_X$ with the sample covariance of $\B{X}$ without using any sparse GGM method. The average results for 20 selections of train/validation/test data are reported in Table~\ref{table:real1} (the runtime for \ourmethod~is to MIP gap of $5\%$). Overall, we see that our method provides the highest return. In terms of risk, \texttt{GLASSO} has a lower risk compared to our method, and our method leads to lower risk compared to other methods. We note that 
	\ourmethod~is more sparse than \texttt{GLASSO}. Compared to \texttt{CONCORD}, our method results in a sparser precision matrix, has higher returns and lower risk. Overall, our method is performing well both statistically and computationally.

	\begin{table}[t!]
		\footnotesize
		\centering\footnotesize
		\caption{Simulation results for the real dataset with top-100 and full stocks in Section~\ref{financesec}. The baseline is using the sample covariance of the data, without using any sparse GGM method.} For more details on the setup, see~\citet{khare2015convex}.  
		\label{table:real1}
		\begin{tabular}{ cccccc|ccc}
			& \multicolumn{5}{c}{top-100} & \multicolumn{3}{c}{full data ($p=1452$)}\\
			&\ourmethod & \texttt{GLASSO}& \texttt{CONCORD}&\texttt{CLIME} & Baseline & \ourmethod &\texttt{GLASSO} & Baseline\\
			\hline\hline 
			Returns &  $25.02 $
			& $24.98$ & $24.87$ & $24.50$ & 24.02 & 8.96 &  2.50 & 1.34 \\
			Risk &  $0.38 $
			& $0.34$ & $0.41$ & $0.47$ & 0.35 &  0.36 & 0.20 & 0.68\\
			$\|\hat{\B{\Theta}}\|_0$ &  $2398 $
			& $3060$ & $107$ & $3371$ & - & 27055& 114450 & -  \\
			Runtime &  19.02
			& 0.42 &12.63 & 44.11 & - &  459 & 1470 & -  \\
			
		\end{tabular}
	\end{table}
	
	Next, we use all stocks in the dataset ($p=1452$) and repeat the same experiment. In this case, we show results from only our method and \texttt{GLASSO} reported in Table~\ref{table:real1}. (Other methods faced numerical issues). Overall, our method provides a considerably higher value of return, while providing better returns to risk ratio. This is while our solution is more sparse and our algorithm is faster, showing that our estimator works well in terms of statistical and computational performance.

\section{Conclusion}
We propose a new estimator for the sparse GGM problem based on an $\ell_0$-regularized version of the pseudo-likelihood function. Our estimator is given by the solution of a mixed integer convex program.  
We propose a global optimization framework to obtain optimal solutions to the MIP using a custom branch-and-bound solver where we use specialized first-order methods to solve the convex problems at the node relaxations. We also present fast approximate solutions for the MIP which is of independent interest. 
We demonstrate that our proposed BnB framework can deliver near-optimal solutions with dual bounds for sparse instances with $n \approx 500$ samples and $p\approx 5,000$ features (i.e., around $12\times 10^6$ binary variables) in less than an hour. We also discuss the statistical properties of our estimator and derive estimation error and variable selection guarantees that generally match or improve upon existing theoretical guarantees for other sparse GGM methods. Our numerical experiments on synthetic and financial data show promising statistical and computational performance of our proposal.

Finally, although we focused on Gaussian graphical models, there are applications where the joint distribution underlying the data is not Gaussian, e.g., when the data is discrete. In such cases the pseudo-likelihood framework can still be applied, even if the data is not Gaussian--see for example~\citep[Chapter 9]{hastie2015statistical} and~\cite{hofling2009estimation}---one has to use a loss function appropriate for the data type. We leave further exploration of our methodology for 
non-Gaussian datasets for future work.

\section*{Acknowledgement}
Wenyu Chen and Kayhan Behdin contributed to the work when they were PhD students at MIT Operations Research Center. The authors would like to thank MIT SuperCloud~\citep{reuther2018interactive} for partially providing the computational resources for this work.  This research was supported by grants from the Office of Naval Research (ONR-N000142212665).

\bibliographystyle{informs2014}
\bibliography{references}

\appendix
    \numberwithin{equation}{section}
\numberwithin{lemma}{section}
\numberwithin{proposition}{section}
\numberwithin{definition}{section}
\numberwithin{theorem}{section}
\numberwithin{figure}{section}
\numberwithin{table}{section}

\newpage
\section*{Appendices}
\section{Computation: Additional Technical Details}\label{app:comp-proofs}

\subsection{Overview of Nonlinear BnB} For completeness, we provide a brief overview of nonlinear BnB. Nonlinear BnB is a general framework for solving mixed integer
nonlinear programs~\citep{belotti2013mixed}. The algorithm starts by solving the root relaxation \eqref{eqn:root-relaxation} of Problem~\eqref{eqn:mio}. Then, the algorithm chooses a branching variable, say $z_{k\ell}$ and creates two new nodes (optimization subproblems): one with $z_{k\ell}=0$ and the other with $z_{k\ell}=1$, where all the other binary $z_{ij}$'s are relaxed to the interval $[0,1]$. The algorithm then proceeds recursively: for every unexplored node, it solves the corresponding
optimization problem and then branches on a new fractional variable (if any) to create new nodes. This leads to a search tree with nodes corresponding to
optimization subproblems and edges representing branching decisions.

While growing the search tree, BnB prunes a node when (a)~solving the relaxation at the current node results in an integral $\bm z$ or (b)~the objective of the current relaxation exceeds the best available upper bound
on \eqref{eqn:mio}.

\subsection{Properties and optimization oracles related to regularizers}\label{subsec:oracles-reg}
We discuss some technical details related to our coordinate descent algorithm.

In Sections~\ref{subsubsec:offdiag-update-derive} and \ref{subsubsec:diagonal-update}, we first present the derivations related to updates \eqref{eqn:offdiag-update} and \eqref{eqn:diag-update}, and we reduce the off-diagonal update \eqref{eqn:diag-update} to a proximal operator computation problem. In Sections~\ref{subsubsec:g-oracles} and \ref{subsubsec:L0-oracles}, we derive closed-form expressions of the proximal operators for the node/root relaxation subproblems~\eqref{eqn:relaxation} and the incumbent solving problem~\eqref{eqn:node-incumbent}.

\subsubsection{Off-diagonal update}\label{subsubsec:offdiag-update-derive}
We show that the update of $\hat\theta_{ij}$ in line 3 of Algorithm~\ref{algo:CD} is given by \eqref{eqn:offdiag-update}. For any $i<j$ and $h_{ij}$, we have
\begin{align*}
	&~F(\hat{\bm\Theta}-\hat \theta_{ij}\bm E_{ij}-\hat\theta_{ij}\bm E_{ji}+\theta_{ij}\bm E_{ij}+\theta_{ij}\bm E_{ji})\\
	=&~\text{Const}+\frac{1}{\hat\theta_{ii}}\norm{\btX \hat{\bm\theta}_i-\hat\theta_{ij}\btx_j+\theta_{ij}\btx_j}^2+\frac{1}{\hat\theta_{jj}}\norm{\btX\hat{\bm\theta}_j-\hat\theta_{ij}\btx_i+\theta_{ij}\btx_i}^2+h_{ij}(\theta_{ij})\\
	\overset{(a)}{=}&~\text{Const}+\frac{1}{\hat\theta_{ii}}\norm{\bm r_i-\hat\theta_{ij}\btx_j+\theta_{ij}\btx_j}^2+\frac{1}{\hat\theta_{jj}}\norm{\bm r_j-\hat\theta_{ij}\btx_i+\theta_{ij}\btx_i}^2+h_{ij}(\theta_{ij})\\
	\overset{(b)}{=}&~\text{Const}+\frac{\norm{\btx_j}^2}{\hat\theta_{ii}}\theta_{ij}^2+\frac{2\btx_j^\top(\bm r_i-\hat\theta_{ij}\btx_j)}{\hat\theta_{ii}}\theta_{ij}+\frac{\norm{\btx_i}^2}{\hat\theta_{jj}}\theta_{ij}^2+\frac{2\btx_i^\top(\bm r_j-\hat\theta_{ij}\btx_i)}{\hat\theta_{jj}}\theta_{ij}+h_{ij}(\theta_{ij})\\
	\overset{(c)}{=}&~\text{Const}+a_{ij}\theta_{ij}^2+b_{ij}\theta_{ij}+h_{ij}(\theta_{ij})
\end{align*}
where ``\text{Const}" denotes the constant terms (i.e., not depending on the optimization variable $\theta_{ij}$) and may vary from one line to another. Above,  $(a)$ uses the definition of $\bm r_i=\btX\hat{\bm\theta}_i$, $(b)$ expands the squared norm and moves the constant terms into \text{Const}, and $(c)$ is due to $v_i=\norm{\btx_i}^2$ and the definitions of $a_{ij}$ and $b_{ij}$.

In fact, the above update can be expressed using the proximal 
operator~\citep{beck2017first} for the regularizer $h_{ij}$ under some scaling as we discuss below. For a lower-semicontinuous function $h$, we let
\begin{equation}
	\mathcal{Q}_h(a,b) = \arg\min_{\theta}~a\theta^2+b\theta+h(\theta),
\end{equation}
and denote the proximal operator
\begin{equation}
	\prox{h}(\tilde\theta) =
	\arg\min_{\theta}~\frac12(\theta-\tilde\theta)^2+h(\theta).
\end{equation}
One can verify that
\begin{equation}\label{eqn:Q-prox}
	\mathcal{Q}_h(a,b)=\prox{\frac1{2a}h}\left(-\frac{b}{2a}\right).
\end{equation}

According to \eqref{eqn:Q-prox}, it suffices to investigate how to compute the proximal operators for the regularizers. Below we present closed-form expressions for the proximal operators arising from the different regularizers considered in Section~\ref{subsec:ASCD}.

\subsubsection{Regularizers for convex relaxations (node and root subproblems)}\label{subsubsec:g-oracles}
The results here extend those discussed in~\citet{hazimeh2020sparse}.

\noindent\textbf{Interval relaxation:} Recall that when $\underline{z}=0,\bar{z}=1$, the regularizer $g$ becomes
$$\psi(\theta;\lambda_0,\lambda,M)=\min_{z,s}~\lambda_0z+\lambda_2s,~~\text{s.t.}~~sz\geq \theta^2,~|\theta|\leq Mz,~z\in[0,1].$$
We summarize different cases of $\psi$ in Table~\ref{table:summary-psi} (see also~\eqref{eqn:psi} in main paper).

\begin{table}[H]
	\centering
	\caption{Summary of different regimes and cases of $\psi$}
	\label{table:summary-psi}
	\begin{tabular}{cccccc}
		\toprule
		Regime& Range of $|\theta|$&$\psi(\theta;\lambda_0,\lambda_2,M)$&$z^*$&$s^*$  \\
		\midrule
		\multirow{3}{*}{$\sqrt{\lambda_0/\lambda_2}\leq M$}& $[0,\sqrt{\lambda_0/\lambda_2})$& $2\sqrt{\lambda_0\lambda_2}|\theta|$&$\sqrt{\lambda_2/\lambda_0}|\theta|$&$\sqrt{\lambda_0/\lambda_2}|\theta|$\\
		&$(\sqrt{\lambda_0/\lambda_2},M]$&$\lambda_0+\lambda_2\theta^2$&$1$&$\theta^2$\\
		&$(M,\infty)$&$\infty$&$\varnothing$&$\varnothing$\\\midrule
		\multirow{2}{*}{$\sqrt{\lambda_0/\lambda_2}>M$}&$[0,M]$&$(\lambda_0/M+\lambda_2M)|\theta|$&$|\theta|/M$&$|\theta|M$\\
		&$(M,\infty)$&$\infty$&$\varnothing$&$\varnothing$\\
		\bottomrule
	\end{tabular}
\end{table}

Given non-negative parameters $\lambda$ and $M$, we define the 
{\it{boxed soft-thresholding operator}} $\mathcal{T}:\R\to\R$ as follows:
\begin{equation}
	\mathcal{T}(x;\lambda,M)=\left\{
	\begin{array}{ll}
		0& ~\text{if}~ |x|\leq \lambda  \\
		(|x|-\lambda)\sign(x)& ~\text{if}~\lambda\leq |x|\leq \lambda +M\\
		M\sign(x)&~\text{otherwise.}
	\end{array}
	\right.
\end{equation}
Note that $\mathcal{T}(x;\lambda,M)$ is the proximal operator for the boxed $\ell_1$ regularizer $$h(x)=\lambda|x|+\chi\{|x|\leq M\}.$$

Then, according to \citet{hazimeh2020sparse}, the proximal operator of $\psi$ is given by 
\begin{align}
	\prox{\psi}(\tilde\theta;\lambda_0,\lambda_2,M)&=\arg\min_\theta~\frac12(\theta-\tilde\theta)^2+\psi(\theta;\lambda_0,\lambda_2,M)\label{eqn:prox-psi}\\
	&=\left\{
	\begin{array}{ll}
		\mathcal{T}(\tilde\theta;2\sqrt{\lambda_0\lambda_2},M)&  ~\text{if}~|\tilde\theta|\leq 2\sqrt{\lambda_0\lambda_2}+\sqrt{\lambda_0/\lambda_2}~\text{and}~\sqrt{\lambda_0/\lambda_2}\leq M\\
		\mathcal{T}(\tilde\theta/(1+2\lambda_2);0,M)&~\text{if}~|\tilde\theta|> 2\sqrt{\lambda_0\lambda_2}+\sqrt{\lambda_0/\lambda_2}~\text{and}~\sqrt{\lambda_0/\lambda_2}\leq M\\
		\mathcal{T}(\tilde\theta;\lambda_0/M+\lambda_2M,M)&~\text{if}~\sqrt{\lambda_0/\lambda_2}>M.
	\end{array}
	\right. \nonumber
\end{align}

Based on this, we define the following quadratic minimization oracle
\begin{equation}\label{eqn:Q-psi}
	\mathcal{Q}_\psi(a,b;\lambda_0,\lambda_2,M):=\arg\min_x~ax^2+bx+\psi(x;\lambda_0,\lambda_2,M) = \prox{\psi}\left(-\frac{b}{2a};\frac{\lambda_0}{2a},\frac{\lambda_2}{2a},M\right).
\end{equation}

\noindent\textbf{Fixed $z$:} Recall that when $\underline{z}=\bar{z}=z\in\{0,1\}$, the regularizer $g$ becomes $\varphi$ in \eqref{eqn:phi}, i.e.
\begin{align}
	\varphi(\theta;z,\lambda_0,\lambda_2,M)&:=\min_{s}~~\lambda_0z+\lambda_2s\nonumber\\
	&~~~~~~\text{s.t.}~~sz\geq \theta^2, |\theta|\leq Mz, z\in[0,1]\nonumber\\
	&=\left\{\begin{array}{ll}
		0 &~\text{if}~ z=0~\text{and}~|\theta|=0  \\
		\infty & ~\text{if}~ z=0~\text{and}~|\theta|>0\\
		\lambda_0+\lambda_2\theta^2 &~\text{if}~ z=1~\text{and}~|\theta|\leq M\\
		\infty&~\text{if}~ z=1~\text{and}~|\theta|>M,
	\end{array}\right.
\end{align}
and its corresponding proximal operator is
\begin{align}
	\prox{\varphi}(\tilde\theta;z,\lambda_0,\lambda_2,M)&=\arg\min_\theta~\frac12(\theta-\tilde\theta)^2+\varphi(\theta;z,\lambda_0,\lambda_2,M)\nonumber\\
	&=\left\{
	\begin{array}{ll}
		0&  ~\text{if}~z=0\\
		\mathcal{T}(\tilde\theta/(1+2\lambda_2);0,M)&~\text{if}~z=1.
	\end{array}
	\right.\label{eqn:prox-phi}
\end{align}

Based on this, we define the following regularized quadratic optimization problem:
\begin{equation}\label{eqn:Q-phi}
	\begin{aligned}
		\mathcal{Q}_\varphi(a,b;z,\lambda_0,\lambda_2,M):=&\arg\min_x~ax^2+bx+\varphi(x;z,\lambda_0,\lambda_2,M)\\
		=&\prox{\varphi}\left(-\frac{b}{2a};z,\frac{\lambda_0}{2a},\frac{\lambda_2}{2a}, M\right).
	\end{aligned}
\end{equation}

\subsubsection{$\ell_0\ell_2$ regularizers}\label{subsubsec:L0-oracles}
We derive the closed-form expression for the proximal operator corresponding to the $\ell_0\ell_2$ regularizer (i.e., a weighted sum of the $\ell_0$ and squared $\ell_2$ penalty):
$$h(\theta)=\lambda_0\bm1\{\theta\neq 0\}+\lambda_2\theta^2+\chi\{|\theta|\leq M\},$$
where $M \in (0, \infty]$, and $M$ could take the value $\infty$.

The proximal operator of $h$ is 
$$\prox{h}(\tilde\theta;\lambda_0,\lambda_2,M)=\arg\min_{|\theta|\leq M} q (\theta):=\frac12(\theta-\tilde\theta)^2+\lambda_0\bm1\{\theta\neq 0\}+\lambda_2\theta^2.$$

When $\theta=0$, we have $q(0)=(1/2)\tilde\theta^2$; when $\theta\neq 0$, we have $q(\theta)= \lambda_0+\lambda_2\theta^2+\frac12(\theta-\tilde\theta)^2$, which is minimized at $\theta'=\min\{|\tilde\theta|/(1+2\lambda_2),M\}\sign(\tilde\theta)$.

Without loss of generality, we assume $\tilde\theta>0$. If $\frac{\tilde\theta}{1+2\lambda_2}>M$, then $\theta'=M$, and
$$q(\theta')=\lambda_0+\lambda_2M^2+\frac{1}{2}(M-\tilde\theta)^2.$$
The root of $q(\theta')=q(0)$ is $\tilde\theta=(\frac12+\lambda_2)M+\frac{\lambda_0}M$.

If, on the other hand, $\frac{\tilde\theta}{1+2\lambda_2}\leq M$, then $\theta'=\frac{\tilde\theta}{1+2\lambda_2}$, and
$$q(\theta')= \lambda_0+\frac{\lambda_2\tilde\theta^2}{1+2\lambda_2}.$$
The root of $q(\theta')=q(0)$ is $\tilde\theta=\sqrt{2\lambda_0(1+2\lambda_2)}$.

Putting together the pieces, we obtain the following closed-form expression for the proximal operator corresponding to the $\ell_0\ell_2$ regularizer 
\begin{equation}\label{eqn:prox-L0L2}
	\begin{aligned}
		\prox{h}(\tilde\theta;\lambda_0,\lambda_2,M)=&\arg\min_{|\theta|\leq M} q (\theta)\\
		=&\left\{\begin{array}{ll}
			\{M\sign(\tilde\theta)\}, &\text{if}~|\tilde \theta|>\max\left\{(\frac12+\lambda_2)M+\frac{\lambda_0}{M},(1+2\lambda_2)M\right\}\\
			\{0,M\sign(\tilde\theta)\}, &\text{if}~ |\tilde\theta|=(\frac12+\lambda_2)M+\frac{\lambda_0}{M}>(1+2\lambda_2)M\\
			\{0\},&\text{if}~|\tilde\theta|\in \left((1+\lambda_2)M,(\frac12+\lambda_2)M+\frac{\lambda_0}{M}\right)\\
			\{\frac{\tilde\theta}{1+2\lambda_2}\},&\text{if}~|\tilde\theta|\in \left(\sqrt{2\lambda_0(1+2\lambda_2)},(1+2\lambda_2)M\right]\\
			\{0,\frac{\tilde\theta}{1+2\lambda_2}\},&\text{if}~|\tilde\theta|= \sqrt{2\lambda_0(1+2\lambda_2)}\leq (1+2\lambda_2)M\\
			\{0\},&\text{if}~|\tilde\theta|< \min\{\sqrt{2\lambda_0(1+2\lambda_2)}, (1+2\lambda_2)M\}.
		\end{array}\right.
	\end{aligned}
\end{equation}
Note that in the special case when $M=\infty$, \eqref{eqn:prox-L0L2} (i.e., the last three conditions) recovers the closed-form expression for $\ell_0\ell_2$ regularizer provided in \citet{hazimeh2020fast}.

\subsubsection{Diagonal update}\label{subsubsec:diagonal-update}

We show that the update of $\hat\theta_{ii}$ in line 6 of Algorithm~\ref{algo:CD} is given by \eqref{eqn:diag-update}. For any $i$,  we have
\begin{align*}
	F(\hat{\bm\Theta}-\hat\theta_{ii}\bm{E}_{ii}+\theta_{ii}\bm{E}_{ii})&=\text{Const}-\log\theta_{ii}+\frac{1}{\theta_{ii}}\norm{\btX\hat{\bm\theta}_i-\hat{\theta}_{ii}\btx_i+\theta_{ii}\btx_i}^2\\
	&\overset{(a)}{=}\text{Const}-\log\theta_{ii}+\frac{1}{\theta_{ii}}\norm{\bm r_i-\hat{\theta}_{ii}\btx_i+\theta_{ii}\btx_i}^2\\
	&\overset{(b)}{=}\text{Const}-\log\theta_{ii}+\frac{1}{\theta_{ii}}(\norm{\bm e_i}^2+2\theta_{ii}\bm e_i^\top\btx_i+\theta_{ii}^2\norm{\btx_i}^2)\\
	&\overset{(c)}{=}\text{Const}-\log\theta_{ii}+\frac{\norm{\bm e_i}^2}{\theta_{ii}}+\theta_{ii}v_i,
\end{align*}
where ``\text{Const}" denotes the constant terms (similar notation as earlier); $(a)$ and $(b)$ uses the  definitions of $\bm r_i=\btX\hat{\bm\theta}_i$ and $\bm e_i=\bm r_i-\hat\theta_{ii}\btx_i$, and $(c)$ is due to $v_i=\norm{\btx_i}^2$ and $2\bm e_i^\top\btx_i$ absorbed into \text{Const}.

Since the function 
$$\theta_{ii} \mapsto F(\hat{\bm\Theta}-\hat\theta_{ii}\bm{E}_{ii}+\theta_{ii}\bm{E}_{ii}) $$
is convex in $\theta_{ii}$, by considering the first-order optimality condition and taking the positive root, we get
$$\arg\min_{\theta_{ii}}F(\hat{\bm\Theta}-\hat\theta_{ii}\bm{E}_{ii}+\theta_{ii}\bm{E}_{ii})=\arg\min_\theta -\log\theta+\theta v_i+\frac{\norm{\bm e_i}^2}{\theta}=\frac{1+\sqrt{1+4v_i\norm{\bm e_i}^2}}{2v_i}.$$

\subsection{Convergence guarantee of Algorithm~\ref{algo:CD}}\label{subsec:convergence-guarantee-proof}
In this section, we present a general convergence statement for Algorithm~\ref{algo:CD} that applies to the unified formulation \eqref{eqn:unified}, with the following assumption on $h_{ij}$:
\begin{assumption}\label{assum:reg-assumption}
	Assume that for each $1\leq i<j\leq p$, $h_{ij}(\theta)$ is convex in $\theta$. In addition, there exist two constants $c_1,c_2\geq0$ with $c_1+c_2>0$, such that for any $1\leq i<j\leq p$, we have
	$$h_{ij}(\theta)\geq \min\{c_1|\theta|,c_2\theta^2\}.$$
\end{assumption}
It is easy to see that the usual $\ell_1$, $\ell_2$ (squared) penalties and their nonegative weighted combinations satisfy Assumption~\ref{assum:reg-assumption}. The following proposition states that the relaxation regularizer $g$ (see definition~\eqref{eqn:relaxation}) also satisfies this assumption.
\begin{proposition}\label{prop:g-assumption}
	For any $\underline{z}_{ij}\leq \bar{z}_{ij}\in\{0,1\}$, $g(\theta;\lambda_0,\lambda_2,M,\underline{z}_{ij},\bar z_{ij})$ satisfies Assumption~\ref{assum:reg-assumption} with $c_1=2\sqrt{\lambda_0\lambda_2}, c_2=0$.
\end{proposition}
\begin{proof}[\textbf{Proof of Proposition~\ref{prop:g-assumption}}.]
Based on the definition of $\varphi$ and $\psi$ in different cases, using the inequality $a^2+b^2\geq 2ab$, one can verify that $\psi(\theta;\lambda_0,\lambda_2,M)\geq 2\sqrt{\lambda_0\lambda_2}|\theta|$ and $\varphi(\theta;z,\lambda_0,\lambda,M)\geq 2\sqrt{\lambda_0\lambda_2}|\theta|$.
\end{proof}

\begin{lemma}\label{lem:coercive}
	Under Assumption~\ref{assum:reg-assumption}, given any $U\geq F^*=\min_{\bm\Theta\in\mathbb{S}^p} F(\bm\Theta)$, there exist constants $u_\theta\geq l_\theta>0$ and $u_r>0$, such that for any $\bm\Theta$ with  $F(\bm\Theta)\leq U$, we have
	$$l_\theta\leq \theta_{ii}\leq u_\theta\quad\text{and}\quad \norm{\btX\bm\theta_i}\leq u_r$$
	for all $i\in[p]$.
\end{lemma}
\begin{proof}[\textbf{Proof of Lemma~\ref{lem:coercive}}.]
For any $\bm\Theta$ such that $F(\bm\Theta)\leq U$, let $k=\arg\max_{i\in[p]}\theta_{ii}$. Then, we have
\begin{equation}\label{eqn:theta-bounds-ineq1}
	\begin{aligned}
		U&\geq F(\bm\Theta)\\
		&=\sum_{i=1}^p\left(-\log\theta_{ii}+\frac{1}{\theta_{ii}}\norm{\btX\bm\theta_i}^2\right)+\sum_{i<j}h_{ij}(\theta_{ij})\\
		&\geq -p\log\theta_{kk}+\frac{1}{\theta_{kk}}\norm{\btX\bm\theta_k}^2+\sum_{i\neq k}\min\{c_1|\theta_{ik}|,c_2\theta_{ik}^2\},
	\end{aligned}
\end{equation}
where the last line is because (i) $-\log\theta_{ii}\geq -\log\theta_{kk}$ by definition of $k$; (ii) by Assumption~\ref{assum:reg-assumption}, $h_{ij}(\theta_{ij})\geq \min\{c_1|\theta_{ij}|,c_2\theta_{ij}^2\}\geq 0$; (iii) $\frac{1}{\theta_{ii}}\norm{\btX\bm\theta_i}^2$ is nonnegative for any $i\neq k$.

Now define $\bm\beta_k\in\R^p$, with $\beta_{kk}=0$ and $\beta_{ki}=-\theta_{ik}/\theta_{kk}$, then we can rewrite \eqref{eqn:theta-bounds-ineq1} as 
\begin{align}
	U&\geq -p\log\theta_{kk}+\frac{1}{\theta_{kk}}\norm{\theta_{kk}\btx_k-\btX\theta_{kk}\bm\beta_k}^2+\sum_{i\neq k}\min\{c_1|\theta_{ik}|,c_2\theta_{ik}^2\}\nonumber\\
	&= -p\log\theta_{kk}+\theta_{kk}\norm{\btx_k-\btX\bm\beta_k}^2+\sum_{i\neq k}\min\{c_1\theta_{kk}|\beta_{ki}|,c_2\theta_{kk}^2\beta_{ki}^2\}\nonumber\\
	&\overset{(A)}{\geq} -p\log\theta_{kk}+\theta_{kk}\max\left\{\frac12\norm{\btx_k}^2-\norm{\btX\bm\beta_k}^2,0\right\}+\sum_{i\neq k}\min\{c_1\theta_{kk}|\beta_{ki}|,c_2\theta_{kk}^2\beta_{ki}^2\}\nonumber\\
	&\overset{(B)}{\geq} -p\log\theta_{kk}+\theta_{kk}\max\left\{\frac12s_{\min}-L_{\btX}^2\norm{\bm\beta_k}^2,0\right\}+\sum_{i\neq k}\min\{c_1\theta_{kk}|\beta_{ki}|,c_2\theta_{kk}^2\beta_{ki}^2\},\label{eqn:theta-bounds-ineq2}
\end{align}
where inequality $(A)$ uses the fact that $\norm{\bm a+\bm b}^2\leq 2(\norm{\bm a}^2+\norm{\bm b}^2)$ with $\bm a = \btx_k-\btX\bm\beta_k$ and $\bm b=\btX\bm\beta_k$; in inequality $(B)$, we define $s_{\min}=\min_iv_i=\min_i\norm{\btx_i}^2>0$ and $L_{\btX}=\norm{\btX}$.

Now if $\norm{\bm\beta_k}\leq \epsilon:=L_{\btX}\sqrt{s_{\min}}/2$, then it follows from \eqref{eqn:theta-bounds-ineq2} that 
$$-p\log\theta_{kk}+\frac14s_{\min}\theta_{kk}\leq U,$$ from which we can deduce that there exists $u_1>0$ such that $\theta_{kk}\leq u_1$. On the other hand, if $\norm{\bm\beta_k}>\epsilon$, then there exists a $j$ such that $|\beta_{kj}|\geq \epsilon/\sqrt{p}$. Again by \eqref{eqn:theta-bounds-ineq2}, we have 
$$-p\log\theta_{kk}+\min\{c_1\theta_{kk}\epsilon/\sqrt{p}, c_2\theta_{kk}\epsilon^2/p\}\leq U,$$ and thus there exists $u_2>0$, such that $\theta_{kk}\leq u_2$. Therefore, $\theta_{kk}=\max_i\theta_{ii}\leq \max\{u_1,u_2\}$, and by taking $u_\theta=\max\{u_1,u_2\}$, we get the upper bound on $\theta_{ii}$'s.

As for the lower bound on $\theta_{ii}$, let $\ell=\arg\min_i\theta_{ii}$. By nonnegativity of $\frac{1}{\theta_{ii}}\norm{\btX\bm\theta_i}^2$ and $h_{ij}$'s, we have 
$$U\geq -\sum_{i=1}^p\log\theta_{ii}=-\log\theta_{\ell\ell} -\sum_{i\neq \ell}\log \theta_{ii}\geq -\log\theta_{\ell\ell}-(p-1)\log u_\theta.$$
Therefore, $\theta_{\ell\ell}=\min_i\theta_{ii}\geq \exp(-M-(p-1)\log u_\theta)$, and we obtain the lower bound $l_\theta=\exp(-M-(p-1)\log u_\theta)$.

To obtain the upper bound on $\norm{\btX\bm\theta_i}$, again by nonnegativity of $\frac{1}{\theta_{ii}}\norm{\btX\bm\theta_i}^2$ and $h_{ij}$'s, we have for any $j$:
$$U\geq -\sum_{i=1}^p\log\theta_{ii}+\frac{1}{\theta_{jj}}\norm{\btX\bm\theta_j}^2=-p\log u_\theta+\frac{1}{u_\theta}\norm{\btX\bm\theta_j}^2.$$
Therefore, $\norm{\btX\bm\theta_j}^2\leq u_\theta(U+p\log u_\theta)$, and we obtain the upper bound as stated in the lemma with $u_r=\sqrt{u_\theta(U+p\log u_\theta)}$.
\end{proof}

\begin{corollary}\label{cor:smooth-stronglyconvex}
	Let us denote:
	$$f(\bm\Theta)=\sum_{i=1}^p\left(-\log\theta_{ii}+\frac{1}{\theta_{ii}}\norm{\btX\bm\theta_i}^2\right).$$ 
	Under Assumption~\ref{assum:reg-assumption}, given any $F^{(0)}\geq F^*=\min_{\bm\Theta\in\mathbb{S}^p}F(\bm\Theta)$, let us define $\Xi=\{\bm\Theta\in\mathbb{S}^p:F(\bm\Theta)\leq F^{(0)}\}\subseteq\mathbb{S}^p$. Then, over $\Xi$,
	\begin{enumerate}[label=(\alph*)]
		\item There exists a scalar $L>0$ such that $\nabla f$ is $L$-Lipschitz.
		\item  There exist scalars $L_{ij}>0$ such that $\nabla f$ is $\{L_{ij}\}$-coordinatewise Lipschitz. That is, for $\B\Theta\in\Xi$, the function
		$$x\mapsto \nabla_{\theta_{ij}} f(\B\Theta+ x(\B{E}_{ij}+\B{E}_{ji})),~~~\text{with}~~~\B\Theta+ x(\B{E}_{ij}+\B{E}_{ji})\in\Xi $$
		is $L_{ij}$-Lipschitz where $\B{E}_{ij}$ denotes the matrix with coordinate $(i,j)$ equal to one and other coordinates set to zero.
		\item There exist constants $\mu_{ij}>0$ such that the objective function $F(\bm\Theta)$ in~\eqref{eqn:unified} is $\{\mu_{ij}\}$-coordinatewise strongly convex. That is, for $\B\Theta\in\Xi$, the function
		$$x\mapsto F(\B\Theta+ x(\B{E}_{ij}+\B{E}_{ji})),~~~\text{with}~~~\B\Theta+ x(\B{E}_{ij}+\B{E}_{ji})\in\Xi $$
		is $\mu_{ij}$ strongly convex.
	\end{enumerate}
\end{corollary}
\begin{proof}[\textbf{Proof of Corollary~\ref{cor:smooth-stronglyconvex}}.]
We will show (b) and (c), and (a) follows from (b).

From the derivation of the off-diagonal update in Section~\ref{subsubsec:offdiag-update-derive}, we can easily see that the second derivative of $f$ with respect to the off-diagonal entry $\theta_{ij}$ (for any $i<j$) is given by
\begin{equation}\label{eqn:second-derivative-off-diag}
	\nabla_{\theta_{ij}}^2 f(\B\Theta) = \frac{v_i}{\theta_{jj}}+\frac{v_j}{\theta_{ii}}
\end{equation}
with $v_i=\btx_i^\top\btx_i$.

From the derivation of the diagonal update in Section~\ref{subsubsec:diagonal-update}, we see that the second derivative of $f$ with respect to a diagonal entry $\theta_{ii}$ (for any $i$) is given by
\begin{equation}\label{eqn:second-derivative-diagonal}
	\nabla^2_{\theta_{ii}}f(\B\Theta)=\frac{1}{\theta_{ii}^2}+\frac{\norm{\btX_{-i}\bm\theta_{i,-i}}^2}{2\theta_{ii}^2}
\end{equation}
where $\btX_{-i}\in\R^{n\times(p-1)}$ is the data matrix $\btX$ without the $i$-th column, and $\bm\theta_{i,-i}\in\R^{p-1}$ is the vector $\bm\theta_i$ without its $i$-th component.

\noindent {\it Proof of Part \textbf{(b)}} By Lemma~\ref{lem:coercive}, we have $\theta_{ii}\geq l_\theta$, and it follows from \eqref{eqn:second-derivative-off-diag} that $\nabla_{\theta_{ij}}^2f(\B\Theta)\leq (v_i+v_j)/l_\theta$. Therefore, we have $\nabla f$ is $L_{ij}$-Lipschitz with respect to $\theta_{ij}$, where $L_{ij}=(v_i+v_j)/l_\theta$.

Again by Lemma~\ref{lem:coercive}, we have $\theta_{ii}\geq l_\theta$ and $\norm{\btX\bm\theta_i}^2\leq u_r^2$, and thus
$$\begin{aligned}
	\nabla_{\theta_{ii}}^2f(\B\Theta)=&\frac1{\theta_{ii}^2}+\frac{\norm{\btX\bm\theta_i-\theta_{ii}\btx_i}^2}{2\theta_{ii}^2}\\
	\leq& \frac{1+\norm{\btX\bm\theta_i}^2+\theta_{ii}^2\norm{\btx_i}^2}{\theta_{ii}^2} \\
	\leq& \frac{1+u_r^2}{l_\theta^2}+v_i.
\end{aligned}
$$
Therefore, we have $\nabla f$ is $L_{ii}$-Lipschitz with respect to $\theta_{ii}$, where $L_{ii}=(1+u_r^2)/l_\theta^2+v_i$.

\noindent {\it Proof of Part \textbf{(c)}} By Lemma~\ref{lem:coercive}, we have $\theta_{ii}\leq u_\theta$, so $\nabla_{\theta_{ij}}^2f(\B\Theta)\geq (v_i+v_j)/u_\theta$. Therefore, we have $F$ is $\mu_{ij}$-strongly convex with respect to every coordinate $\theta_{ij}$, where $\mu_{ij}=(v_i+v_j)/u_\theta$.

Similarly, we have $F$ is $\mu_{ii}$-strongly convex with respect to $\theta_{ii}$ with $\mu_{ii}=1/u_\theta^2$.
\end{proof}

\begin{theorem}\label{thm:convergence-general}
	Under Assumption~\ref{assum:reg-assumption}, given any initialization $\bm\Theta^{(0)}$, let $\bm\Theta^{(t)}$ be the $t$-th iterate generated by Algorithm~\ref{algo:CD}. Then there is a constant $C$ that depends on $\bm\Theta^{(0)}$, such that for any $t\geq 1$,
	$$F(\bm\Theta^{(t)})-F^*\leq \frac{C}{t},$$
	where $F^*=\min_{\bm\Theta\in\mathbb{S}^p}F(\bm\Theta)$.
\end{theorem}
\begin{proof}[\textbf{Proof of Theorem~\ref{thm:convergence-general}}.]
Note that Algorithm~\ref{algo:CD} is a descent algorithm---the objective function decreases after each coordinate update. Therefore, we must have
$$F(\bm\Theta^{(t)})\leq F(\bm\Theta^{(0)}),$$
i.e. $\bm\Theta^{(t)}\in \{\bm\Theta:F(\bm\Theta)\leq F(\bm\Theta^{(0)})\}$.

Since Assumption~\ref{assum:reg-assumption} holds, invoking Corollary~\ref{cor:smooth-stronglyconvex} with $F^{(0)}=F(\bm\Theta^{(0)})$, we get $f$ is coordinatewise Lipschitz and $F$ is coordinatewise strongly convex with some parameters depending on $F^{(0)}$, and thus on $\bm\Theta^{(0)}$. According to \citet{hong2017iteration}, we get the sublinear rate of convergence of Algorithm~\ref{algo:CD}, i.e. 
$$F(\bm\Theta^{(t)})-F^*\leq \frac{C}{t},$$where the constant $C$ depends on $\bm\Theta^{(0)}$.
\end{proof}
\begin{remark}
	According to Proposition~\ref{prop:g-assumption}, the regularizers $g(\theta_{ij};\lambda_0,\lambda_2,M,\underline{z}_{ij},\bar{z}_{ij})$ in $F_{\mathsf{node}}$ satisfy Assumption~\ref{assum:reg-assumption}, and thus Theorem~\ref{thm:convergence-general} applies to $F_{\mathsf{node}}$, which is exactly Theorem~\ref{thm:convergence-relaxation} in the main paper. 
\end{remark}

\subsection{Dual bound}\label{subsec:dual-bound-appendix}
We first present a proof for Theorem~\ref{thm:dual-bound} in Section~\ref{subsubsec:proof-dual-bound}, deriving the Lagrangian dual of the node relaxation objective $F_{\mathsf{node}}$. We then provide a sketch of how to compute the convex conjugate of $\psi$ and $\varphi$ as two cases of $g$ in Section~\ref{subsubsec:convex-conjugate}. Finally, we provide the proof for Proposition~\ref{prop:eff-comp-dual-bounds} in Section~\ref{subsubsec:proof-eff-comp}.
\subsubsection{Proof of Theorem~\ref{thm:dual-bound}}\label{subsubsec:proof-dual-bound}
We introduce auxiliary primal variables $\bm r_i= \btX\bm\theta_i$ to rewrite  Problem~\eqref{eqn:relaxation} as follows:
\begin{equation}\label{eqn:equiv-primal}
	\begin{aligned} 
		\min_{\bm\Theta\in\mathbb{S}^p}~&~\sum_{i=1}^p( -\log\theta_{ii}+\frac1{\theta_{ii}}\norm{\bm r_i}^2)+\sum_{i<j}g(\theta_{ij};\lambda_0,\lambda_2,M,\underline{z}_{ij},\bar{z}_{ij})\\
		~~\text{s.t.}~&~ \bm r_i=\btX\bm\theta_i,~\forall i\in[p].
	\end{aligned}
\end{equation}
By dualizing the constraints in \eqref{eqn:equiv-primal}, we can write the Lagrangian as 
\begin{equation}\label{eqn:Lagrangian}
	\mathcal{L}(\bm\Theta,\bm r;\bm\nu) =\sum_{i=1}^p( -\log\theta_{ii}+\frac1{\theta_{ii}}\norm{\bm r_i}^2+\langle \bm\nu_i,\bm r_i-\btX\bm\theta_i\rangle)+\sum_{i<j}g(\theta_{ij};\lambda_0,\lambda_2,M,\underline{z}_{ij},\bar{z}_{ij}).
\end{equation}
The Lagrangian dual is given by $D(\bm\nu)=\min_{\bm\Theta\in\mathbb{S}^p}\mathcal{L}(\bm\Theta,\bm r;\bm\nu).$ Since Slater's condition holds \citep{bertsekas2016nonlinear}, we have by strong duality:
$$\min_{\bm\Theta\in\mathbb{S}^p}F_{\text{node}}(\bm\Theta)=\max_{\bm\nu}D(\bm\nu).$$

Minimizing \eqref{eqn:Lagrangian} with respect to $\bm r_i$, we get 
\begin{equation}\label{eqn:dual-r}
	\bm r_i=-\frac{\theta_{ii}}{2}\bm{\nu}_i.
\end{equation}
Plugging this back to the Lagrangian~\eqref{eqn:Lagrangian}, we get 
\begin{equation*}
	\theta_{ii}=\arg\min_\theta~-\log \theta+\theta(-\norm{\bm\nu_i}^2/4-\btx_i^\top\bm{\nu}_i),
\end{equation*}
which yields 
\begin{equation}\label{eqn:dual-diag}
	\theta_{ii}=\frac{1}{-\norm{\bm \nu_i}^2/4-\btx_i^\top\bm \nu_i} \quad\text{if}\quad -\norm{\bm\nu_i}^2/4-\btx_i^\top\bm\nu_i>0.
\end{equation}
If $-\norm{\bm\nu_i}^2/4-\btx_i^\top\bm\nu_i\leq 0$, then $\theta_{ii}\to \infty$, and the minimum value is $-\infty$, which cannot be achieved. 

As for $\theta_{ij}=\theta_{ji}$,
\begin{align}
	\theta_{ij}=\theta_{ji}&=\arg\min_\theta ~(-\btx_j^\top\bm{\nu}_i-\btx_i^\top\bm\nu_j)\theta+g(\theta;\lambda_0,\lambda_2,M,\underline{z}_{ij},\bar{z}_{ij})\nonumber\\
	&=\arg\max_\theta~(\btx_j
	^\top\bm\nu_i+\btx_i^\top\bm\nu_j)\theta -g(\theta;\lambda_0,\lambda_2,M,\underline{z}_{ij},\bar{z}_{ij})\nonumber\\
	&\in\partial g^*(\btx_j^\top\bm\nu_i+\btx_i^\top\bm\nu_j;\lambda_0,\lambda_2,M,\underline{z}_{ij},\bar{z}_{ij}).\label{eqn:dual-offdiag}
\end{align}

Therefore, plugging \eqref{eqn:dual-r}, \eqref{eqn:dual-diag} and \eqref{eqn:dual-offdiag} into the Lagrangian function~\eqref{eqn:Lagrangian}, we get the Lagrangian dual problem:
\begin{equation}
	\max_{\bm \nu}~D(\bm\nu)=p+\sum_{i=1}^p\log(-\norm{\bm\nu_i}^2/4-\btx_i^\top\bm \nu_i)-\sum_{i<j}g^*(\btx_j^\top\bm \nu_i+\btx_i^\top\bm\nu_j;\lambda_0,\lambda_2,M,\underline{z}_{ij},\bar{z}_{ij}).
\end{equation}
\endproof

\subsubsection{Computing the convex conjugates}\label{subsubsec:convex-conjugate}~\\

\customparagraph{Convex Conjugate of $\psi$} We consider the Fenchel conjugate $\psi^*$ of $\psi$:
\begin{align}
	\psi^*(\alpha; \lambda_0,\lambda_2,M)&:=\sup_\theta ~\alpha \theta-\psi(\theta;\lambda_0,\lambda_2,M).
\end{align}
According to \citet{hazimeh2020sparse}, when $\sqrt{\lambda_0/\lambda_2}\leq M$, 
\begin{equation}
	\psi^*(\alpha;\lambda_0,\lambda_2,M)=\min_\gamma \left[\frac{(\gamma -\alpha)^2}{4\lambda_2}-\lambda_0\right]_++M|\gamma|;
\end{equation}
and when $\sqrt{\lambda_0/\lambda_2}>M$, we have
\begin{equation}
	\psi^*(\alpha;\lambda_0,\lambda_2,M)=\min_\mu~M|\mu|\quad\text{s.t.}\quad |\alpha|-\mu\leq \lambda_0/M+\lambda_2M.
\end{equation}
Table~\ref{table:summary-psistar} below summarizes different expressions for $\psi^*$:
\begin{table}[H]
	\centering
	\caption{Summary of different regimes and cases of $\psi^*$}
	\label{table:summary-psistar}
	\begin{tabular}{cccccc}
		\toprule
		Regime& Range of $|\alpha|$&$\psi^*(\alpha;\lambda_0,\lambda_2,M)$&$\theta^*\in\partial \psi^*(\alpha)$&$\gamma^*/\mu^*$  \\
		\midrule
		\multirow{3}{*}{$\sqrt{\lambda_0/\lambda_2}\leq M$}& $[0,2\sqrt{\lambda_0\lambda_2})$& $0$&$0$&$0$\\
		&$(2\sqrt{\lambda_0\lambda_2},2\lambda_2 M]$&$\frac{\alpha^2}{4\lambda_2}-\lambda_0$&$\frac{\alpha}{2\lambda_2}$&$0$\\
		&$(2\lambda_2 M,\infty)$&$M|\alpha|-(\lambda_0+\lambda_2M^2)$&$M\sign(\alpha)$&$\alpha-2M\lambda_2\sign(\alpha)$\\\midrule
		\multirow{2}{*}{$\sqrt{\lambda_0/\lambda_2}>M$}&$[0,\lambda_0/M+\lambda_2M]$&$0$&$0$&$0$\\
		&$(\lambda_0/M+\lambda_2M,\infty)$&$M|\alpha|-(\lambda_0+\lambda_2M^2)$&$M\sign(\alpha)$&$|\alpha|-(\lambda_0/M+\lambda_2M) $\\
		\bottomrule
	\end{tabular}
\end{table}

\customparagraph{Convex conjugate of $\varphi$} We consider the Fenchel conjugate $\varphi^*$ of $\varphi$:
\begin{equation}
	\varphi^*(\alpha;z,\lambda_0,\lambda_2,M):=\sup_\theta~\alpha \theta-\varphi(\theta;z,\lambda_0,\lambda_2,M).
\end{equation}
In Table~\ref{table:summary-phistar}, we summarize the expressions for $\varphi^*$
\begin{table}[H]
	\centering
	\caption{Summary of different regimes and cases of $\varphi^*$}
	\label{table:summary-phistar}
	\begin{tabular}{ccccc}
		\toprule
		Regime& Range of $|\alpha|$&$\varphi^*(\alpha;z,\lambda_0,\lambda_2,M)$&$\theta^*\in\partial \varphi^*(\alpha;z)$ \\
		\midrule
		$z=0$&$[0,\infty)$&$0$&$0$\\\midrule
		\multirow{2}{*}{$z=1,\lambda_2>0$}&$[0,2\lambda_2M]$&$\frac{\alpha^2}{4\lambda_2}-\lambda_0$&$\frac{\alpha}{2\lambda_2}$\\
		&$(2\lambda_2M,\infty)$&$M|\alpha|-(\lambda_0+\lambda_2 M^2)$&$M\sign(\alpha)$\\\midrule
		$z=1,\lambda_2=0$& $[0,\infty)$&$M|\alpha|-\lambda_0$&$M\sign(\alpha)$\\
		\bottomrule
	\end{tabular}
\end{table}

\subsubsection{Proof of Proposition~\ref{prop:eff-comp-dual-bounds}}\label{subsubsec:proof-eff-comp}
To prove Proposition~\ref{prop:eff-comp-dual-bounds}, we will need a couple of helper propositions~\ref{prop:zero-thresh} and~\ref{prop:zero-thresh-following}
\begin{proposition}\label{prop:zero-thresh}
	Denote by
	\begin{equation}
		c(\lambda_0,\lambda_2,M)=
		\left\{\begin{array}{ll}
			2\sqrt{\lambda_0\lambda_2}& ~\text{if}~ \sqrt{\lambda_0/\lambda_2}\leq M  \\
			\lambda_0/M+\lambda_2M& ~\text{otherwise.}
		\end{array}\right.
	\end{equation}
	The following statements hold
	\begin{enumerate}[label=(\alph*)]
		\item $\prox{\psi}(\tilde\beta;\lambda_0,\lambda_2,M)=0~\iff~ |\tilde\beta|\leq c(\lambda_0,\lambda_2,M)$
		\item $\mathcal{Q}_\psi(a,b;\lambda_0,\lambda_2,M)=0~\iff~|b|\leq c(\lambda_0,\lambda_2,M)$
		\item $\psi^*(\alpha;\lambda_0,\lambda_2,M)=0~\iff~|\alpha|\leq c(\lambda_0,\lambda_2,M)$
	\end{enumerate}
\end{proposition}
The proof of Proposition~\ref{prop:zero-thresh} follows by using \eqref{eqn:prox-psi}, \eqref{eqn:Q-psi} and Table~\ref{table:summary-psistar}.

Before stating the next proposition~\ref{prop:zero-thresh-following}, we
recall the definitions of $\hat{\mathcal{S}}$ and $\mathcal{F}_1$:
$$\hat{\mathcal{S}}=\{(i,j):i<j, \hat{\theta}_{ij}\neq 0\}, \quad\text{and}\quad \mathcal{F}_1= \{(i,j):\underline z_{ij}=\bar z_{ij}=1\}.$$
Additionally, we define
$$\mathcal{F}_0=\{(i,j):\underline z_{ij}=\bar z_{ij}=0\}\quad\text{and}\quad \mathcal{R}=\{(i,j):\underline z_{ij}=0,\bar z_{ij}=1\}.$$

We will denote by $g^*_{ij}:=g^*(\btx_i^\top\hat{\bm\nu}_j+\btx_j^\top\hat{\bm\nu}_i;\lambda_0,\lambda_2,M,\underline{z}_{ij},\bar z_{ij})$ where $g^*(\cdot; \lambda_0,\lambda_2,M,\underline{z},\bar{z})$ is the convex conjugate of $g(\cdot; \lambda_0,\lambda_2,M,\underline{z},\bar{z})$ defined in~\eqref{eqn:relax-reg}.

\begin{proposition}\label{prop:zero-thresh-following}
	For any $\hat\theta_{ij}=0$, i.e. $(i,j)\in\hat{\mathcal{S}}^c$, 
	\begin{enumerate}[label=(\alph*)]
		\item if $(i,j)\in\mathcal{R}\cup \mathcal{F}_0$, then $g^*_{ij}=0$.
		\item if $(i,j)\in\mathcal{F}_1$, then
		$g^*_{ij}=-\lambda_0.$
	\end{enumerate}
\end{proposition}
\begin{proof}[\textbf{Proof of Proposition~\ref{prop:zero-thresh-following}}.] We first note that when Algorithm~\ref{algo:ASCD} terminates, then $\mathcal{V}$ must be empty, i.e. for any $\hat\theta_{ij}=0$, we must have 
$$0\in\arg\min_{\theta_{ij}}F(\hat{\bm\Theta}-\hat\theta_{ij}\bm E_{ij}-\hat\theta_{ij}\bm E_{ji}+\theta_{ij}\bm E_{ij}+\theta_{ij}\bm E_{ji}),$$
or equivalently (according to Section~\ref{subsubsec:offdiag-update-derive}),
\begin{equation}\label{eqn:V-empty-Q-0}
	\mathcal{Q}_g\left(a_{ij},b_{ij};\lambda_0,\lambda_2,M,\underline z_{ij},\bar z_{ij}\right)=0,
\end{equation}
where 
$$a_{ij}=\frac{v_j}{\hat\theta_{ii}}+\frac{v_i}{\hat\theta_{jj}},$$
and
\begin{equation}\label{eqn:b-nu}
	b_{ij}=\frac{2\btx_j^\top(\hat{\bm r}_i-\hat\theta_{ij}\btx_j)}{\hat\theta_{ii}}+\frac{2\btx_i^\top(\hat{\bm r}_j-\hat\theta_{ij}\btx_i)}{\hat\theta_{jj}}=-(\btx_i^\top\hat{\bm\nu}_j+\btx_j^\top\hat{\bm\nu}_j).
\end{equation}
Here, \eqref{eqn:b-nu} follows from $\hat\theta_{ij}=0$ and using the dual solution definition \eqref{eqn:dual-solution}.

\noindent {\it Proof of Part (a)} If $(i,j)\in\mathcal{F}_0$, then reading Table~\ref{table:summary-phistar}, we get $g_{ij}^*=0$.

If $(i,j)\in\mathcal{R}$, we have $g=\psi$. According to Proposition~\ref{prop:zero-thresh}~(b), \eqref{eqn:V-empty-Q-0} with \eqref{eqn:b-nu} implies $|b_{ij}|=|\btx_i^\top\hat{\bm\nu}_j+\btx_j^\top\hat{\bm\nu}_j|\leq c(\lambda_0,\lambda_2,M)$, which, by Proposition~\ref{prop:zero-thresh}~(c), implies $g^*_{ij}=\psi^*(\btx_i^\top\hat{\bm\nu}_j+\btx_j^\top\hat{\bm\nu}_j)=0$.

\noindent {\it Proof of Part~(b)} If $(i,j)\in\mathcal{F}_1$, then $g(\theta_{ij})=\psi(\theta_{ij};z,\lambda_0,\lambda_2,M)$ with $z=1$. According to \eqref{eqn:Q-phi} and \eqref{eqn:prox-phi} in the case of $z=1$, we have \eqref{eqn:V-empty-Q-0} with \eqref{eqn:b-nu} implies $|b_{ij}|=|\btx_i^\top\hat{\bm\nu}_j+\btx_j^\top\hat{\bm\nu}_j|=0$, which implies $g^*_{ij}=\varphi^*(\btx_i^\top\hat{\bm\nu}_j+\btx_j^\top\hat{\bm\nu}_j;z=1)=-\lambda_0$, according to Table~\ref{table:summary-phistar}.
\end{proof}

\customparagraph{Proof of Proposition~\ref{prop:eff-comp-dual-bounds}}
We are now ready to complete the proof of
this proposition by making use of the last two propositions. 

Note that we can decompose the sum over all pairs of $(i,j)$ into four parts:
$$\sum_{i,j}g^*_{ij}=\sum_{(i,j)\in\hat{\mathcal{S}}}g^*_{ij}+\sum_{(i,j)\in\hat{\mathcal{S}}^c\cap \mathcal{F}_1}g^*_{ij}+\sum_{(i,j)\in\hat{\mathcal{S}}^c\cap \mathcal{F}_0}g^*_{ij}+\sum_{(i,j)\in\hat{\mathcal{S}}^c\cap \mathcal{R}}g^*_{ij}.$$
With Proposition~\ref{prop:zero-thresh-following}, we have the last two terms are $0$ and the second term becomes $-\lambda_0|\mathcal{F}_1\backslash \hat{\mathcal{S}}|$. Thus, we prove the desired result.

\subsection{Local Search Details} \label{subsec:localsearch}
After obtaining an approximate solution from the CD method in Section~\ref{subsec:incumbents-solving}, we 
attempt to improve the solution via local search
motivated by~\citet{hazimeh2020sparse,beck2013sparsity}. To this end, we consider setting a nonzero coordinate of the the current solution $\hat{\B\Theta}$ to zero, and allow a zero coordinate of $\hat{\B\Theta}$ to take nonzero values. Then, we check if optimizing over the new nonzero coordinate (while keeping other coordinates fixed) improves the objective. However, as $\hat{\B\Theta}$ is a symmetric matrix, we only consider row-wise swaps---we swap zero and nonzero coordinates in a row, and we cycle through the rows. We also update symmetric above and below diagonal coordinates simultaneously to preserve the symmetry.

Let us call the row $i\in[p]$ active if there exists $j_0\neq i$ such that $(i,j_0)\in\mathcal{S}$. We denote the set of all active rows by $\mathcal{S}_{act}$. We also define the objective with a single swap as
\begin{equation}
	F_0(\hat{\bm\Theta},i,j_1,j_2,\theta)=F(\hat{\bm\Theta}-\hat \theta_{ij_1}\bm E_{ij_1}-\hat\theta_{ij_1}\bm E_{j_1i}+\theta\bm E_{ij_2}+\theta\bm E_{j_2i})
\end{equation}
with $F$ defined in~\eqref{eqn:unified}. Our local search is summarized in Algorithm~\ref{algo:localsearch}. We note that in the calculation of the optimal swap in line~3 of Algorithm~\ref{algo:localsearch}, the value of $\hat{\theta}$ can be obtained using the closed-form solutions that we derive for the CD updates.

\begin{algorithm}[H]
	\begin{algorithmic}[1]
		\Require An approximate solution from CD, $\hat{\bm\Theta}$
		\While {not converged}
		\For {each $i\in\mathcal{S}_{act}$}
		\State $(\hat{j}_1,\hat{j}_2,\hat\theta)\gets\arg\min_{ j_1:\hat\theta_{ij_1}\neq 0,j_2:\hat\theta_{ij_2}=0 ,\theta } F_0(\hat{\bm\Theta},i,j_1,j_2,\theta)$
		\State $\hat{\theta}_{i\hat{j}_1}=\hat{\theta}_{\hat{j}_1i}\gets 0,~~\hat{\theta}_{i\hat{j}_2}=\hat{\theta}_{\hat{j}_2i}\gets \hat\theta$
		\EndFor
		\For {$i=1,2,\ldots, p$}
		\State $\hat\theta_{ii}\gets \arg\min_{\theta_{ii}} F(\hat{\bm\Theta}-\hat \theta_{ii}\bm E_{ii}+\theta_{ii}\bm E_{ii})$
		\EndFor
		\EndWhile
	\end{algorithmic}
	\caption{Local Search for Problem~\eqref{eqn:unified}}
	\label{algo:localsearch}
\end{algorithm}

\section{Proofs on Statistical Properties (Section~\ref{sec:stat})}\label{app:stat-proofs}
Before proceeding with the proof of main results, we present a 
result (see discussion at the beginning of Section~\ref{sec:proposedest}) that we use throughout the proofs. 
\begin{lemma}\label{lem_reg_ext}
	For $j\in[p]$, let $\B{\varepsilon}_j\in\R^{n}$ be such that
	$$\B{\varepsilon}_{j}=\B{x}_j-\sum_{i\neq j}\beta^*_{ij}\B{x}_i.$$
	Then, $\B{\varepsilon}_{j}$ and $\{\B{x}_i\}_{i\neq j}$ are independent for every $j$. Moreover, for every $j$,
	$$\B{\varepsilon}_{j}\sim\mathcal{N}(\B{0},(\sigma_j^*)^2\B{I}_n).$$
\end{lemma}
\subsection{Useful Lemmas}
\begin{lemma}[Theorem 1.19, \citet{rigollet2015high}]\label{maximal}
	Let $\omega \in\mathbb{R}^p$ be a random vector with ${\omega}_i\stackrel{\text{iid}}{\sim}\mathcal{N}(0,\sigma^2)$,  then 
	\begin{equation}
		\p(\sup_{\B{\theta}\in \mathcal{B}(p)}\B{\theta}^\top\B{\omega}>t)\leq \exp\left(-\frac{t^2}{8\sigma^2}+p\log 5\right),
	\end{equation}
	where $\mathcal{B}(p)$ denotes the unit Euclidean ball of dimension $p$.
\end{lemma}
\begin{lemma}[Lemma~5, \citet{behdin2021integer}]\label{minnorm2}
	Suppose the rows of the matrix $\B{X}\in\mathbb{R}^{n\times p}$ (with $n\geq p$) are iid draws from a multivariate Gaussian distribution  with zero mean and covariance matrix $\B{G}$ i.e., $\mathcal{N}(\B{0},\B{G})$.
	Moreover, suppose $0<\bar{\sigma}\leq \sqrt{\lambda_{\min}(\B{G})}$ where $\lambda_{\min}$ denotes the smallest eigenvalue. Then,
	$$\bar{\sigma}\left(1-c_0\left(\sqrt{\frac{p}{n}}+\frac{t}{\sqrt{n}}\right)\right)\lesssim \frac{1}{\sqrt{n}}\sigma_{\min}(\B{X})$$
	with probability at least $1-\exp(-Ct^2)$ for some universal constants $C,c_0>0$.
\end{lemma}
\begin{lemma}\label{bernlem}
	Let the rows of the data matrix $\B{X}\in\mathbb{R}^{n\times p}$ be iid draws from $\mathcal{N}(\B{0},\B{G})$. For fixed $j_1,j_2\in[p]$, we have
	\begin{equation}\label{bern2}
		\p\left(\left\vert \frac{1}{n}\sum_{i=1}^n x_{ij_1}x_{ij_2} - G_{j_1j_2}\right\vert>c_{\Psi}(|G_{j_1j_2}|+\sqrt{G_{j_1j_1}G_{j_2j_2}})\sqrt{\frac{\log(1/\delta)}{C_b n}}\right) \leq 2\delta
	\end{equation}
	when $n>\frac{2}{C_b}\log(1/\delta)$, for some universal constants $C_b,c_{\Psi}>0$. Above, $G_{ij}$ is the $(i,j)$-th coordinate of $\B G$.
\end{lemma}
\begin{proof}[\textbf{Proof of Lemma~\ref{bernlem}}.]
For $s\in\{1,2\}$, we define the $\psi_s$-Orlicz norm~\citep[Ch. 2.5 and 2.7]{vershynin2018high} of a random variable $X$ as
$$\|X\|_{\psi_s}=\inf\{t>0:\E[\exp(|X|^s/t^s)]\leq 2\}.$$
Note that $\E[x_{ij_1}x_{ij_2}]=G_{j_1j_2}$, therefore, by Lemmas 2.7.6 and 2.7.7 of~\citet{vershynin2018high}, the $\psi_1$-Orlicz norm of $x_{ij_1}x_{ij_2}-G_{j_1j_2}$ can be bounded as
$$\|x_{ij_1}x_{ij_2}-G_{j_1j_2}\|_{\psi_1}\leq \|x_{ij_1}\|_{\psi_2}\|x_{ij_2}\|_{\psi_2}+\|G_{j_1j_2}\|_{\psi_1}\leq c_{\Psi}(|G_{j_1j_2}|+\sqrt{G_{j_1j_1}G_{j_2j_2}})$$
for some $c_{\Psi}>0$. Consequently, by Bernstein's inequality \cite[Theorem 2.8.1]{vershynin2018high}
\begin{multline}\label{bernineq}
	\p\left(\left\vert\frac{1}{n}\sum_{i=1}^n x_{ij_1}x_{ij_2} - G_{j_1j_2}\right\vert>t\right)\\ \leq 2\exp\left(-C_bn\left[\frac{t^2}{c_{\Psi}^2(|G_{j_1j_2}|+\sqrt{G_{j_1j_1}G_{j_2j_2}})^2}\land \frac{t}{c_{\Psi}(|G_{j_1j_2}|+\sqrt{G_{j_1j_1}G_{j_2j_2}})} \right]\right)
\end{multline}
for some constant $C_b>0$. Let us take
$$t = c_{\Psi}(|G_{j_1j_2}|+\sqrt{G_{j_1j_1}G_{j_2j_2}})\sqrt{\frac{\log(1/\delta)}{C_b n}}$$
and 
$$n>\frac{2}{C_b}\log(1/\delta).$$
As a result, 
$$\frac{t^2}{c_{\Psi}^2(|G_{j_1j_2}|+\sqrt{G_{j_1j_1}G_{j_2j_2}})^2}=\frac{\log(1/\delta)}{C_b n}\leq \sqrt{\frac{\log(1/\delta)}{C_b n}}= \frac{t}{c_{\Psi}(|G_{j_1j_2}|+\sqrt{G_{j_1j_1}G_{j_2j_2}})}$$
so
$$\left[\frac{t^2}{c_{\Psi}^2(|G_{j_1j_2}|+\sqrt{G_{j_1j_1}G_{j_2j_2}})^2}\land \frac{t}{c_{\Psi}(|G_{j_1j_2}|+\sqrt{G_{j_1j_1}G_{j_2j_2}})} \right]=\frac{t^2}{c_{\Psi}^2(|G_{j_1j_2}|+\sqrt{G_{j_1j_1}G_{j_2j_2}})^2}$$
which completes the proof of this lemma by making use of~\eqref{bernineq}.
\end{proof}

\subsection{Proof of Theorem \ref{lowdimthm}}

\subsubsection{Some helper lemmas} 
We present a few lemmas and proofs which will be used in our proof of 
Theorem \ref{lowdimthm}.
\begin{lemma}\label{thm1techlem1}
	Let 
	\begin{equation}\label{yjdef}
		\begin{aligned}
			\hat{\B{y}}_j & = \B{x}_j-\sum_{i\neq j}\hat{\beta}_{ij}\B{x}_i,\\
			{\B{y}}^*_j & = \B{x}_j-\sum_{i\neq j}{\beta}^*_{ij}\B{x}_i
		\end{aligned}
	\end{equation}
	and denote
	\begin{equation}\label{sjdef}
		\begin{aligned}
			\hat{\CS}_j & =\{i\in[p]:i\neq j,\hat{\beta}_{ij}\neq 0\},~~~~~~
			{\CS}_j^* & =\{i\in[p]:i\neq j,{\beta}_{ij}^*\neq 0\}
		\end{aligned}
	\end{equation}
	for $j\in[p]$. Let the event $E_1$ be defined as
	\begin{equation}
		E_1=\left\{ \sum_{j\in[p]}\frac{1}{20}(\hat{\sigma}_j-\sigma^*_j)^2 + \sum_{j\in[p]}\frac{\|\hat{\B{y}}_j\|_2^2-\|\B{y}_j^*\|_2^2}{2n\hat{\sigma}_j^2}\lesssim \frac{1}{l_{\sigma}^2}\frac{p\log(p/k)}{n}+\lambda \sum_{j=1}^p\left[|S_j^*|-|\hat{S}_j|\right]\right\}.
	\end{equation}
	Under the assumptions of Theorem \ref{lowdimthm}, we have
	$$\p(E_1)\geq 1-p(k/p)^{10}.$$
\end{lemma}
\begin{proof}[\textbf{Proof of Lemma~\ref{thm1techlem1}}.]
Let the event $\CE_j$ be defined as
$$\CE_j:=\left\{\left\vert ({\sigma_j^*})^2- \frac{\|\B{y}_j^*\|_2^2}{n}\right\vert  \lesssim ({\sigma_j^*})^2\sqrt{\frac{\log(p/k)}{n}}\right\}.$$
Note that $\B{y}_j^*=\B{\varepsilon}_j$ by definition, therefore
$$\left\vert ({\sigma_j^*})^2- \frac{\|\B{y}_j^*\|_2^2}{n}\right\vert =\left\vert ({\sigma_j^*})^2- \frac{\|\B{\varepsilon}_j\|_2^2}{n}\right\vert.$$
Invoke Lemma~\ref{bernlem} for $\B\varepsilon_j\overset{\text{iid}}{\sim}\mathcal{N}(0,(\sigma_j^*)^2)$ with $\delta=(p/k)^{10}$. As $n\gtrsim \log p$, one has
\begin{align}\label{sigmajconv}
	\p\left(\CE_j\right)=\p\left(\left\{\left\vert ({\sigma_j^*})^2- \frac{\|\B{\varepsilon}_j\|_2^2}{n}\right\vert  \lesssim ({\sigma_j^*})^2\sqrt{\frac{\log(p/k)}{n}}\right\}\right)\geq 1-(k/p)^{10}.
\end{align}
As a result, by union bound 
\begin{equation}\label{ejunionbounded}
	\p\left(\bigcap_{j\in[p]} \CE_j\right)\geq 1-p(k/p)^{10}.
\end{equation}
In particular, on  event $\CE_j$ if we take $n\gtrsim 36\log(p/k)$, we achieve
\begin{equation}\label{yjsigma}
	\frac{\|\B{y}_j^*\|_2^2}{n}\geq \frac{5({\sigma_j^*})^2}{6} 
\end{equation}
with high probability. The rest of the proof is on the event $\bigcap_{j\in[p]} \CE_j$.\\
Let
\begin{equation}\label{fj}
	f_j(x) = \log(x)+ \frac{\|\B{y}_j^*\|_2^2}{2n}\frac{1}{x^2}.
\end{equation}
By optimality of $\{\hat{\sigma}_j,\hat{\beta}_{ij}\}$ and feasibility of $\{\sigma_j^*,\beta_{ij}^*\}$ for Problem~\eqref{main-sup},
\begin{align}
	& \sum_{j\in[p]}\log(\hat{\sigma}_j)+\frac{\|\hat{\B{y}}_j\|_2^2}{2n\hat{\sigma}_j^2}+\lambda\sum_{j=1}^p |\hat{S}_j|\leq \sum_{j\in[p]} \log(\sigma^*_j)+\frac{\|\B{y}_j^*\|_2^2}{2n({\sigma_j^*})^2}+\lambda\sum_{j=1}^p |S_j^*| \nonumber \\
	\Rightarrow & \sum_{j\in[p]}\left[\log(\frac{\hat{\sigma}_j}{\sigma^*_j})  +\frac{\|\B{y}_j^*\|_2^2}{2n}\left(\frac{1}{\hat{\sigma}_j^2}-\frac{1}{({\sigma_j^*})^2}\right) + \frac{\|\hat{\B{y}}_j\|_2^2-\|\B{y}_j^*\|_2^2}{2n\hat{\sigma}^2_j}\right]\leq \lambda\sum_{j=1}^p |S_j^*|-\lambda\sum_{j=1}^p |\hat{S}_j| \nonumber\\ 
	\Rightarrow &\sum_{j\in[p]}\left[f_j(\hat{\sigma}_j)-f_j({\sigma_j^*})\right]  \leq \sum_{j\in[p]}\frac{\|\B{y}_j^*\|_2^2-\|\hat{\B{y}}_j\|_2^2}{2n\hat{\sigma}^2_j}+\lambda \sum_{j=1}^p\left[|S_j^*|-|\hat{S}_j|\right]. \label{outlineineq01}
\end{align}
From~\eqref{fj},
\begin{equation}
	\begin{aligned}
		f_j'(x) &= \frac{1}{x}- \frac{\|\B{y}_j^*\|_2^2}{nx^3},\\
		f_j''(x) &=-\frac{1}{x^2}+ \frac{3\|\B{y}_j^*\|_2^2}{nx^4}.
	\end{aligned}
\end{equation}
Therefore, by Taylor's expansion of $f_j$, 
\begin{align}
	f_j(\hat{\sigma}_j)-f_j({\sigma_j^*}) &= \left[\frac{1}{\sigma_j^*}- \frac{\|\B{y}_j^*\|_2^2}{n({\sigma_j^*})^3}\right](\hat{\sigma}_j-\sigma^*_j) + \frac{1}{2}\left[-\frac{1}{x^2}+ \frac{3\|\B{y}_j^*\|_2^2}{nx^4}\right](\hat{\sigma}_j-\sigma^*_j)^2 \nonumber\\
	& \stackrel{(A)}{\geq}- \left[\frac{1}{\sigma_j^*}- \frac{\|\B{y}_j^*\|_2^2}{n({\sigma_j^*})^3}\right]^2 +\frac{1}{2}\left[-\frac{1}{x^2}+ \frac{3\|\B{y}_j^*\|_2^2}{nx^4}-\frac{1}{2}\right](\hat{\sigma}_j-\sigma^*_j)^2\label{outlineineq02}
\end{align}
for some $x$ between $\sigma_j^*$ and $\hat{\sigma}_j$ where $(A)$ is by the inequality $2ab\geq -2a^2-b^2/2$. Consequently, for any $x\in[l_{\sigma},u_{\sigma}]$, 
\begin{align}
	-\frac{1}{x^2}+ \frac{3\|\B{y}_j^*\|_2^2}{nx^4}-\frac{1}{2} & = \frac{6\|\B{y}_j^*\|_2^2/n-2x^2-x^4}{2x^4} \nonumber \\
	& \stackrel{(A)}{\geq }\frac{5({\sigma_j^*})^2-2u_{\sigma}^2-u_{\sigma}^4}{2x^4} \nonumber\\
	& \geq \frac{5l_{\sigma}^2-2u_{\sigma}^2-u_{\sigma}^4}{2x^4}\nonumber \\
	& \geq \frac{5l_{\sigma}^2-2u_{\sigma}^2-u_{\sigma}^4}{2u_{\sigma}^4}>\frac{1}{10} \label{convexlower}
\end{align}
where the last inequality is due to Assumption~\ref{assum1-3} and inequality $(A)$ is due to~\eqref{yjsigma} on event $\CE_j$. 
Thus, 
\begin{align}
	\frac{1}{20}(\hat{\sigma}_j-\sigma^*_j)^2& \stackrel{(a)}{\leq} \frac{1}{2}\left[-\frac{1}{x^2}+ \frac{3\|\B{y}_j^*\|_2^2}{nx^4}-\frac{1}{2}\right](\hat{\sigma}_j-\sigma^*_j)^2\nonumber \\
	& \stackrel{(b)}{\leq} f_j(\hat{\sigma}_j)-f_j({\sigma_j^*}) + \left[\frac{1}{\sigma_j^*}- \frac{\|\B{y}_j^*\|_2^2}{n({\sigma_j^*})^3}\right]^2 \label{mid-upper-crude}
\end{align}
where $(a)$ is by~\eqref{convexlower} and $(b)$ is by~\eqref{outlineineq02}.
By substituting~\eqref{mid-upper-crude} into~\eqref{outlineineq01}, we obtain 
\begin{equation}\label{uppercrude}
	\sum_{j\in[p]}\frac{1}{20}(\hat{\sigma}_j-\sigma^*_j)^2 + \sum_{j\in[p]}\frac{\|\hat{\B{y}}_j\|_2^2-\|\B{y}_j^*\|_2^2}{2n\hat{\sigma}^2_j}\leq \sum_{j\in[p]}\left[\frac{1}{\sigma_j^*}- \frac{\|\B{y}_j^*\|_2^2}{n({\sigma_j^*}^3)}\right]^2+\lambda \sum_{j=1}^p\left[|S_j^*|-|\hat{S}_j|\right].
\end{equation}
For any $j\in[p]$, on event $\CE_j$ which holds with high probability (see~\eqref{sigmajconv}),  
\begin{align}
	\left[\frac{1}{\sigma_j^*}- \frac{\|\B{y}_j^*\|_2^2}{n({\sigma_j^*})^3}\right]^2 &=\frac{1}{({\sigma_j^*})^6}\left[({\sigma_j^*})^2- \frac{\|\B{y}_j^*\|_2^2}{n}\right]^2\nonumber \\
	& \lesssim \frac{1}{({\sigma_j^*})^2}\frac{\log(p/k)}{n}\lesssim  \frac{1}{l_{\sigma}^2}\frac{\log(p/k)}{n}\label{sigmaineq}.
\end{align}
As a result, by using the bound in~\eqref{sigmaineq} into~\eqref{uppercrude} the proof is complete.
\end{proof}

\begin{lemma}\label{new-lemma-lambda}
	Define the event $E^{\lambda}_j$ for $j\in[p]$ as 
	\begin{equation}
		E^{\lambda}_j=\left\{\sup_{\substack{ \B{v}\in\R^{p}\\ \text{Supp}(\B{v})=\hat{S}_j\cup S_j^*}}\left(\B{\varepsilon}_j^{\top}\frac{\B{X}\B{v}}{\|\B{X}\B{v}\|_2}\right)^2-nl_{\sigma}^2\tilde{c}_{\lambda}\lambda |\hat{S}_j|\lesssim {u_{\sigma}^2k\log(2p/k)}\right\}
	\end{equation}
	where $\hat{S}_j,S_j^*$ are defined in~\eqref{sjdef}, $\text{Supp}(\B v)=\{i:v_i\neq 0\}$ denotes the support of a vector and $\tilde{c}_{\lambda}$ is an arbitrary but fixed universal constant. Under the assumptions of Theorem~\ref{lowdimthm}, let $c_{\lambda}\gtrsim 1/\tilde{c}_{\lambda}$ appearing in~\eqref{thm1-lambda} [Theorem~\ref{lowdimthm}] be sufficiently large. Then, 
	$$\p(\cap_{j\in[p]}E^{\lambda}_j)\geq 1-p^2\exp(-10k\log(p/k)).$$
\end{lemma}
\begin{proof}[\textbf{Proof of Lemma~\ref{new-lemma-lambda}}.]
Suppose $\B{\Phi}_S\in\R^{n\times |S|}$ is an orthonormal basis for the column span of $\B{X}_S$ for $S\subseteq[p]$. By Lemma~\ref{lem_reg_ext}, if $j\notin S$, then $\B{\varepsilon}_j$ and $\B{X}_{[p]\setminus\{j\}}$ are independent. As a result, we have the conditional distribution  $\B{\Phi}_{S}^\top\B{\varepsilon}_{j}|\B{X}_{[p]\setminus\{j\}}\stackrel{}{\sim}\mathcal{N}(0,({\sigma_j^*})^2\B{I}_{|S|})$. Given this fact, one has for $t>0$ and a fixed $j\in[p]$,
\begin{align}
	&  \p\left(\sup_{\substack{ \B{v}\in\R^{p}\\ \text{Supp}(\B{v})=\hat{S}_j\cup S_j^*}}\left(\B{\varepsilon}_j^{\top}\frac{\B{X}\B{v}}{\|\B{X}\B{v}\|_2}\right)^2  > t+nl_{\sigma}^2\tilde{c}_{\lambda}\lambda |\hat{S}_j|\right) \nonumber\\
	\leq & \p\left(\max_{s\in[p-1]}\max_{\substack{\CS\subseteq[p]\setminus\{j\} \\ |\CS|= s}}\sup_{\B v\in\R^{|S\cup S_j^*|}}\left[\left(\B{\varepsilon}_j^{\top}\frac{\B{X}_{\CS\cup S_j^*}\B{v}}{\|\B{X}_{\CS\cup S_j^*}\B{v}\|_2}\right)^2-nl_{\sigma}^2\tilde{c}_{\lambda}\lambda |S|\right] > t\right) \nonumber\\
	\stackrel{(a)}{=}& \p\left(\left.\max_{s\in[p-1]}\max_{\substack{\CS\subseteq[p]\setminus\{j\} \\ |\CS|= s}}\sup_{\B v\in\R^{|S\cup S_j^*|}}\left[\left(\B{\varepsilon}_j^{\top}\frac{\B{X}_{\CS\cup S_j^*}\B{v}}{\|\B{X}_{\CS\cup S_j^*}\B{v}\|_2}\right)^2-nl_{\sigma}^2\tilde{c}_{\lambda}\lambda |S|\right] > t\right\vert \B{X}_{[p]\setminus\{j\}}\right) \nonumber\\
	\stackrel{(b)}{\leq} &\p\left(\left.\max_{s\in[p-1]}\max_{\substack{\CS\subseteq[p]\setminus\{j\} \\ |\CS|= s}}\sup_{\|\B{\alpha}\|_2=1}\left[\left(\B{\varepsilon}_j^{\top}\B{\Phi}_{\CS\cup S_j^*}\B{\alpha}\right)^2 -nl_{\sigma}^2\tilde{c}_{\lambda}\lambda |S|\right]> t\right\vert \B{X}_{[p]\setminus\{j\}}\right) \nonumber\\
	\stackrel{(c)}{\leq}  & \sum_{s=1}^p \sum_{\substack{\CS\subseteq[p]\setminus\{j\} \\ |\CS|= s}}\p\left(\left.\sup_{\|\B{\alpha}\|_2=1}\left(\B{\varepsilon}_j^{\top}\B{\Phi}_{\CS}\B{\alpha}\right)^2 > t+nl_{\sigma}^2\tilde{c}_{\lambda}\lambda s\right\vert \B{X}_{[p]\setminus\{j\}}\right) \nonumber \\
	\stackrel{(d)}{\leq}& \sum_{s=1}^p \sum_{\substack{\CS\subseteq[p]\setminus\{j\} \\ |\CS|= s}} \exp\left(-\frac{t+n\tilde{c}_{\lambda}\lambda l_{\sigma}^2 s}{8({\sigma_j^*})^2}+(s+k)\log 5\right) \nonumber \\
	\stackrel{(e)}{\leq}& \sum_{s=1}^p \sum_{\substack{\CS\subseteq[p]\setminus\{j\} \\ |\CS|= s}} \exp\left(-\frac{t}{8({\sigma_j^*})^2}-5s\log(2p/k) + k\log 5\right) \nonumber \\
	\stackrel{(f)}{\leq}& \sum_{s=1}^p\left({ep}\right)^{s}\exp\left(-\frac{t}{8({\sigma_j^*})^2}-5s\log (2p/k)+k\log 5\right) \nonumber \\
	\leq &p\exp\left(-\frac{t}{8u_{\sigma}^2}+4k\log(2p/k)\right)\label{thm2-ineq1.555}
\end{align}
where $(a)$ is due to independence of $\B{\varepsilon}_j$ and $\B{X}_{[p]\setminus\{j\}}$ as discussed above, $(b)$ is true as $\B{X}_S\B{v}$ is in the column span of $\B{\Phi}_S$ and $\|\B\alpha\|_2=1$, $(c)$ is by union bound, $(d)$ is due to Lemma~\ref{maximal} and the conditional distribution $\B{\Phi}_{S}^\top\B{\varepsilon}_{j}|\B{X}_{[p]\setminus\{j\}}\stackrel{}{\sim}\mathcal{N}(0,({\sigma_j^*})^2\B{I}_{|S|})$ discussed above, $(e)$ is by choice of $\lambda$ in~\eqref{thm1-lambda} and taking $c_{\lambda}\gtrsim 1/\tilde{c}_{\lambda}$, and $(f)$ is due
to the inequality ${p\choose k}\leq (ep/k)^k$. Suppose we take 
$$t=8cu_{\sigma}^2k\log(2p/k)$$
then from~\eqref{thm2-ineq1.555}, we have:
\begin{equation}
	\p\left(\sup_{\substack{ \B{v}\in\R^{p}\\ \text{Supp}(\B{v})=\hat{S}_j\cup S_j^*}}\left(\B{\varepsilon}_j^{\top}\frac{\B{X}\B{v}}{\|\B{X}\B{v}\|_2}\right)^2 -nl_{\sigma}^2\tilde{c}_{\lambda}\lambda |\hat{S}_j|> 8cu_{\sigma}^2k\log(2p/k)\right) \leq p\exp\left(-(c-4)k\log(2p/k)\right).
\end{equation}
Take $c$ sufficiently large and by union bound over $j\in[p]$, we have 
$$\p(\cap_{j\in[p]}E^{\lambda}_j)\geq 1-p^2\exp(-10k\log(p/k)).$$
\end{proof}

\begin{lemma}\label{thm1techlem2}
	Let $\B{y}_j^*,\hat{\B{y}}_j$ be defined as in \eqref{yjdef} and $\tilde{c}_{\lambda}$ be as in Lemma~\ref{new-lemma-lambda}. Let the event $E_2$ be defined as
	$$E_2=\left\{  \frac{1}{2n u_{\sigma}^2}\sum_{j\in[p]} \left\Vert\sum_{i:i\neq j} (\beta_{ij}^*-\hat{\beta}_{ij})\B{x}_i\right\Vert_2^2 \lesssim \sum_{j\in[p]} \frac{ \|\hat{\B{y}}_j\|_2^2-\|\B{y}_j^*\|_2^2}{n\hat{\sigma}_j^2}+\frac{u_{\sigma}^2kp\log(2p/k)}{l_{\sigma}^2n}+\tilde{c}_{\lambda}\lambda\sum_{j\in[p]}|\hat{S}_j|\right\}.$$
	Under the assumptions of Theorem~\ref{lowdimthm}, we have:
	$$\p(E_2)\geq 1-p^2\exp(-10k\log(p/k)).$$
\end{lemma}
\begin{proof}[\textbf{Proof of Lemma~\ref{thm1techlem2}}.]
The proof of this lemma is on the event $\cap_j E_j^{\lambda}$ defined in Lemma~\ref{new-lemma-lambda} which happens with probability at least $1-p^2\exp(-10k\log(p/k))$.

Note that (deterministically) one has
\begin{align}
	\|\hat{\B{y}}_j\|_2^2-\|\B{y}_j^*\|_2^2 & = \left\Vert\sum_{i:i\neq j} (\beta_{ij}^*-\hat{\beta}_{ij})\B{x}_i+\B{\varepsilon}_j\right\Vert_2^2 - \|\B{\varepsilon}_j\|_2^2 \nonumber\\
	& = \left\Vert\sum_{i:i\neq j} (\beta_{ij}^*-\hat{\beta}_{ij})\B{x}_i\right\Vert_2^2 + 2\B{\varepsilon}_j^{\top}\sum_{i\neq j} (\beta_{ij}^*-\hat{\beta}_{ij})\B{x}_i \nonumber \\
	& = \left\Vert\sum_{i:i\neq j} (\beta_{ij}^*-\hat{\beta}_{ij})\B{x}_i\right\Vert_2^2 + 2\B{\varepsilon}_j^{\top}\frac{\sum_{i\neq j} (\beta_{ij}^*-\hat{\beta}_{ij})\B{x}_i}{\|\sum_{i:i\neq j} (\beta_{ij}^*-\hat{\beta}_{ij})\B{x}_i\|_2}\|\sum_{i\neq j} (\beta_{ij}^*-\hat{\beta}_{ij})\B{x}_i\|_2 \nonumber \\
	& \stackrel{(A)}{\geq}  \frac{1}{2}\left\Vert\sum_{i:i\neq j} (\beta_{ij}^*-\hat{\beta}_{ij})\B{x}_i\right\Vert_2^2 - 2\left(\B{\varepsilon}_j^{\top}\frac{\sum_{i:i\neq j} (\beta_{ij}^*-\hat{\beta}_{ij})\B{x}_i}{\|\sum_{i:i\neq j} (\beta_{ij}^*-\hat{\beta}_{ij})\B{x}_i\|_2}\right)^2\label{thm2-ineq2}
\end{align}
where $(A)$ is by the inequality $2ab\geq -2a^2-b^2/2$.

Note that
\begin{align}
	&\p\left(\frac{1}{n\hat{\sigma}_j^2}\left(\B{\varepsilon}_j^{\top}\frac{\sum_{i:i\neq j} (\beta_{ij}^*-\hat{\beta}_{ij})\B{x}_i}{\|\sum_{i:i\neq j} (\beta_{ij}^*-\hat{\beta}_{ij})\B{x}_i\|_2}\right)^2-{\tilde{c}_{\lambda}\lambda}|\hat{S}_j| > \frac{t}{n l_{\sigma}^2}\right) \nonumber \\ \stackrel{(a)}{\leq} &   
	\p\left(\frac{1}{nl_{\sigma}^2}\left(\B{\varepsilon}_j^{\top}\frac{\sum_{i:i\neq j} (\beta_{ij}^*-\hat{\beta}_{ij})\B{x}_i}{\|\sum_{i:i\neq j} (\beta_{ij}^*-\hat{\beta}_{ij})\B{x}_i\|_2}\right)^2-{\tilde{c}_{\lambda}\lambda}|\hat{S}_j| > \frac{t}{n l_{\sigma}^2}\right) \nonumber \\ \leq &  \p\left(\sup_{\substack{ \B{v}\in\R^{p}\\ \text{Supp}(\B{v})=\hat{S}_j\cup S_j^*}}\left(\B{\varepsilon}_j^{\top}\frac{\B{X}\B{v}}{\|\B{X}\B{v}\|_2}\right)^2 -nl_{\sigma}^2\tilde{c}_{\lambda}\lambda |\hat{S_j}| > t\right) 
\end{align}
where $(a)$ is by $\hat{\sigma}_j\geq l_{\sigma}$ and the second inequality is true as 
$$\left(\B{\varepsilon}_j^{\top}\frac{\sum_{i:i\neq j} (\beta_{ij}^*-\hat{\beta}_{ij})\B{x}_i}{\|\sum_{i:i\neq j} (\beta_{ij}^*-\hat{\beta}_{ij})\B{x}_i\|_2}\right)^2\leq \sup_{\substack{ \B{v}\in\R^{p}\\ \text{Supp}(\B{v})=\hat{S}_j\cup S_j^*}}\left(\B{\varepsilon}_j^{\top}\frac{\B{X}\B{v}}{\|\B{X}\B{v}\|_2}\right)^2. $$
As a result, by Lemma~\ref{new-lemma-lambda} on event $\cap_j E_j^{\lambda}$,
\begin{align}
	\frac{1}{n\hat{\sigma}_j^2}\left(\B{\varepsilon}_j^{\top}\frac{\sum_{i:i\neq j} (\beta_{ij}^*-\hat{\beta}_{ij})\B{x}_i}{\|\sum_{i:i\neq j} (\beta_{ij}^*-\hat{\beta}_{ij})\B{x}_i\|_2}\right)^2&\lesssim \frac{1}{n\hat{\sigma}_j^2}\left(u_{\sigma}^2k\log(2p/k)+nl_{\sigma}^2\tilde{c}_{\lambda}\lambda|\hat{S}_j|\right) \nonumber \\
	& \lesssim \frac{u_{\sigma}^2k\log(2p/k)}{nl_{\sigma}^2} + \tilde{c}_{\lambda}\lambda |\hat{S}_j|\label{thm2-ineq3}
\end{align}
with high probability. 

By~\eqref{thm2-ineq2} and~\eqref{thm2-ineq3},
\begin{align*}
	\sum_{j\in[p]}\frac{1}{2nu_{\sigma}^2}\left\Vert\sum_{i:i\neq j} (\beta_{ij}^*-\hat{\beta}_{ij})\B{x}_i\right\Vert_2^2 & \leq \sum_{j\in[p]}\frac{1}{2n\hat{\sigma}_j^2}\left\Vert\sum_{i:i\neq j} (\beta_{ij}^*-\hat{\beta}_{ij})\B{x}_i\right\Vert_2^2 \\ 
	&\leq \sum_{j\in [p]}\left[ \frac{ \|\hat{\B{y}}_j\|_2^2-\|\B{y}_j^*\|_2^2}{n\hat{\sigma}_j^2}+\frac{2}{n\hat{\sigma}_j^2}\left(\B{\varepsilon}_j^{\top}\frac{\sum_{i:i\neq j} (\beta_{ij}^*-\hat{\beta}_{ij})\B{x}_i}{\|\sum_{i:i\neq j} (\beta_{ij}^*-\hat{\beta}_{ij})\B{x}_i\|_2}\right)^2\right] \\
	& \lesssim\sum_{j\in [p]} \frac{ \|\hat{\B{y}}_j\|_2^2-\|\B{y}_j^*\|_2^2}{n\hat{\sigma}_j^2} +\frac{u_{\sigma}^2kp\log(2p/k)}{l_{\sigma}^2n}+\tilde{c}_{\lambda}\lambda \sum_{j\in[p]} |\hat{S}_j|
\end{align*}
with high probability.
\end{proof}

\begin{lemma}\label{E3lemma}
	Let the event $E_3$ be defined as
	\begin{equation}
		E_3=\left\{ \frac{1}{n} \sum_{\in[p]} \left\Vert\sum_{i:i\neq j} (\beta_{ij}^*-\hat{\beta}_{ij})\B{x}_i\right\Vert_2^2  \gtrsim \sum_{j\in[p]} \sum_{i:i\neq j} (\beta_{ij}^*-\hat{\beta}_{ij})^2\right\}.
	\end{equation}
	Then, under the assumptions of Theorem~\ref{lowdimthm}, 
	$$\p(E_3)\geq 1-\exp(-10p).$$
\end{lemma}
\begin{proof}[\textbf{Proof of Lemma~\ref{E3lemma}}.]
Let us define the event $A$ as
\begin{equation}\label{defn-A}
	A=\left\{\sigma_{\min}(\B{X}_S)\gtrsim \sqrt{n}: S\subseteq[p]\right\}
\end{equation}
where $\B{X}_S$ is the submatrix of $\B{X}$ with columns sampled from $S$.
By Lemma~\ref{minnorm2}, take $t=\tilde{c}\sqrt{p}$ for some $\tilde{c}>0$. Then, for a given $S\subseteq[p]$, 
\begin{align}
	\frac{1}{\sqrt{n}}\sigma_{\min}(\B{X}_S)&\gtrsim \sqrt{\lambda_{\min}(\B{\Sigma}^*_{S,S})}\left(1-c_0\left(\sqrt{\frac{|S|}{n}}+\tilde{c}\sqrt{\frac{p}{n}}\right)\right)\nonumber \\
	&\stackrel{(a)}{\gtrsim} \left(1-c_0(1+\tilde{c})\sqrt{\frac{p}{n}}\right)
\end{align}
with probability at least $1-\exp(-C\tilde{c}^2 p)$, where $(a)$ is by Assumption~\ref{assum1-5}. 

Note that:
\begin{align}
	\p(A^c)& \leq \sum_{S\subseteq[p]} \p\left(   \frac{1}{\sqrt{n}}\sigma_{\min}(\B{X}_S)\lesssim 
	\left(1-c_0(1+\tilde{c})\sqrt{\frac{p}{n}}\right)\right)\nonumber \\
	& \leq 2^p \exp(-C\tilde{c}p) =\exp((-C\tilde{c}+\log 2)p).
\end{align}
Take $\tilde{c}$ large enough such that $(-C\tilde{c}+\log 2)\leq -10$ and $n\gtrsim p$ such that 
$$\left(1-c_0(1+\tilde{c})\sqrt{\frac{p}{n}}\right)\geq 0.1$$
so we achieve $\p(A)\geq 1-\exp(-10p)$.

On event $A$ as defined in~\eqref{defn-A}, we have
\begin{align}
	\frac{1}{n} \sum_{\in[p]} \left\Vert\sum_{i:i\neq j} (\beta_{ij}^*-\hat{\beta}_{ij})\B{x}_i\right\Vert_2^2 & = \frac{1}{n} \sum_{j\in[p]} \left\Vert\B{X}_{\CS_j}(\hat{\B{\beta}}_{\CS_j,j}-\B{\beta}^*_{\CS_j,j})\right\Vert_2^2\nonumber\\
	& \geq \frac{1}{n} \sum_{j\in[p]} \sigma^2_{\min}(\B{X}_{\CS_j})\left\Vert \hat{\B{\beta}}_{\CS_j,j}-\B{\beta}^*_{\CS_j,j}\right\Vert_2^2 \nonumber \\
	& \gtrsim \sum_{j\in[p]} \sum_{i:i\neq j} (\beta_{ij}^*-\hat{\beta}_{ij})^2\label{thm1proofineq1}
\end{align}
where $\B{\beta}_{S_j,j}\in\R^{|S_j|}$ is the vector containing the values $\{\beta_{i,j}\}$ for $i\in S_j$. The last inequality above is a result of event $A$.
\end{proof}

\subsubsection{Proof of Theorem \ref{lowdimthm}}
The proof is on the intersection of events $E_1,E_2,E_3$ from Lemmas~\ref{thm1techlem1},~\ref{thm1techlem2} and~\ref{E3lemma} which happens with probability at least
\begin{equation}\label{thm1prob}
	1-\exp(-10p)-p(k/p)^{10}-p^2\exp(-10k\log(p/k)).
\end{equation}
On the intersection of events $E_1,E_2,E_3$, we have:
\begin{align}
	&    \sum_{j\in[p]}(\hat{\sigma}_j-\sigma^*_j)^2 + \frac{1}{u_{\sigma}^2}\sum_{j\in[p]} \sum_{i:i\neq j} (\beta_{ij}^*-\hat{\beta}_{ij})^2 \nonumber \\
	& \stackrel{(a)}{\lesssim} \sum_{j\in[p]}(\hat{\sigma}_j-\sigma^*_j)^2 +    \frac{1}{nu_{\sigma}^2} \sum_{\in[p]} \left\Vert\sum_{i:i\neq j} (\beta_{ij}^*-\hat{\beta}_{ij})\B{x}_i\right\Vert_2^2 \nonumber \\
	& \stackrel{(b)}{\lesssim} \sum_{j\in[p]}(\hat{\sigma}_j-\sigma^*_j)^2+\sum_{j\in[p]} \frac{ \|\hat{\B{y}}_j\|_2^2-\|\B{y}_j^*\|_2^2}{n\hat{\sigma}_j^2}+\frac{u_{\sigma}^2kp\log(2p/k)}{l_{\sigma}^2n}+\tilde{c}_{\lambda}\lambda\sum_{j\in[p]}|\hat{S}_j|\label{final-mid-helper}\\
	& \stackrel{(c)}{\lesssim} \frac{u_{\sigma}^2kp\log(2p/k)}{l_{\sigma}^2n} + \lambda\sum_{j\in[p]}|S_j^*|\nonumber \\
	& \lesssim \frac{u_{\sigma}^2kp\log(2p/k)}{l_{\sigma}^2n}
\end{align}
where $(a)$ uses description of event $E_3$ from Lemma~\ref{E3lemma}, $(b)$ is from event $E_2$ from Lemma~\ref{thm1techlem2} and $(c)$ is from event $E_1$ from Lemma~\ref{thm1techlem1}, and choosing $\tilde{c}_{\lambda}$ such that the coefficients of $\lambda|\hat{S}_j|$ from event $E_1$ and~\eqref{final-mid-helper} cancel each other. The final inequality follows from~\eqref{thm1-lambda} and Assumption~\ref{assum1-4}.

\subsection{Proof of Theorem \ref{lowdimthm2}}
\begin{proof}
Based on the change of variable $\theta_{jj}=1/\sigma_j^2$ and $\beta_{ij}=-
\theta_{ji}/\theta_{jj}$ for $\B{\Theta}^*,\hat{\B{\Theta}}$, one has for $i\neq j\in[p]$ 
\begin{align}
	|\hat{\theta}_{ji}-\theta_{ji}^*| & = \left\vert \frac{\hat{\beta}_{ij}}{\hat{\sigma}_j^2}- \frac{{\beta}^*_{ij}}{{({\sigma}_j^*})^2} \right\vert\nonumber \\
	& \stackrel{(a)}{\leq} \frac{|\hat{\beta}_{ij}({\sigma_j^*})^2-{\beta}^*_{ij}\hat{\sigma}_j^2|}{l_{\sigma}^4} \nonumber \\
	& \leq \frac{|(\hat{\beta}_{ij}-{\beta}^*_{ij})({\sigma_j^*})^2|+|{\beta}^*_{ij}(\hat{\sigma}_j^2-({\sigma_j^*})^2)|}{l_{\sigma}^4}\nonumber \\
	& \leq \frac{|(\hat{\beta}_{ij}-{\beta}^*_{ij})({\sigma_j^*})^2|}{l_{\sigma}^4} + \frac{|{\beta}^*_{ij}||\hat{\sigma}_j-{\sigma_j^*}||\hat{\sigma}_j+{\sigma_j^*}|}{l_{\sigma}^4}\nonumber\\
	& \lesssim \frac{|\hat{\beta}_{ij}-{\beta}^*_{ij}|u_{\sigma}^2}{l_{\sigma}^4} + \frac{|\hat{\sigma}_j-{\sigma_j^*}|u_{\sigma}^2}{l_{\sigma}^4}\label{thm2ineq1}
\end{align}
where $(a)$ is true as $\hat{\sigma}_j,\sigma_j^*\leq l_{\sigma}$ and the last inequality is due to Assumption~\ref{assum1-2}. 

Similarly, we can obtain a bound on the error of the diagonal entry:
\begin{align}
	|\hat{\theta}_{jj}-\theta_{jj}^*| & = \left\vert \frac{1}{\hat{\sigma}_j^2}- \frac{1}{{({\sigma}_j^*})^2} \right\vert\nonumber \\
	& \leq \frac{|({\sigma_j^*})^2-\hat{\sigma}_j^2|}{l_{\sigma}^4} \nonumber \\
	& \lesssim  \frac{|\hat{\sigma}_j-{\sigma_j^*}|u_{\sigma}^2}{l_{\sigma}^4}.\label{thm2ineq2}
\end{align}
As a result, on the event of Theorem~\ref{lowdimthm},
\begin{align}
	\left\Vert\hat{\B{\Theta}}-\B{\Theta}^*\right\Vert_F^2 & = \sum_{i\neq j\in[p]} |\hat{\theta}_{ji}-\theta_{ji}^*|^2 + \sum_{j\in[p]} |\hat{\theta}_{jj}-\theta_{jj}^*|^2\nonumber \\
	& \stackrel{(a)}{\lesssim} \frac{u_{\sigma}^4}{l_{\sigma}^8}\left[\sum_{i\neq j\in[p]} |\hat{\beta}_{ij}-\beta_{ij}^*|^2 + \sum_{j\in[p]} |\hat{\sigma}_{j}-\sigma_{j}^*|^2\right] \nonumber \\
	& \stackrel{(b)}{\lesssim}  \frac{(u_{\sigma}^6+u_{\sigma}^8)kp\log(2p/k)}{l_{\sigma}^{10}n}
\end{align}
with high probability, where $(a)$ is due to~\eqref{thm2ineq1} and~\eqref{thm2ineq2} and $(b)$ is because of Theorem~\ref{lowdimthm}.
\end{proof}

\subsection{Proof of Theorem~\ref{supthm}}\label{app:proof-supthm}

Let us introduce some notation that we will be using in this proof. 
\subsubsection{Notation} For $S\subseteq[p]$, we denote the projection matrix onto the column span of $\B{X}_S$ by $\B{P}_{\B{X}_{S}}$. Note that if $\B{X}_S$ has linearly independent columns, $\B{P}_{\B{X}_{S}}=\B{X}_{S}(\B{X}_{S}^{\top}\B{X}_{S})^{-1}\B{X}_{S}^{\top}$. In our case, as the data is drawn from a multivariate normal distribution with a full-rank covariance matrix, for any $S\subseteq[p]$ with $|S|<n$, $\B{X}_S$ has linearly independent columns with probability one. We define the operator norm of $\B{A}\in\R^{p_1\times p_2}$ as
\begin{equation*}
	\begin{aligned}
		\|\B{A}\|_{\text{op}}&=\max_{\substack {\B{x}\in\R^{p_2}\\\B{x}\neq 0}}\frac{\|\B{Ax}\|_2}{\|\B{x}\|_2}.
	\end{aligned}
\end{equation*}
The solution to the least squares problem with the support restricted to $S$,
\begin{equation}\label{limitedls}
	\min_{\B{\beta}_{S^c}=0}\frac{1}{n}\|\B{y}-\B{X\beta}\|_2^2
\end{equation}
for $\B{y}\in\R^n$ and $\B{X}\in\R^{n\times p}$ is given by $$\B{\beta}_{S}=(\B{X}_{S}^{\top}\B{X}_{S})^{-1}\B{X}_{S}^{\top}\B{y}.$$
Note that as in our case the data is drawn from normal distribution with a full-rank covariance matrix, $(\B{\beta}_S)_i\neq 0$ for $i\in S$ with probability one. Consequently, we denote the optimal objective in~\eqref{limitedls} by 
\begin{equation}\label{optimalobj}
	\CR_{S}(\B{y})= \frac{1}{n}\B{y}^{\top}(\B{I}_n-\B{P}_{\B{X}_{S}})\B{y}.
\end{equation}
For $S_1,S_2\subseteq [p]$, $\B{\Sigma}\in\R^{p\times p}$ positive definite and $S_0=S_2\setminus S_1$, we let
\begin{equation}\label{schur-defined}
	\schur{\B{\Sigma}}{S_1}{S_2} = \B{\Sigma}_{S_0,S_0}- \B{\Sigma}_{S_0,S_1}\B{\Sigma}^{-1}_{S_1,S_1}\B{\Sigma}_{S_1,S_0}.
\end{equation}
Note that $\schur{\B{\Sigma}}{S_1}{S_2}$ is the Schur complement of the matrix
\begin{equation}\label{gs1s2}
	\B{\Sigma}(S_1,S_2)=  \begin{bmatrix}
		\B{\Sigma}_{S_1,S_1} & \B{\Sigma}_{S_1,S_0} \\
		\B{\Sigma}_{S_0,S_1} & \B{\Sigma}_{S_0,S_0}
	\end{bmatrix}.
\end{equation}
Let $S_j^*,\hat{S}_j,t_j,\tilde{t}_j$ for $j\in[p]$ be defined as
\begin{equation}\label{sjdef-sup}
	\begin{aligned}
		\hat{\CS}_j & =\{i\in[p]:i\neq j,\hat{\beta}_{ij}\neq 0\},\\
		{\CS}_j^* & =\{i\in[p]:i\neq j,{\beta}_{ij}^*\neq 0\},\\
		t_j &= |S_j^*\setminus\hat{S}_j|,\\
		\bar{t}_j &= |\hat{S}_j\setminus S_j^*|,\\
		\tilde{t}_j &= \left\vert (S_j^*\setminus\hat{S}_j)\cap \{j+1,\cdots,p\}\right\vert.
	\end{aligned}
\end{equation}
Moreover, for $S\subseteq [p]$, let
$\CS_j^0=\CS_j^*\setminus \CS$ and $\tilde{\CS}^0_j=\CS_j^0\cap \{j+1,\cdots,p\}$.  We let $\B{\beta}^*_{\CS^0_j,j}$ to be the vector $\{\beta^*_{ij}\}_{i\in \CS^0_j}$ and $\hat{\B{\Sigma}}=\B{X}^\top\B{X}/n$.
Let us define for $j\in[p]$,
\begin{equation}\label{hdefine}
	h_j(\sigma,S) = \log(\sigma)+\frac{\mathcal{L}_{S}(\B{x}_j)}{2\sigma^2}.
\end{equation}
\subsubsection{Roadmap of proof} At optimality of Problem~\eqref{main-sup}, the optimal objective is given as 
$$\sum_{j=1}^p\left\{h_j(\hat{\sigma}_j,\hat{S}_j)+\lambda |\hat{S}_j|\right\}.$$
Similarly, if we fix the value of $z_{ij}$ to $z^*_{ij}$, the objective value is 
$$\sum_{j=1}^p\left\{h_j(\tilde{\sigma}_j,{S}_j^*)+\lambda |{S}_j^*|\right\}$$
where $\tilde{\sigma}_j^2$ are optimal variance values from Problem~\eqref{main-sup} on the underlying support. Next, we divide variables into two parts based on the value of $\mathcal{L}_{\hat{S}_j}$: 
\begin{align}
	\mathcal{J}=\left\{j\in[p]: \mathcal{L}_{\hat{S}_j}(\B{x}_j)\geq \ell\right\}.
\end{align}
Consider the function $x \mapsto f(x)$ 
$$f(x)=\log(x) + \frac{a}{2x^2}$$
on $x>0$ for a fixed $a>0$. The function is minimized for $x^2=a$. Therefore, for $j\in\mathcal{J}$ we have $\hat{\sigma}_j^2 = \mathcal{L}_{\hat{S}_j}(\B{x}_j)\geq \ell$. This leads to $h_j(\hat{\sigma}_j,\hat{S}_j)=\log(\mathcal{L}_{\hat{S}_j}(\B{x}_j))/2+1/2$. Moreover, 
$$f'(x)=\frac{1}{x}-\frac{a}{x^3}=\frac{1}{x^3}(x^2-a)\geq 0$$
for $x\geq \sqrt{a}$, showing $f(x)$ is minimized for $x=\sqrt{a}$ for $x\geq \sqrt{a}$. As a result, for $j\in \mathcal{J}^c$, we have $\hat{\sigma}_j^2=\ell$. The optimal objective of Problem~\eqref{main-sup} is given as
\begin{align}
	&\sum_{j=1}^p\left\{h_j(\hat{\sigma}_j,\hat{S}_j)+\lambda |\hat{S}_j|\right\}\nonumber \\ =&\sum_{j\in\mathcal{J}^c}\left\{h_j(\sqrt{\ell},\hat{S}_j)+\lambda|\hat{S}_j|\right\} + \sum_{j\in\mathcal{J}}\left\{h_j(\sqrt{\mathcal{L}_{\hat{S}_j}(\B{x}_j)},{\hat{S}_j})+\lambda|\hat{S}_j|\right\}\nonumber \\
	=& \sum_{j\in\mathcal{J}^c}\left\{h_j(\sqrt{\ell},\hat{S}_j)+\lambda|\hat{S}_j|\right\} + \sum_{j\in\mathcal{J}}\left\{\frac{1}{2}\log(\mathcal{L}_{\hat{S}_j}(\B{x}_j))+\frac{1}{2}+\lambda|\hat{S}_j|\right\}\label{h-twopart}.  
\end{align}
We also will show the optimal cost on the correct support is 
\begin{align}
	\sum_{j=1}^p\left\{h_j(\tilde{\sigma}_j,{S}_j^*)+\lambda |{S}_j^*|\right\}&= \sum_{j=1}^p\left\{\frac{1}{2}\log(\mathcal{L}_{{S}^*_j}(\B{x}_j))+\frac{1}{2}+\lambda|{S}_j^*|\right\}.
\end{align}
Our approach for this proof is to first show that: $\mathcal{J}^c=\emptyset$, and second, for $j\in\mathcal{J}$, the support is estimated correctly by comparing the objective value of optimal and correct support. \\
We note that by Assumption~\ref{assum2degree}, $t_j\leq k$. Let us define the following basic events for $j\in[p]$ and $S\subseteq[p]$:
\begin{equation}\label{bigblockevents}
	\begin{aligned}
		\mathcal{E}_1(j,S) &= \left\{ {(\B{\beta}^*_{\CS^0_j,j})}^{\top} (\schur{\hat{\B{\Sigma}}}{\CS}{\CS_j^*}) \B{\beta}^*_{\CS^0_j,j} \geq   0.2\eta \frac{|\tilde{\CS}^0_j|\log p}{n}\right\} \\
		\mathcal{E}_2(j,S) &= \left\{\frac{1}{n}\left\Vert\B{X}_{S_j^0}\B{\beta}^*_{S_j^0,j}\right\Vert_2^2\leq 4\frac{|S_j^0|}{k}\right\} \\
		\mathcal{E}_3(j,S) &= \left\{ \frac{1}{n}\B{\varepsilon}_j^{\top} (\B{I}_n-\B{P}_{\B{X}_{\CS}})\B{X}_{\CS^0_j}\B{\beta}^*_{\CS^0_j,j}\geq -c_{t_1}\sigma^*_j \sqrt{{(\B{\beta}^*_{\CS^0_j,j})}^{\top} (\schur{\hat{\B{\Sigma}}}{\CS}{\CS_j^*}) \B{\beta}^*_{\CS^0_j,j}}\sqrt{\frac{(|S_j^*\setminus S|+|S\setminus S_j^*|)\log p}{n}}\right\} \\
		\mathcal{E}_4(j,S) &= \left\{\B{\varepsilon}_j^{\top} (\B{P}_{\B{X}_{\CS}}-\B{P}_{\B{X}_{\CS_j^*}}) \B{\varepsilon}_j\leq  c_{t_2}(\sigma_j^*)^2(|S_j^*\setminus S|+|S\setminus S_j^*|)\log p\right\} \\
		\mathcal{E}_5(j,S) &= \left\{ -c_{t_3}(\sigma_j^*)^2k\log p\leq\B{\varepsilon}_j^{\top} \B{P}_{\B{X}_{\CS}} \B{\varepsilon}_j\leq  c_{t_3}(\sigma_j^*)^2k\log p\right\} \\
		\mathcal{E}_6(j) &= \left\{\CR_{\hat{\CS}_j}(\B{x}_j) \geq \CR_{\CS^*_j}(\B{x}_j)+\frac{3}{20}\eta\tilde{t}_j\frac{\log p}{n} -(\sigma_j^*)^2\frac{(t_j+\bar{t}_j)\log p}{n}(4c_{t_1}^2+c_{t_2})-\frac{4t_j}{k}\right\} \\
		\mathcal{E}_7(j) &= \left\{\ell<\frac{2}{3}(\sigma_j^*)^2 \leq \CR_{\CS^*_j}(\B{x}_j)\leq \frac{4}{3}(\sigma_j^*)^2\right\} 
	\end{aligned}
\end{equation}
for some numerical constants $c_{t_1},c_{t_2},c_{t_3}>0$. The following lemmas establish that the events defined above hold with high probability. The proof of some of these results are similar to results shown in~\citet{behdin2021integer}, building and improving upon the results of~\citet{fan2020best}. 

\subsubsection{Useful Lemmas}
\begin{lemma}\label{suptechlem1}
	Under the assumptions of Theorem~\ref{supthm}, we have
	\begin{equation}\label{lem8-1}
		\p\left(\bigcap_{j\in[p]}\bigcap_{\substack{S_j\subseteq[p]\setminus\{j\} \\|S_j|\leq k }}\mathcal{E}_1(j,S_j)\right)\geq 1-p^{-8}.
	\end{equation}
	and
	\begin{equation}\label{lem8-2}
		\p\left(\bigcap_{j\in[p]}\bigcap_{\substack{S_j\subseteq[p]\setminus\{j\}  }}\mathcal{E}_2(j,S_j)\right)\geq 1-p^{-8}
	\end{equation}
	where $\mathcal{E}_1,\mathcal{E}_2$ are defined in~\eqref{bigblockevents}.
\end{lemma}
\begin{proof}[\textbf{Proof of Lemma~\ref{suptechlem1}}.]
Let the events $\mathcal{E}_0(S)$ for $S\subseteq[p]$ with $|S|\leq 2k$ and $\mathcal{E}_0$ be defined as
\begin{equation*}
	\begin{aligned}
		\mathcal{E}_0(S)&=\left\{\left\|\hat{\B{\Sigma}}_{S,S}-\B{\Sigma}^*_{S,S}\right\|_{\text{op}} \lesssim \sqrt{\frac{k\log p}{n}}\right\},\\
		\mathcal{E}_0 &= \bigcap_{\substack{S\subseteq[p] \\  |S|\leq 2k }}\mathcal{E}_0(S).
	\end{aligned}
\end{equation*}
One has (for example, by Theorem~5.7 of~\citet{rigollet2015high} with $\delta = \exp(-11k\log p)$)
$$\p(\mathcal{E}_0(S))\geq 1-\exp(-11k\log p)$$
as $n=c_n k\log p$ is sufficiently large and by Assumption~\ref{kappaassum}, $\|\B{\Sigma}^*_{S,S}\|_{\text{op}}\lesssim 1$. As a result, by union bound
\begin{align*}
	\p(\mathcal{E}_0) &\geq 1-\sum_{\substack{ S\subseteq[ p] \\ |S|\leq 2k}}(1-\p(\mathcal{E}_0(S)))\geq 1-\sum_{t=1}^{2k}{p\choose t}\exp(-11k\log p) \\
	&\geq 1-\sum_{t=1}^{2k} p^{2k} p^{-11k} \geq 1-p\times p^{-9}=1-p^{-8}.
\end{align*}
The rest of the proof is on event $\mathcal{E}_0$. 

\noindent {\it Proof of~\eqref{lem8-1}:}
We first consider the proof of~\eqref{lem8-1}. Consequently, as $|S_j|,|S_j^*|\leq k$,
\begin{align}
	\|\hat{\B{\Sigma}}(\CS,\CS_j^*)-\B{\Sigma}^*(\CS,\CS_j^*)\|_{\text{op}} &  \leq c_b \sqrt{\frac{k\log p}{n}} :=\pi\label{stewineq2}
\end{align}
for some constant $c_b>0$ where $\B{\Sigma}(S_1,S_2)$ is defined in~\eqref{gs1s2}. Let $c_n$ be sufficiently large such that $\pi<0.1$. Therefore, one has 
\begin{align*}
	\lambda_{\min}(\schur{\hat{\B{\Sigma}}}{\CS}{\CS_j^*})\stackrel{(a)}{\geq} \lambda_{\min}({\hat{\B{\Sigma}}}({\CS},{\CS_j^*})) 
	&\stackrel{(b)}{\geq}  \lambda_{\min}({\B{\Sigma}^*}({\CS},{\CS_j^*})) - \|{\hat{\B{\Sigma}}}({\CS},{\CS_j^*}) - {\B{\Sigma}^*}({\CS},{\CS_j^*})\|_{\text{op}}  \geq\kappa^2 - 0.1>0.2
\end{align*}
where $(a)$ is by Corollary 2.3 of \citet{zhang2006schur}, $(b)$ is due to Weyl's inequality and the last inequality is by Assumption~\ref{kappaassum}. Finally,
\begin{align*}
	{(\B{\beta}^*_{\CS^0_j,j})}^{\top} (\schur{\hat{\B{\Sigma}}}{\CS}{\CS_j^*}) \B{\beta}^*_{\CS^0_j,j} \geq \lambda_{\min}(\schur{\hat{\B{\Sigma}}}{\CS}{\CS_j^*})\|\B{\beta}^*_{\CS^0_j,j}\|_2^2 \geq 0.2\|\B{\beta}^*_{\CS^0_j,j}\|_2^2\geq 0.2\eta \frac{|\tilde{\CS}_j^0|\log p}{n}
\end{align*}
where the last inequality is achieved by substituting $\beta_{\min}$ condition from Assumption~\ref{betaminassum}. This completes the proof of~\eqref{lem8-1}.

\noindent {\it Proof of~\eqref{lem8-2}:}
We now proceed to prove~\eqref{lem8-2}.  Note that by Weyl's inequality, 
$$\lambda_{\max}(\hat{\B{\Sigma}}_{S_j^0,S_j^0})\leq \lambda_{\max}(\B{\Sigma}^*_{S_j^0,S_j^0}) + \|\B{\Sigma}^*_{S_j^0,S_j^0}-\hat{\B{\Sigma}}_{S_j^0,S_j^0}\|_{\text{op}}\leq \lambda_{\max}(\B{\Sigma}^*_{S_j^0,S_j^0}) +0.1 \leq 4 $$
where the second inequality is due to event $\mathcal{E}_0$ (note that $|S_j^0|\leq k$)  and the last inequality is due to Assumption~\ref{kappaassum}. Finally, note that
$$\frac{1}{n}\left\Vert\B{X}_{S_j^0}\B{\beta}^*_{S_j^0,j}\right\Vert_2^2=(\B{\beta}^*_{S_j^0,j})^\top \frac{\B{X}_{S_j^0}^\top\B{X}_{S_j^0}}{n}\B{\beta}^*_{S_j^0,j}\leq \lambda_{\max}(\hat{\B{\Sigma}}_{S_j^0,S_j^0})\|\B{\beta}^*_{S_j^0,j}\|_2^2\leq 4\|\B{\beta}^*_{S_j^0,j}\|_2^2 \leq 4\frac{|S_j^0|}{k}$$
where the last inequality is due to Assumption~\ref{beta-bound}.
\end{proof}
\begin{lemma}\label{suptechlem2}
	Under the assumptions of Theorem~\ref{supthm}, we have
	\begin{equation}
		\p\left(\bigcap_{j\in[p]}\bigcap_{S\subseteq[p]\setminus \{j\} }\mathcal{E}_3(j,S)\right)\geq 1-2kp^{-7}
	\end{equation}
	where $\mathcal{E}_3$ is defined in~\eqref{bigblockevents}.
\end{lemma}
\begin{proof}[\textbf{Proof of Lemma~\ref{suptechlem2}}.]
The proof follows a similar path to the proof of Lemma~13 of~\citet{behdin2021integer}. Fix $j\in[p], S\subseteq[p]\setminus\{j\}$, and let $t=|S_j^*\setminus S|,\bar{t}=|S\setminus S_j^*|,S_j^0=S_j^*\setminus S$.  Note that if $t=0$, the lemma is trivial. Therefore, without loss of generality we assume $t\geq 1$. Let
\begin{equation*}
	\B{\gamma}^{(j,\CS)}= (\B{I}_n-\B{P}_{\B{X}_{\CS}})\B{X}_{\CS^0_j}\B{\beta}^*_{\CS^0_j,j}.
\end{equation*}
Following the same calculations in Lemma~13 of~\citet{behdin2021integer}, one has
\begin{align}
	\p\left(\frac{\B{\varepsilon}_{j}^\top \B{\gamma}^{(j,\CS)}}{\|\B{\gamma}^{(j,\CS)}\|_2}< -x\right) \leq  \exp\left(-\frac{x^2}{8({\sigma_j^*})^2}+t\log 5\right).
\end{align}
Take 
$$x^2 = {8\xi^2 ({\sigma_j^*})^2 (t+\bar{t})\log p }$$
for some universal constant $\xi>0$ that is sufficiently large, and noting
$$\|\B{\gamma}^{(j,\CS)}\|_2=\sqrt{{(\B{\beta}^*_{\CS^0_j,j})}^\top (\schur{\B{\hat{\Sigma}}}{\CS}{\CS_j^*}) \B{\beta}^*_{\CS^0_j,j}},$$
we achieve
\begin{equation*}
	\p\left(\frac{\B{\varepsilon}_{j}^\top \B{\gamma}^{(j,\CS)}}{n}<- \sqrt{8}\xi\sigma_j^*\sqrt{{(\B{\beta}^*_{\CS^0_j,j})}^\top (\schur{\B{\hat{\Sigma}}}{\CS}{\CS_j^*}) \B{\beta}^*_{\CS^0_j,j}}\sqrt{\frac{(t+\bar{t})\log p}{n}}\right) \leq \exp(-10(t+\bar{t})\log p).
\end{equation*}
Finally, we complete the proof by using union bound over all possible choices of $j,t,S$. As a result, the probability of the desired event in the lemma being violated is bounded as
\begin{align*}
	\sum_{j=1}^p \sum_{t=1}^k \sum_{\bar{t}=0}^{p-k}\sum_{\substack{ S\subseteq[p]\setminus \{j\} \\ |S_j^*\setminus S|=t \\|S\setminus S_j^*|=\bar{t}}}\exp(-10(t+\bar{t})\log p)&= \sum_{j=1}^p \sum_{t=1}^k \sum_{\bar{t}=0}^{p-k}{k\choose t}{p-k \choose \bar{t}}\exp(-10(t+\bar{t})\log p) \\
	& \leq p \sum_{t=1}^k\sum_{\bar{t}=0}^p p^t p^{\bar{t}}\exp(-10(t+\bar{t})\log p) \\
	& \leq p \sum_{t=1}^k\sum_{\bar{t}=0}^p \exp(-9(t+\bar{t})\log p)\\
	& \leq p \sum_{t=1}^k\sum_{\bar{t}=0}^p \exp(-9\log p)\\
	& \leq kp^2 \frac{p+1}{p} p^{-9} = 2kp^{-7}.
\end{align*}
\end{proof}
\begin{lemma}\label{suptechlem3}
	Under the assumptions of Theorem~\ref{supthm} we have
	\begin{equation}
		\p\left(\bigcap_{j\in[p]}\bigcap_{\substack{ S\subseteq[p]\setminus \{j\}}}\mathcal{E}_4(j,S)\right)\geq 1-8kp^{-7}
	\end{equation}
	where $\mathcal{E}_4$ is defined in~\eqref{bigblockevents}.
\end{lemma}
\begin{proof}[\textbf{Proof of Lemma~\ref{suptechlem3}}.]
The proof of this lemma follows a similar path to the proof of Lemma~15 of~\citet{behdin2021integer}. Fix $j\in[p], S\subseteq[p]\setminus\{j\}$, and let 
$$t=|S_j^*\setminus S|,\bar{t}=|S\setminus S_j^*|,S_j^0=S_j^*\setminus S.$$ 
Let $\mathcal{W}$ be the column span of $\B{X}_{\CS\cap\CS_j^*}$. Moreover, let $\mathcal{U},\mathcal{V}$ be orthogonal complement of $\mathcal{W}$ as subspaces of column spans of $\B{X}_{\CS}$ and $\B{X}_{\CS_j^*}$, respectively. Let $\B{P}_{\mathcal{U}},\B{P}_{\mathcal{V}},\B{P}_{\mathcal{W}}$ be projection matrices onto $\mathcal{U},\mathcal{V},\mathcal{W}$, respectively. With this notation in place, one has
\begin{align*}
	\B{\varepsilon}_{j}^\top (\B{P}_{\B{X}_{\CS}}-\B{P}_{\B{X}_{\CS_j^*}}) \B{\varepsilon}_{j}  =  \B{\varepsilon}_{j}^\top ( \B{P}_{\mathcal{U}}-\B{P}_{\mathcal{V}} ) \B{\varepsilon}_{j}.
\end{align*}
Note that $\text{dim}(\mathcal{U})=\bar{t},\text{dim}(\mathcal{V})= t$. As a result, by calculations similar to one in Lemma~15 of~\citet{behdin2021integer}, we obtain
\begin{multline*}
	\p \bigg(    \B{\varepsilon}_{j}^\top\B{P}_{\mathcal{U}}\B{\varepsilon}_{j} \leq \bar{t}({\sigma_j^*})^2 + ({\sigma_j^*})^2x,~~~\B{\varepsilon}_{j}^\top\B{P}_{\mathcal{V}}\B{\varepsilon}_{j} \geq - ({\sigma_j^*})^2x \bigg)  \\ \geq 1-2\exp(-c \min(x,x^2/t))-2\exp(-c \min(x,x^2/\bar{t})).~~~~~~~~~~~~
\end{multline*}
Without loss of generality, we assume $\bar{t}\geq 1$ (otherwise, the lemma is trivial). Taking 
$$x=\xi(t+\bar{t})\log p$$
for some sufficiently large universal constant $\xi$, we obtain
\begin{align*}
	\p \left(    \B{\varepsilon}_{j}^\top\B{P}_{\mathcal{U}}\B{\varepsilon}_{j}-\B{\varepsilon}_{j}^\top\B{P}_{\mathcal{V}}\B{\varepsilon}_{j} \lesssim (\sigma_j^*)^2(t+\bar{t})\log p \right) \geq 1-4\exp(-10(t+\bar{t})\log p).
\end{align*}
The proof is completed by union bound similar to Lemma~\ref{suptechlem2}.
\end{proof}

\begin{lemma}\cite[Lemma 14]{behdin2021integer}\label{suptechlem3.5}
	Under the assumptions of Theorem~\ref{supthm}, One has
	\begin{equation}
		\p\left(\bigcap_{j\in[p]}\bigcap_{\substack{ S\subseteq[p]\setminus \{j\}\\ |S|\leq k }}\mathcal{E}_5(j,S)\right)\geq 1-2kp^{-7}
	\end{equation}
	where $\mathcal{E}_5$ is defined in~\eqref{bigblockevents}.
\end{lemma}

\begin{lemma}\label{e5bound}
	Under the assumptions of Theorem~\ref{supthm}, 
	\begin{equation}
		\p\left(\bigcap_{j\in [p]}\mathcal{E}_6(j)\right)\geq 1-12kp^{-7}
	\end{equation}
	where $\mathcal{E}_6$ is defined in~\eqref{bigblockevents}.
\end{lemma}
\begin{proof}[\textbf{Proof of Lemma~\ref{e5bound}}.]
In this proof, we assume without loss of generality that $|\hat{S}_j|\leq n$ as otherwise, it is possible to remove some redundant indices in $\hat{S}_j$ without increasing $\mathcal{L}_{\hat{S}_j}(\B{x}_j)$, as this quantity is zero in both cases. The proof of this lemma is on events $\mathcal{E}_1,\mathcal{E}_2,\mathcal{E}_3$ and $\mathcal{E}_4$ over all values of $j,S$, as in Lemmas~\ref{suptechlem1},~\ref{suptechlem2} and~\ref{suptechlem3}. The intersection of these events happen with probability at least
$$1-12kp^{-7}.$$
Recalling the definition of $\CR_{S}(\cdot)$ in~\eqref{optimalobj}, one has (see calculations leading to~(89) of~\citet{behdin2021integer} and~(6.1) of~\citet{fan2020best}),
\begin{equation}\label{objexpanded}
	\begin{aligned}
		n\CR_{\CS_j}(\B{x}_j) &=n{(\B{\beta}^*_{\CS^0_j,j})}^\top (\schur{\hat{\B{\Sigma}}}{\CS_j}{\CS_j^*}) \B{\beta}^*_{\CS^0_j,j} \\
		&+ 2\B{\varepsilon}_{j}^\top (\B{I}_n-\B{P}_{\B{X}_{\CS_j}})\B{X}_{\CS^0_j}\B{\beta}^*_{\CS^0_j,j} + \B{\varepsilon}_{j}^\top(\B{I}_n-\B{P}_{\B{X}_{\CS_j}})\B{\varepsilon}_{j}
	\end{aligned}
\end{equation}
where $S^0_j=S_j^*\setminus S_j$. As a result, one has
\begin{align}
	&~  n\left[\CR_{\hat{\CS}_j}(\B{x}_j) - \CR_{\CS^*_j}(\B{x}_j) \right] \nonumber\\ \stackrel{(a)}{=} &~  n{(\B{\beta}^*_{\CS^0_j,j})}^\top (\schur{\hat{\B{\Sigma}}}{\CS_j}{\CS_j^*}) \B{\beta}^*_{\CS^0_j,j} + 2\B{\varepsilon}_j^\top (\B{I}_n-\B{P}_{\B{X}_{\CS_j}})\B{X}_{\CS^0_j}\B{\beta}^*_{\CS^0_j,j} + \B{\varepsilon}_j^\top(\B{P}_{\B{X}_{\CS^*_j}}-\B{P}_{\B{X}_{\CS_j}})\B{\varepsilon}_j\nonumber \\
	\stackrel{(b)}{\geq} &~  n{(\B{\beta}^*_{\CS^0_j,j})}^\top (\schur{\hat{\B{\Sigma}}}{\CS_j}{\CS_j^*}) \B{\beta}^*_{\CS^0_j,j} -2c_{t_1}(\sigma_j^*) \sqrt{n{(\B{\beta}^*_{\CS^0_j,j})}^{\top} (\schur{\hat{\B{\Sigma}}}{\CS}{\CS_j^*}) \B{\beta}^*_{\CS^0_j,j}}\sqrt{{(t_j+\bar{t}_j)\log p}} \nonumber \\
	& -c_{t_2}(\sigma_j^*)^2(t_j+\bar{t}_j)\log p\nonumber \\
	\stackrel{(c)}{\geq} &~  \frac{3}{4}n{(\B{\beta}^*_{\CS^0_j,j})}^\top (\schur{\hat{\B{\Sigma}}}{\CS_j}{\CS_j^*}) \B{\beta}^*_{\CS^0_j,j} -4c^2_{t_1}(\sigma_j^*)^2 {{(t_j+\bar{t}_j)\log p}} -c_{t_2}(\sigma_j^*)^2(t_j+\bar{t}_j)\log p\label{sup-lem-helper-1}
\end{align}
where $(a)$ is due to~\eqref{objexpanded}, $(b)$ is due to events $\mathcal{E}_3,\mathcal{E}_4$ and $(c)$ is by inequality $2ab\geq -a^2/4-4b^2$. Next, let $\tilde{S}_j=S_j\cap S_j^*$. Note that $S_j^0=S_j^*\setminus \tilde{S}_j$. Write 
\begin{align}
	&{(\B{\beta}^*_{\CS^0_j,j})}^\top (\schur{\hat{\B{\Sigma}}}{\CS_j}{\CS_j^*}) \B{\beta}^*_{\CS^0_j,j} \nonumber \\
	=& {(\B{\beta}^*_{\CS^0_j,j})}^\top (\schur{\hat{\B{\Sigma}}}{\CS_j}{\CS_j^*}) \B{\beta}^*_{\CS^0_j,j} - {(\B{\beta}^*_{\CS^0_j,j})}^\top (\schur{\hat{\B{\Sigma}}}{\tilde{\CS}_j}{\CS_j^*}) \B{\beta}^*_{\CS^0_j,j}+{(\B{\beta}^*_{\CS^0_j,j})}^\top (\schur{\hat{\B{\Sigma}}}{\tilde{\CS}_j}{\CS_j^*}) \B{\beta}^*_{\CS^0_j,j} \nonumber\\
	\stackrel{(a)}{=}& {(\B{\beta}^*_{\CS^0_j,j})}^\top\left(\hat{\B{\Sigma}}_{S_0^j,\tilde{S}_j}\hat{\B{\Sigma}}_{\tilde{S}_j,\tilde{S}_j}^{-1}\hat{\B{\Sigma}}_{\tilde{S}_j,S_j^0}-\hat{\B{\Sigma}}_{S_0^j,{S}_j}\hat{\B{\Sigma}}_{{S}_j,{S}_j}^{-1}\hat{\B{\Sigma}}_{{S}_j,S_j^0}\right) {(\B{\beta}^*_{\CS^0_j,j})}+{(\B{\beta}^*_{\CS^0_j,j})}^\top (\schur{\hat{\B{\Sigma}}}{\tilde{\CS}_j}{\CS_j^*}) \B{\beta}^*_{\CS^0_j,j}\nonumber\\
	\stackrel{(b)}{=}& \frac{1}{n}{(\B{X}_{S_j^0}\B{\beta}^*_{\CS^0_j,j})}^\top\left(\B{X}_{\tilde{S}_j}(\B{X}_{\tilde{S}_j}^\top \B{X}_{\tilde{S}_j})^{-1}\B{X}_{\tilde{S}_j}^\top-\B{X}_{{S}_j}(\B{X}_{{S}_j}^\top \B{X}_{S_j})^{-1}\B{X}_{{S}_j}^\top\right) {(\B{X}_{S_j^0}\B{\beta}^*_{\CS^0_j,j})}+{(\B{\beta}^*_{\CS^0_j,j})}^\top (\schur{\hat{\B{\Sigma}}}{\tilde{\CS}_j}{\CS_j^*}) \B{\beta}^*_{\CS^0_j,j}\nonumber\\
	\stackrel{(c)}{=}&\frac{1}{n}{(\B{X}_{S_j^0}\B{\beta}^*_{\CS^0_j,j})}^\top\left(\B{P}_{\B{X}_{\tilde{S}_j}}-\B{P}_{\B{X}_{{S}_j}}\right) {(\B{X}_{S_j^0}\B{\beta}^*_{\CS^0_j,j})}+{(\B{\beta}^*_{\CS^0_j,j})}^\top (\schur{\hat{\B{\Sigma}}}{\tilde{\CS}_j}{\CS_j^*}) \B{\beta}^*_{\CS^0_j,j}\nonumber \\
	\stackrel{(d)}{\geq }&-\frac{1}{n}{(\B{X}_{S_j^0}\B{\beta}^*_{\CS^0_j,j})}^\top\B{P}_{\B{X}_{{S}_j}} {(\B{X}_{S_j^0}\B{\beta}^*_{\CS^0_j,j})}+{(\B{\beta}^*_{\CS^0_j,j})}^\top (\schur{\hat{\B{\Sigma}}}{\tilde{\CS}_j}{\CS_j^*}) \B{\beta}^*_{\CS^0_j,j} \nonumber\\
	\stackrel{(e)}{\geq }&-\frac{1}{n}\left\Vert\B{X}_{S_j^0}\B{\beta}^*_{\CS^0_j,j}\right\Vert_2^2+{(\B{\beta}^*_{\CS^0_j,j})}^\top (\schur{\hat{\B{\Sigma}}}{\tilde{\CS}_j}{\CS_j^*}) \B{\beta}^*_{\CS^0_j,j}  \nonumber\\
	\stackrel{(f)}{\geq}& 0.2\eta\frac{|\tilde{S}_j^0|\log p}{n}-4\frac{|S_j^0|}{k}\label{sup-lem-helper-1.5}
\end{align}
where $(a)$ is achieved by substituting the Schur complement definition~\eqref{schur-defined}, $(b)$ is achieved by substituting $\hat{\B{\Sigma}}=\B X^\top\B X/n$, $(c)$ is by definition of projection matrices, $(d)$ is true as a projection matrix is positive semidefinite, $(e)$ is true as the largest eigenvalue of a projection matrix is bounded above by 1, and $(f)$ is due to events $\mathcal{E}_1,\mathcal{E}_2$ as $|\tilde{S}_j|\leq k$. Substituting~\eqref{sup-lem-helper-1.5} into the right hand side of inequality~\eqref{sup-lem-helper-1} completes the proof.

\end{proof}

\begin{lemma}\label{e6bounded}
	Under the assumptions of Theorem~\ref{supthm},
	\begin{equation}
		\p\left(\bigcap_{j\in[p]} \mathcal{E}_7(j)\right)\geq 1-2kp^{-7}-p(k/p)^{10}
	\end{equation}
	where $\mathcal{E}_7$ is defined in~\eqref{bigblockevents}.
\end{lemma}
\begin{proof}[\textbf{Proof of Lemma~\ref{e6bounded}}.]
The proof of this lemma is on the event considered in~\eqref{ejunionbounded} and the intersection of events $\mathcal{E}_{5}(j,S)$ for all $j,S$ as in Lemma~\ref{suptechlem3.5}. Note that by union bound, this happens with probability greater than 
$$1-2kp^{-7}-p(k/p)^{10}.$$
Based on the event considered in~\eqref{ejunionbounded} (and arguments leading to~\eqref{yjsigma}), 
\begin{equation}\label{epsilonbounded}
	5(\sigma_j^*)^2n/6\leq \|\B{\varepsilon}_j\|_2^2\leq 7(\sigma_j^*)^2n/6.  
\end{equation}
In addition, from~\eqref{objexpanded},
\begin{equation}
	n\CR_{S_j^*}(\B{x}_j)=\B{\varepsilon}_{j}^\top(\B{I}_n-\B{P}_{\B{X}_{\CS_j^*}})\B{\varepsilon}_{j}=\|\B{\varepsilon}_j\|_2^2 - \B{\varepsilon}_j^{\top} \B{P}_{\B{X}_{\CS_j^*}} \B{\varepsilon}_j.
\end{equation}
As a result, from~\eqref{epsilonbounded} we have
\begin{equation}\label{lsjhelper}\frac{5(\sigma_j^*)^2}{6}- \frac{1}{n}\B{\varepsilon}_j^{\top} \B{P}_{\B{X}_{\CS_j^*}} \B{\varepsilon}_j\leq \CR_{S_j^*}(\B{x}_j) \leq \frac{7(\sigma_j^*)^2}{6}-\frac{1}{n} \B{\varepsilon}_j^{\top} \B{P}_{\B{X}_{\CS_j^*}} \B{\varepsilon}_j.\end{equation}
Moreover, by taking $n=c_n k\log p$ to be sufficiently large and by event $\mathcal{E}_4$, 
$$-(\sigma_j^*)^2/6\leq \frac{1}{n} \B{\varepsilon}_j^{\top} \B{P}_{\B{X}_{\CS_j^*}} \B{\varepsilon}_j\leq (\sigma_j^*)^2/6$$
which together with~\eqref{lsjhelper} completes the proof.
\end{proof}

\begin{lemma}\label{lem14}
	Under the assumptions of Theorem~\ref{supthm}, one has
	\begin{equation}\label{eq-lem14}
		\p\left(\bigcap_{j\in\mathcal{J}}\left\{\frac{-99}{100}\leq \frac{\CR_{\hat{S}_j}(\B{x}_j)-\CR_{S_j^*}(\B{x}_j)}{\CR_{S_j^*}(\B{x}_j)}\leq 100\right\}\right)\geq 1-2p(k/p)^{10}-13kp^{-7}.
	\end{equation}
\end{lemma}
\begin{proof}[\textbf{Proof of Lemma~\ref{lem14}}.]
Note that $\B{x}_j\sim\CN(\B{0},(\B{\Sigma}^*)_{jj}\B{I}_n)$. Let events $A_j$ for $j\in[p]$ be defined as
\begin{equation}
	A_j=\left\{\frac{1}{n}\|\B{x}_j\|_2^2\leq \frac{7}{6}(\B{\Sigma}^*)_{jj}\right\}.
\end{equation}
By Lemma~\ref{bernlem} and an argument similar to the one leading to~\eqref{ejunionbounded}, by taking $n\gtrsim \log p$, we have
$$\p(A_j)\geq 1-(k/p)^{10}$$
which leads to 
$$\p\left(\bigcap_{j\in[p]} A_j\right)\geq 1-p(k/p)^{10}$$
by union bound. The proof of this lemma is on events $\bigcap_{j\in[p]} A_j$, and $\mathcal{E}_7$ over all choices of $j$, as considered in Lemma~\ref{e6bounded}. By union bound, the intersection of these events occur with probability at least 
$$1-2p(k/p)^{10}-2kp^{-7}.$$
First, for $j\in\mathcal{J}$ we have
\begin{align}
	\frac{\CR_{\hat{S}_j}(\B{x}_j)-\CR_{S_j^*}(\B{x}_j)}{\CR_{S_j^*}(\B{x}_j)} & \stackrel{(a)}{\geq} \frac{\ell}{\CR_{S_j^*}(\B{x}_j)}-1\nonumber\\
	& \stackrel{(b)}{\geq} \frac{3\ell}{4(\sigma_j^*)^2}-1 \nonumber\\
	& \stackrel{(c)}{\geq} \frac{l_{\sigma}^2}{4u_{\sigma}^2}-1 \geq-\frac{99}{100}
\end{align}
where $(a)$ is true as for $j\in\mathcal{J}$, $\CR_{\hat{S}_j}(\B{x}_j)\geq \ell$, $(b)$ is due to event $\mathcal{E}_7$, $(c)$ and the last inequality are due to Assumption~\ref{sigmaboundedassum}. The proof of lower bound in~\eqref{eq-lem14} is completed.

Next, note that
\begin{align}
	\frac{\CR_{\hat{S}_j}(\B{x}_j)-\CR_{S_j^*}(\B{x}_j)}{\CR_{S_j^*}(\B{x}_j)} & \leq \frac{\CR_{\hat{S}_j}(\B{x}_j)}{\CR_{S_j^*}(\B{x}_j)} \nonumber \\
	& \stackrel{(a)}{\leq} \frac{\frac{1}{n}\|\B{x}_j\|_2^2}{\CR_{S_j^*}(\B{x}_j)} \nonumber \\
	& \stackrel{(b)}{\leq} \frac{3}{2}\frac{\frac{1}{n}\|\B{x}_j\|_2^2}{(\sigma_j^*)^2} \nonumber \\
	& \stackrel{(c)}{\leq} \frac{7}{4} \frac{(\B{\Sigma}^*)_{jj}}{(\sigma_j^*)^2}\leq 100
\end{align}
where $(a)$ is due to definition of $\CR_{\hat{S}_j}$, $(b)$ is by event $\mathcal{E}_7$, $(c)$ is a result of event $A_j$ and the last inequality is a result of Assumption~\ref{assum2-max}.
\end{proof}

\begin{lemma}\label{ejclem}
	Let $h_j$ be defined as in~\eqref{hdefine}. Let the event $\mathcal{E}_{\mathcal{J}^c}$ be defined as
	\begin{multline}
		\mathcal{E}_{\mathcal{J}^c}= \bigg\{\sum_{j\in\mathcal{J}^c} \left[h_j(\hat{\sigma}_j,\hat{S}_j)+\lambda|\hat{S}_j|-h_j(\tilde{\sigma}_j,{S}^*)-\lambda|S_j^*|\right]\geq \\
		\frac{c_1}{l_{\sigma}^2}\eta\frac{\log p}{n} \sum_{j\in\mathcal{J}^c}\tilde{t}_j  + \left(c_{\lambda}-c_2\right)\frac{\log p}{n}\sum_{j\in\mathcal{J}^c}\bar{t}_j + \left(-c_{\lambda}-c_3-\frac{c_nc_4}{l_{\sigma}^2}\right)\frac{\log p}{n} \sum_{j\in\mathcal{J}^c} t_j
		\bigg\}
	\end{multline}
	for some universal constants $c_1,\cdots,c_4>0$. Then, under the assumptions of Theorem~\ref{supthm} 
	$$\p(\mathcal{E}_{\mathcal{J}^c})\geq 1-14kp^{-7}-p(k/p)^{10}.$$
\end{lemma}
\begin{proof}[\textbf{Proof of Lemma~\ref{ejclem}}.]
The proof of this lemma in on the intersection of events $\mathcal{E}_6$ and $\mathcal{E}_7$ as in Lemmas~\ref{e5bound} and~\ref{e6bounded}. Note that this happens with probability at least
$$1-14kp^{-7}-p(k/p)^{10}.$$
One has
\begin{align}
	& h_j(\hat{\sigma}_j,\tilde{S}_j) - h_{j}(\sigma_j^*,S_j^*) \nonumber \\
	& = \log(\hat{\sigma}_j)+\frac{\mathcal{L}_{\hat{S}_j}(\B{x}_j)}{2\hat{\sigma}_j^2}-\log(\tilde{\sigma}_j)-\frac{\mathcal{L}_{S_j^*}(\B{x}_j)}{2\tilde{\sigma}_j^2} \nonumber\\
	& = \left[\log(\hat{\sigma}_j)-\log(\tilde{\sigma}_j) + \mathcal{L}_{S_j^*}(\B{x}_j)(\frac{1}{2\hat{\sigma}_j}-\frac{1}{2\tilde{\sigma}_j})\right] + \frac{\mathcal{L}_{\hat{S}_j}(\B{x}_j)-\mathcal{L}_{S_j^*}(\B{x}_j)}{2\hat{\sigma}_j}\nonumber \\
	& \stackrel{(a)}{\geq} \frac{\mathcal{L}_{\hat{S}_j}(\B{x}_j)-\mathcal{L}_{S_j^*}(\B{x}_j)}{2\hat{\sigma}_j}\nonumber \\
	& \stackrel{(b)}{=} \frac{\mathcal{L}_{\hat{S}_j}(\B{x}_j)-\mathcal{L}_{S_j^*}(\B{x}_j)}{2\ell}\nonumber \\
	& \stackrel{(c)}{\geq} \frac{    \frac{3}{20}\eta\tilde{t}_j\frac{\log p}{n} -(\sigma_j^*)^2\frac{(t_j+\bar{t}_j)\log p}{n}(4c_{t_1}^2+c_{t_2})-\frac{4t_j}{k}}{2\ell} \nonumber\\
	& \stackrel{(d)}{\geq} \frac{9}{40l_{\sigma}^2}\eta\tilde{t}_j\frac{\log p}{n} -\frac{6t_j}{l_{\sigma}^2 k} - \frac{3u_{\sigma}^2}{2l_{\sigma}^2}(4c_{t_1}^2+c_{t_2})\frac{\log p}{n}t_j - \frac{3u_{\sigma}^2}{2l_{\sigma}^2}(4c_{t_1}^2+c_{t_2})\frac{\log p}{n}\bar{t}_j \nonumber\\
	& \stackrel{(e)}{\geq} \frac{9}{40l_{\sigma}^2}\eta\tilde{t}_j\frac{\log p}{n} -\frac{6c_n}{l_{\sigma}^2 }\frac{\log p}{n}t_j - \frac{75}{2}(4c_{t_1}^2+c_{t_2})\frac{\log p}{n}t_j - \frac{75}{2}(4c_{t_1}^2+c_{t_2})\frac{\log p}{n}\bar{t}_j\label{ejchelper}
\end{align}
where $(a)$ is true as on event $\mathcal{E}_7$, we have $\mathcal{L}_{S_j^*}(\B{x}_j)\geq 2(\sigma_j^*)^2/3>\ell$ so $\tilde{\sigma}_j=\sqrt{\mathcal{L}_{S_j^*}(\B{x}_j)}$ and $h_j(\tilde{\sigma}_j,S_j^*)\leq h_j(\hat{\sigma}_j,S_j^*)$, $(b)$ is true as $\hat{\sigma}_j\geq \ell$, $(c)$ is by event $\mathcal{E}_6$, $(d)$ is by substituting $\ell=l_{\sigma}^2/3$ and $\sigma_j^*\leq u_\sigma$, and $(e)$ is by Assumption~\ref{sigmaboundedassum}, $u_\sigma/l_\sigma\leq 5$ and also $n=c_nk\log p$. By summing~\eqref{ejchelper} over $j\in\mathcal{J}^c$, we achieve  
\begin{multline*}
	\sum_{j\in\mathcal{J}^c} \left[h_j(\hat{\sigma}_j,\hat{S}_j)+\lambda|\hat{S}_j|-h_j(\tilde{\sigma}_j,{S}^*)-\lambda|S_j^*|\right]\geq \\ \frac{c_1}{l_{\sigma}^2}\eta\frac{\log p}{n} \sum_{j\in\mathcal{J}^c}\tilde{t}_j  + \left(c_{\lambda}-c_2\right)\frac{\log p}{n}\sum_{j\in\mathcal{J}^c}\bar{t}_j + \left(-c_{\lambda}-c_3-\frac{c_nc_4}{l_{\sigma}^2}\right)\frac{\log p}{n} \sum_{j\in\mathcal{J}^c} t_j
\end{multline*}
$c_1=9/40$, $c_2=c_3=75(4c_{t_1}^2+c_{t_2})/2$ and $c_4=6$.
\end{proof}
\begin{lemma}\label{loglem}
	Let $a>0$. Then, 
	$$\log(1+x)\geq \frac{x}{1+a}$$
	for $x\in[0,a]$. Similarly, if $a\in(-1,0)$, 
	$$\log(1+x)\geq \frac{x}{1+a}$$
	for $x\in[a,0]$.
\end{lemma}
\begin{proof}[\textbf{Proof of Lemma~\ref{loglem}}.]
Suppose $a>0$ and $x\in[0,a]$. Note that 
$$\log(1+x) = \int_{0}^x \frac{dt}{1+t}\geq \int_0^x\frac{dt}{1+a}=\frac{x}{1+a}.$$
The proof of other part is similar.
\end{proof}

\begin{lemma}\label{ejlem}
	Let the event $\mathcal{E}_\mathcal{J}$ be defined as
	\begin{multline}
		\mathcal{E}_{\mathcal{J}}= \bigg\{\sum_{j\in\mathcal{J}} \left[h_j(\hat{\sigma}_j,\hat{S}_j)+\lambda|\hat{S}_j|-h_j(\tilde{\sigma}_j,{S}^*)-\lambda|S_j^*|\right]\geq \\
		\frac{c_5\eta}{u_{\sigma}^2}\frac{\log p}{n}\sum_{j\in\mathcal{J}} \tilde{t}_j + (-c_6 -\frac{c_7c_n}{l_{\sigma}^2}-c_\lambda)\frac{\log p}{n}\sum_{j\in\mathcal{J}} {t}_j + (c_{\lambda}-c_8)\frac{\log p}{n}\sum_{j\in\mathcal{J}} \bar{t}_j
		\bigg\}
	\end{multline}
	for some universal constants $c_5,\cdots,c_8>0$. Then, under the assumptions of Theorem~\ref{supthm} 
	$$\p(\mathcal{E}_{\mathcal{J}})\geq 1-27kp^{-7}-3p(k/p)^{10}.$$
\end{lemma}
\begin{proof}[\textbf{Proof of Lemma~\ref{ejlem}}.]
The proof of this lemma in on the intersection of events $\mathcal{E}_6$, $\mathcal{E}_7$ as in Lemmas~\ref{e5bound} and~\ref{e6bounded} and the event in Lemma~\ref{lem14}. Note that this happens with probability at least
$$1-27kp^{-7}-3p(k/p)^{10}.$$
Let $\mathcal{J}_+,\mathcal{J}_-\subseteq[p]$ be defined as
\begin{equation}
	\begin{aligned}
		\mathcal{J}_+&=\{j\in\mathcal{J}: \CR_{\hat{S}_j}(\B{x}_j)-\CR_{S_j^*}(\B{x}_j)\geq 0\} \\
		\mathcal{J}_-&=\{j\in\mathcal{J}: \CR_{\hat{S}_j}(\B{x}_j)-\CR_{S_j^*}(\B{x}_j)< 0\}.
	\end{aligned}
\end{equation}
Based on Lemma~\ref{lem14}, for $j\in\mathcal{J}_+$, we have 
$$0\leq\frac{\CR_{\hat{S}_j}(\B{x}_j)-\CR_{S_j^*}(\B{x}_j)}{\CR_{S_j^*}(\B{x}_j)}\leq 100.$$
Consequently, by Lemma~\ref{loglem}, for $j\in\mathcal{J}_+$ we have
\begin{equation}\label{c1ineq}
	\log\left(1+\frac{\CR_{\hat{S}_j}(\B{x}_j)-\CR_{S_j^*}(\B{x}_j)}{\CR_{S_j^*}(\B{x}_j)}\right)\geq c^{(1)}\frac{\CR_{\hat{S}_j}(\B{x}_j)-\CR_{S_j^*}(\B{x}_j)}{\CR_{S_j^*}(\B{x}_j)} 
\end{equation}
where $c^{(1)}=1/101$. Similarly, for $j\in\mathcal{J}_-$ we have
\begin{equation}\label{c2ineq}
	\log\left(1+\frac{\CR_{\hat{S}_j}(\B{x}_j)-\CR_{S_j^*}(\B{x}_j)}{\CR_{S_j^*}(\B{x}_j)}\right)\geq c^{(2)}\frac{\CR_{\hat{S}_j}(\B{x}_j)-\CR_{S_j^*}(\B{x}_j)}{\CR_{S_j^*}(\B{x}_j)} 
\end{equation}
where $c^{(2)}=100$. By discussion leading to~\eqref{h-twopart}, we have that for $j\in\mathcal{J}$,
$$h_j(\hat{\sigma}_j,\hat{S}_j)=\frac{\log (\CR_{\hat{S}_j}(\B{x}_j))}{2} + \frac{1}{2}, h_j(\tilde{\sigma}_j,{S}^*_j)=\frac{\log(\CR_{{S}^*_j}(\B{x}_j))}{2} + \frac{1}{2}.$$
Therefore, one has
\begin{align}
	&\sum_{j\in\mathcal{J}}\left\{h_j(\hat{\sigma}_j,\hat{S}_j)-h_j(\tilde{\sigma}_j,{S}^*_j)+\lambda |\hat{S}_j|-\lambda |S_j^*|\right\} \nonumber\\
	=& \sum_{j\in\mathcal{J}}  \left\{\frac{1}{2}\log(\CR_{\hat{S}_j}(\B{x}_j))-\frac{1}{2}\log(\CR_{{S}^*_j}(\B{x}_j))+\lambda |\hat{S}_j|-\lambda |S_j^*|\right\}\nonumber \\
	=&   \sum_{j\in\mathcal{J}} \frac{1}{2}\log\left(1+\frac{\CR_{\hat{S}_j}(\B{x}_j)-\CR_{{S}^*_j}(\B{x}_j)}{\CR_{{S}^*_j}(\B{x}_j)}\right) + \lambda \sum_{j\in\mathcal{J}}(\bar{t}_j-t_j) \nonumber\\
	\stackrel{(a)}{=}& \sum_{j\in\mathcal{J}_+} \frac{1}{2}\log\left(1+\frac{\CR_{\hat{S}_j}(\B{x}_j)-\CR_{{S}^*_j}(\B{x}_j)}{\CR_{{S}^*_j}(\B{x}_j)}\right)+\sum_{j\in\mathcal{J}_-} \frac{1}{2}\log\left(1+\frac{\CR_{\hat{S}_j}(\B{x}_j)-\CR_{{S}^*_j}(\B{x}_j)}{\CR_{{S}^*_j}(\B{x}_j)}\right)+ \lambda \sum_{j\in\mathcal{J}}(\bar{t}_j-t_j)\nonumber \\
	\stackrel{(b)}{\geq} &  c^{(1)}\sum_{j\in\mathcal{J}_+} \frac{\CR_{\hat{S}_j}(\B{x}_j)-\CR_{{S}^*_j}(\B{x}_j)}{\CR_{{S}^*_j}(\B{x}_j)}+c^{(2)}\sum_{j\in\mathcal{J}_-} \frac{\CR_{\hat{S}_j}(\B{x}_j)-\CR_{{S}^*_j}(\B{x}_j)}{\CR_{{S}^*_j}(\B{x}_j)} + \lambda \sum_{j\in\mathcal{J}}(\bar{t}_j-t_j)\nonumber \\
	\stackrel{(c)}{\geq }& \frac{3c^{(1)}}{4}\sum_{j\in\mathcal{J}_+} \frac{\CR_{\hat{S}_j}(\B{x}_j)-\CR_{{S}^*_j}(\B{x}_j)}{(\sigma_j^*)^2}+\frac{3c^{(2)}}{2}\sum_{j\in\mathcal{J}_-} \frac{\CR_{\hat{S}_j}(\B{x}_j)-\CR_{{S}^*_j}(\B{x}_j)}{(\sigma_j^*)^2}+ \lambda \sum_{j\in\mathcal{J}}(\bar{t}_j-t_j) \label{proofthm3ineq1}
\end{align}
where $(a)$ is by the fact that $\mathcal{J}_+,\mathcal{J}_-$ is a partition of $\mathcal{J}$, $(b)$ is due to~\eqref{c1ineq} and~\eqref{c2ineq}, and $(c)$ is due to event $\mathcal{E}_7$. From~\eqref{proofthm3ineq1} and event $\mathcal{E}_6$,
\begin{align}
	&\sum_{j\in\mathcal{J}}\left\{h_j(\hat{\sigma}_j,\hat{S}_j)-h_j(\tilde{\sigma}_j,{S}^*_j)+\lambda |\hat{S}_j|-\lambda |S_j^*|\right\} \nonumber\\
	\geq & \frac{3c^{(1)}}{4}\sum_{j\in\mathcal{J}_+} \frac{\frac{3}{20}\eta \frac{\tilde{t}_j\log p}{n}-(\sigma_j^*)^2\frac{(t_j+\bar{t}_j)\log p}{n}(4c_{t_1}^2+c_{t_2})-\frac{4t_j}{k}}{(\sigma_j^*)^2}\nonumber\\&+\frac{3c^{(2)}}{2}\sum_{j\in\mathcal{J}_-} \frac{\frac{3}{20}\eta \frac{\tilde{t}_j\log p}{n}-(\sigma_j^*)^2\frac{(t_j+\bar{t}_j)\log p}{n}(4c_{t_1}^2+c_{t_2})-\frac{4t_j}{k}}{(\sigma_j^*)^2} + c_{\lambda}\frac{\log p}{n} \sum_{j\in\mathcal{J}}(\bar{t}_j-t_j)\nonumber \\
	\stackrel{(a)}{\geq} & \sum_{j\in\mathcal{J}}\left[\frac{9c^{(1)}\eta}{80}\frac{\log p }{n}\tilde{t}_j-\frac{3c^{(2)}(4c_{t_1}^2+c_{t_2})}{2}\frac{\log p }{n}t_j-\frac{{6c^{(2)}}}{l_{\sigma}^2}c_n\frac{\log p }{n}t_j-c_\lambda t_j-\frac{3c^{(2)}(4c_{t_1}^2+c_{t_2})}{2}\frac{\log p }{n}\bar{t}_j+c_\lambda\bar{t}_j\right] \nonumber\\
	\stackrel{}{\geq}& \frac{c_5\eta}{u_{\sigma}^2}\frac{\log p}{n}\sum_{j\in\mathcal{J}} \tilde{t}_j + (-c_6 -\frac{c_7c_n}{l_{\sigma}^2}-c_\lambda)\frac{\log p}{n}\sum_{j\in\mathcal{J}} {t}_j + (c_{\lambda}-c_8)\frac{\log p}{n}\sum_{j\in\mathcal{J}} \bar{t}_j \label{thm3proofineq2}
\end{align} 
where $(a)$ is due to the fact $c^{(2)}>c^{(1)}$, and $c_5=9c^{(1)}/80$, $c_6=c_8=3c^{(2)}(4c_{t_1}^2+c_{t_2})/2$ and $c_7=6c^{(2)}$.
\end{proof}

\subsubsection{Proof of Theorem \ref{supthm}}
The proof of this theorem is on the intersection of events $\mathcal{E}_\mathcal{J}$ and $\mathcal{E}_{\mathcal{J}^c}$ as in Lemmas~\ref{ejlem} and~\ref{ejclem}. Note that this happens with probability at least 
\begin{equation}\label{supthmprob}
	1-4p(k/p)^{10}-41kp^{-7}.
\end{equation}
By optimality of $\hat{z}$ and feasibility of $z^*$ for~\eqref{main-sup}, we have
\begin{align}
	0 \geq &\sum_{j=1}^p \left\{h_j(\hat{\sigma}_j,\hat{S}_j)-h_j(\tilde{\sigma}_j,{S}^*_j)+\lambda |\hat{S}_j|-\lambda |S_j^*|\right\} \nonumber\\
	= &  \sum_{j\in\mathcal{J}} \left\{h_j(\hat{\sigma}_j,\hat{S}_j)-h_j(\tilde{\sigma}_j,{S}^*_j)+\lambda |\hat{S}_j|-\lambda |S_j^*|\right\}+\sum_{j\in\mathcal{J}^c} \left\{h_j(\hat{\sigma}_j,\hat{S}_j)-h_j(\tilde{\sigma}_j,{S}^*_j)+\lambda |\hat{S}_j|-\lambda |S_j^*|\right\}\nonumber \\
	\stackrel{(a)}{\geq} &  \frac{c_5\eta}{u_{\sigma}^2}\frac{\log p}{n}\sum_{j\in\mathcal{J}} \tilde{t}_j + (-c_6 -\frac{c_7c_n}{l_{\sigma}^2}-c_\lambda)\frac{\log p}{n}\sum_{j\in\mathcal{J}} {t}_j + (c_{\lambda}-c_8)\frac{\log p}{n}\sum_{j\in\mathcal{J}} \bar{t}_j\nonumber \\
	& +\frac{c_1}{l_{\sigma}^2}\eta\frac{\log p}{n} \sum_{j\in\mathcal{J}^c}\tilde{t}_j  + \left(c_{\lambda}-c_2\right)\frac{\log p}{n}\sum_{j\in\mathcal{J}^c}\bar{t}_j + \left(-c_{\lambda}-c_3-\frac{c_nc_4}{l_{\sigma}^2}\right)\frac{\log p}{n} \sum_{j\in\mathcal{J}^c} t_j \nonumber \\
	\geq &  \frac{c_{f_1}\eta}{u_{\sigma}^2}\frac{\log p}{n}\sum_{j=1}^p \tilde{t}_j + \left(-c_{\lambda}-c_{f_2}-\frac{c_nc_{f_3}}{l_{\sigma}^2}\right)\frac{\log p}{n} \sum_{j=1}^p t_j +  (c_{\lambda}-c_{f_4})\frac{\log p}{n}\sum_{j=1}^p \bar{t}_j \nonumber \\
	\stackrel{(b)}{=}  &\left[\frac{c_{f_1}\eta}{u_{\sigma}^2}-2c_{\lambda}-2c_{f_2}-\frac{2c_nc_{f_3}}{l_{\sigma}^2}\right]\frac{\log p}{n}\sum_{j=1}^p \tilde{t}_j  +  (c_{\lambda}-c_{f_4})\frac{\log p}{n}\sum_{j=1}^p \bar{t}_j\label{supthmineqfinal}
\end{align}
where 
$$\text{$c_{f_1}=c_1\land c_5$, $c_{f_2}=c_3\lor c_6$, $c_{f_3}=c_4\lor c_7$ and $c_{f_4}=c_2\lor c_8$},$$ $(a)$ is due to  Lemmas~\ref{ejlem} and~\ref{ejclem} and $(b)$ is true as if $\hat{z}_{ij}\neq z^*_{ij}$, then $\hat{z}_{ji}\neq z^*_{ji}$ so $\sum_{j=1}^p t_j=2\sum_{j=1}^p \tilde{t}_j$. Take $c_\lambda>c_{f_4}$ and $\eta\gtrsim (2c_{\lambda}+2c_{f_2}+\frac{2c_nc_{f_3}}{l_{\sigma}^2})u_{\sigma}^2$.

Therefore, from~\eqref{supthmineqfinal} we have
$$0\geq \underbrace{\left[\frac{c_{f_1}\eta}{u_{\sigma}^2}-2c_{\lambda}-2c_{f_2}-\frac{2c_nc_{f_3}}{l_{\sigma}^2}\right]\frac{\log p}{n}}_{>0}\sum_{j=1}^p \tilde{t}_j  +  \underbrace{(c_{\lambda}-c_{f_4})\frac{\log p}{n}}_{>0}\sum_{j=1}^p \bar{t}_j$$
which implies $\sum_{j=1}^p \tilde{t}_j=\sum_{j=1}^p \bar{t}_j=0$ or equivalently $\hat{z}_{ij}=z^*_{ij}$.

\section{Additional Experiments from Section~\ref{sec:expts}}\label{app:numerical}

\subsection{Experimental setup details}\label{app:num-details}

The regularization coefficients for all methods are chosen from a grid ranging from $\sqrt{\log p/n}/100$ and $100\sqrt{\log p/n}$, using a validation set on the pseudo-likelihood loss:
\begin{equation}\label{appendix-eq-val-loss}
    \sum_{i=1}^p\bigl(-\log(\theta_{ii})+\frac1{\theta_{ii}}\norm{\btX_{\text{val}}\bm\theta_{i}}^2\bigr)
\end{equation}
where $\B{X}_{\text{val}}$ denotes the validation dataset. 
Next, we discuss the termination criteria used for our solvers. For any two values $c_1, c_2\in\R$, we define the relative gap between $c_1,c_2$ as
${|c_1-c_2|}/{\max(|c_1|,|c_2|)}.$
For the CD-based approximate solver, including when used as a standalone solver or when used in the BnB search to obtain upper bounds, we use the following termination criteria: We terminate our algorithm when the relative objective value gap between two successive iterations is less than $10^{-4}$.  We also use a primal/dual relative gap tolerance of $10^{-4}$ when applying the CD method to node relaxations.

As for the Big-M value, in our synthetic experiments we set $M=2\max_{i,j\in[p]}|\theta^*_{ij}|$ where $\B\Theta^*$ is the underlying precision matrix. In the real data example, we set $M=2$ as for this value of $M$, we have $M>\max_{i,j\in[p]}|\hat{\theta}_{ij}|$ where $\hat{\B\Theta}$ is the estimated precision matrix from our approximate solver (initial incumbent). Our experiments in this section show that as long as $\lambda_0,\lambda_2>0$ are sufficiently large in Problem~\eqref{eqn:mio}, choosing a larger value of $M$ does not lead to a drastic increase in the runtime. On the other hand, we observe that choosing $M$ to be too small can hurt statistical performance. Therefore, we recommend choosing a sufficiently large value of Big-M so that it does not affect the optimal solution to Problem~\eqref{eqn:mio}. Since our approximate solvers usually return high-quality solutions that are identical or close to the optimal ones, we recommend using a solution from the approximate solver to obtain an estimate of how large $M$ needs to be.

In the rest of this appendix, we present additional numerical experiments:
\begin{enumerate}
    \item Explore how upper/lower bounds of BnB procedure evolve over the depth of BnB tree.
    \item Number of nodes explored by the BnB tree. 
    \item Quality of the root relaxation.
    \item Ablation studies on the impact of Big-M parameter in formulation~\eqref{eqn:mio}. 
    \item The choice of $\ell_2$ tuning parameter and how it affects the statistical and computational performance of our estimator.
\end{enumerate}

\subsection{A deeper investigation of the BnB method and ablation studies}
In this section, we explore several performance characteristics of our BnB method.

\subsubsection{Evolution of lower and upper bounds}
Here, we study the same setup from Section~\ref{num-mosek} with $k=10$. We set $n=1000$ and study two cases with $p=50$ and $p=100$. We run the BnB solver and record the lower bound and the incumbent upper bound for each depth of the BnB search tree. We plot the results for these two cases in Figure~\ref{fig:zerogap}.

\begin{figure}[ht]
     \centering
     \begin{tabular}{cc}
     $p=50$ & $p=100$\\
     \includegraphics[width=0.4\linewidth,trim =.0cm 0cm .0cm 0cm, clip = true]{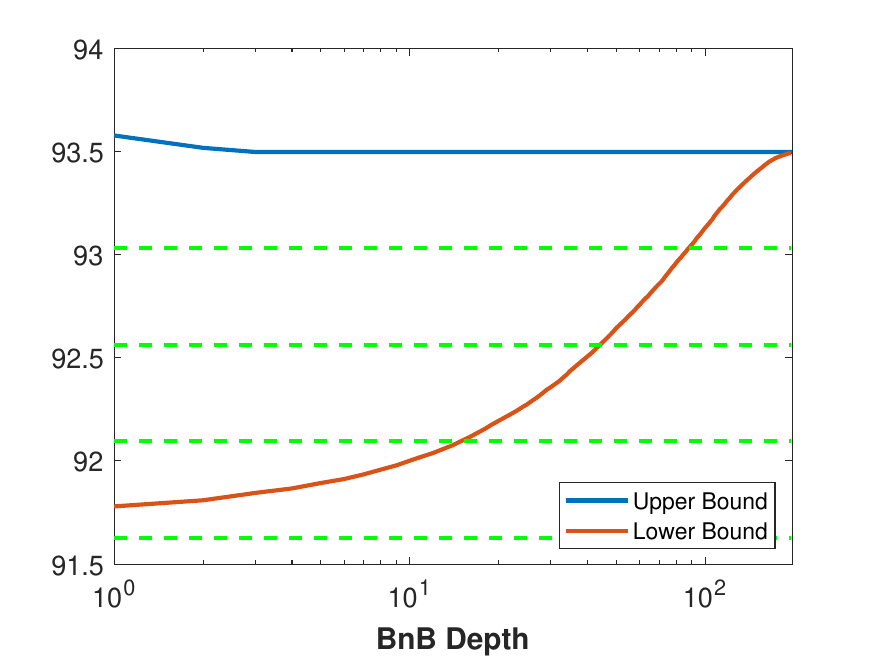}& 
     \includegraphics[width=0.4\linewidth,trim =.0cm 0cm .0cm 0cm, clip = true]{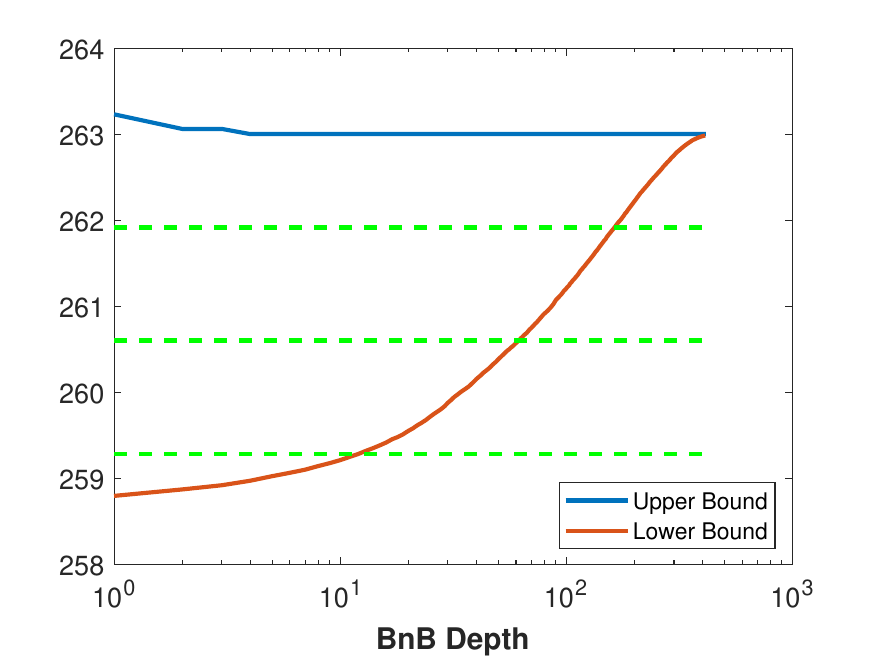}
\end{tabular}
        \caption{The progress of upper and lower bounds over the course of the BnB solver. Each green dashed line represents a $0.5\%$ gap from the initial incumbent upper bound.}
        \label{fig:zerogap}
\end{figure}

We make some observations from this figure:

\customparagraph{Upper bounds} We see that our BnB solver obtains a high-quality incumbent early on. However, even though the root incumbent is close to the optimal objective, BnB makes small improvements to the incumbent when exploring depths less than 5 or so.

\customparagraph{Lower bounds} In cases demonstrated here, at root, the lower bound has an optimality gap of $1.5-2\%$. As we explore a deeper BnB tree, the lower bound gradually increases and reaches $<0.01\%$ optimality gap for sufficiently large depths.

\begin{table}[t!]
\begin{minipage}{0.5\textwidth}
\footnotesize
\centering\footnotesize
\caption{Number of nodes explored for experiments in Section~\ref{num-mosek}.}
\label{table:nodes}
    \begin{tabular}{ ccc }
 $p$ & $n$  & Nodes\\
 \hline\hline
 \multirow{2}{*}{$100$} & 500  & 1427 \\
& 1000  &  1093 \\
 \hline
 \multirow{2}{*}{$250$} & 500   &  6561  \\
& 1000 &  5797 \\
 \hline
  \multirow{2}{*}{$500$} & 500    & 8095  \\
& 1000 &  7933 \\
 \hline
  \multirow{2}{*}{$1000$} & 500   &  15239 \\
& 1000 & 14817  \\
 \hline
   \multirow{2}{*}{$2500$} & 500    & 3835 \\
& 1000 & 3471 \\
 \hline
  \multirow{2}{*}{$5000$} & 500   & 127  \\
& 1000 &  127 \\
 \hline

\end{tabular}
\end{minipage}
\begin{minipage}{0.5\textwidth}
\footnotesize
\centering\footnotesize
\caption{Performance of root relaxation solver in Appendix~\ref{app:root}.}
\label{table:root}
    \begin{tabular}{ cc|cc }
 $p$ & $n$  & Time & Gap\\
 \hline\hline
  \multirow{2}{*}{$500$} & 500   & $1.2\pm0.1$ & $1.9\%$   \\
& 1000 &  $1.3\pm0.2$ & $1.8\%$  \\
 \hline
  \multirow{2}{*}{$1000$} & 500   & $7.9\pm 0.8$  &$2.5\%$  \\
& 1000 &   $6.1\pm 0.4$  &  $2.4\%$ \\
 \hline
   \multirow{2}{*}{$2500$} & 500   & $99.2\pm12.6$ &  3.8\%  \\
& 1000 &  $81.7\pm15.4$ &  3.6\% \\
 \hline
  \multirow{2}{*}{$5000$} & 500   & $384\pm 94$ & $5.2\%$  \\
& 1000 &  $412\pm 86$ & $4.9\%$ \\
 \hline

\end{tabular}
\end{minipage}
\end{table}

\subsubsection{Number of nodes explored}
We report the number of nodes explored by~\ourmethod~for one replication of the experiments in Section~\ref{num-mosek} in Table~\ref{table:nodes}.

\subsubsection{Performance of root relaxation solver}\label{app:root} We also study the performance of our root node solver, in terms of scalability and the quality of dual bounds. In particular, for the experiments in Section~\ref{num-mosek} and Table~\ref{table:times-exact}, we report the runtime to solve the root relaxation, as well as the root MIP gap in Table~\ref{table:root}. As we see, even for $p=1000$ we can obtain dual bounds to our problem in seconds, and for $p=5000$, we can obtain dual bounds in minutes. Even though the root relaxation is convex, it is a challenging optimization problem with around $p^2/2$-many variables. Our results hence show the effectiveness of our CD-based algorithms for solving the root and node convex relaxations. We also see that if the MIP gap at root is not too large, simply solving the root relaxation can deliver optimality certificates for our estimator quickly.

\subsection{Investigating the effect of $M$ and $\lambda_2$ on the runtime}\label{app:bigm}
Next, we study how changing the Big-M value in~\eqref{eqn:mio} impacts the runtime of \ourmethod. In particular, we use the same setup as in Section~\ref{num-mosek} with $p=250$. We take $M=a\max_{i,j\in[p]}|\theta^*_{ij}|$ and vary the value of $a>0$. (For each $M$, we choose the values of $\lambda_0,\lambda_2$ based on the validation loss, as discussed in Appendix~\ref{app:num-details}.) We report the runtime of \ourmethod~to get $1\%$ MIP gap in Table~\ref{table:times-M} for different values of $a$ and $n$.

\begin{table}[t!]
\footnotesize

\centering\footnotesize
\caption{ Effect of changing $M$ in~\eqref{eqn:mio} on the runtime of \ourmethod. We see that the runtime is not too sensitive to the value of Big-M. The details are discussed in Appendix~\ref{app:bigm}.}
\label{table:times-M}
    \begin{tabular}{ c|cccc }
$n$  &  $a=1$ & $a=2 $& $a=5$ & $a=100$\\
 \hline\hline
$500$ & $77.8\pm5.3$ & $89.4\pm6.2$ & $97.9\pm 13.1$   & $105.7\pm 5.1$  \\  
$1000$ & $81.0\pm6.2$ & $93.2\pm7.6$ & $102.5\pm 8.4$& $108.4\pm 3.9$\\
\hline

\end{tabular}
\end{table}
Although a smaller value of $M$ results in faster algorithms in Table~\ref{table:times-M}, we do not observe a significant difference in runtime for different values of $M$. Therefore, as long as $M$ is chosen sufficiently large (so that the optimal solution of~\eqref{eqn:mio} is not affected), our method does not seem too sensitive to $M$.

In our next set of experiments, we investigate the effect of the $\ell_2$ regularization term on the runtime of \ourmethod. To this end, we consider the same setup as in Section~\ref{num-mosek} with $n=1000, p=100$. We also choose our default value of Big-M (as in Appendix~\ref{app:num-details}). Next, we fix a value of $\lambda_2$, and use our approximate solver to calculate a path of solutions for different values of $\lambda_0$. We let $\lambda_0^*(\lambda_2)$ be the value of $\lambda_0$  that minimizes the validation loss (as discussed in Appendix~\ref{app:num-details}). We report the runtime of our BnB solver for different pairs of $(\lambda_0^*(\lambda_2),\lambda_2)$ in Table~\ref{table:times-lambda2}. We see that in general, including the $\ell_2$ regularization term is helpful to reduce \ourmethod's runtime.

\begin{table}[t!]
\footnotesize

\centering\footnotesize
\caption{ Effect of changing $\lambda_2$ in~\eqref{eqn:mio} on the runtime of \ourmethod. We report the runtime to reach $1\%$ MIP gap. If $1\%$ gap is not reached after one hour, we report the final MIP gap in the parenthesis (results averaged across replications).  We see that increasing $\lambda_2$ leads to reduced runtimes. The details are discussed in Appendix~\ref{app:bigm}.}
\label{table:times-lambda2}
    \begin{tabular}{ c|cccccc }
  &  $\lambda_2=0.01$ & $\lambda_2=0.05 $ & $\lambda_2=0.1$ & $\lambda_2=0.2$ & $\lambda_2=0.5$ & $\lambda_2=1$\\
 \hline\hline
Runtime & $(2.3\%)$  & $(1.2\%)$   & $1508\pm227$ &  $46.6\pm 7.3$ &  $9.6\pm0.5$ & $0.2\pm0.0$

\end{tabular}
\end{table}

\subsection{Statistical effect of $\ell_2$ regularization and Big-M}\label{app:l2reg}
We study the effect of $\ell_2$ regularization in Problem~\eqref{eqn:mio} in terms of statistical properties. In particular, we see that in our numerical experiments, when we select the hyper-parameter $\lambda_2$ in Problem~\eqref{eqn:mio} based on validation tuning, we end up with 
$\lambda_2>0$ (strictly away from zero).  
To this end, we consider the same setup as the uniform sparsity of Section~\ref{mediumscale} with $k=10$. Let $\hat{\B\Theta}(\lambda_0,\lambda_2)$ be a solution to~\eqref{eqn:mio} for regularization coefficients $\lambda_0,\lambda_2$, available from our approximate solver. We define $\text{CV}(\lambda_0,\lambda_2)$ as the pseudo-likelihood based validation loss of $\hat{\B\Theta}(\lambda_0,\lambda_2)$ computed on a held-out validation set (this is a proxy for the test error of the estimator). That is,
$$\text{CV}(\lambda_0,\lambda_2)=\sum_{i=1}^p\bigl(-\log(\hat\theta_{ii}(\lambda_0,\lambda_2))+\frac1{\hat\theta_{ii}(\lambda_0,\lambda_2)}\norm{\btX_{\text{val}}\hat{\bm\theta_{i}}(\lambda_0,\lambda_2)}^2\bigr)$$
where $\B X_{\text{val}}$ is the validation data. We define $\lambda_0^*,\lambda_2^*$ as the regularization parameters that minimize the validation loss
$$(\lambda_0^*,\lambda_2^*) \in \arg\min_{\lambda_0,\lambda_2}\text{CV}(\lambda_0,\lambda_2).$$
We study the behavior of $\text{CV}(\lambda_0^*,\lambda_2)$ for different values of $\lambda_2$. We plot this quantity for two values of $n \in \{50, 100\}$ in Figure~\ref{fig:lambda}. As we see, the validation loss is minimized for nonzero values of $\lambda_2$. Interestingly, when $n$ is smaller, a larger value of $\lambda_2$ seems to be helpful suggesting the necessity of higher shrinkage.

\begin{figure}[ht]
     \centering
     \begin{tabular}{cc}
\small $n=50$ & \small $n=100$  \\
     \includegraphics[width=0.3\linewidth,trim =.8cm 0cm .8cm 0cm, clip = true]{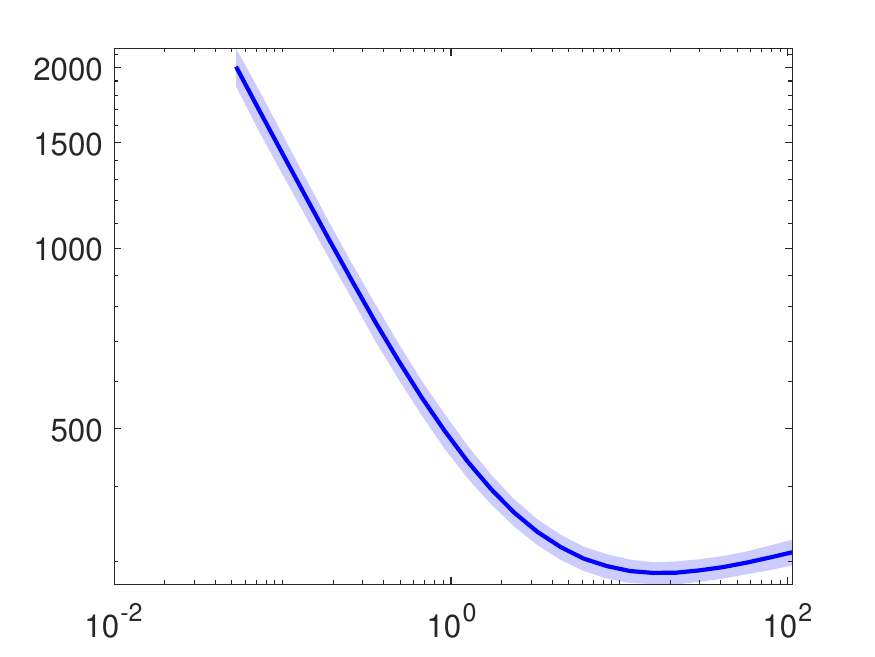}& 
     \includegraphics[width=0.3\linewidth,trim =.8cm 0cm .8cm 0cm, clip = true]{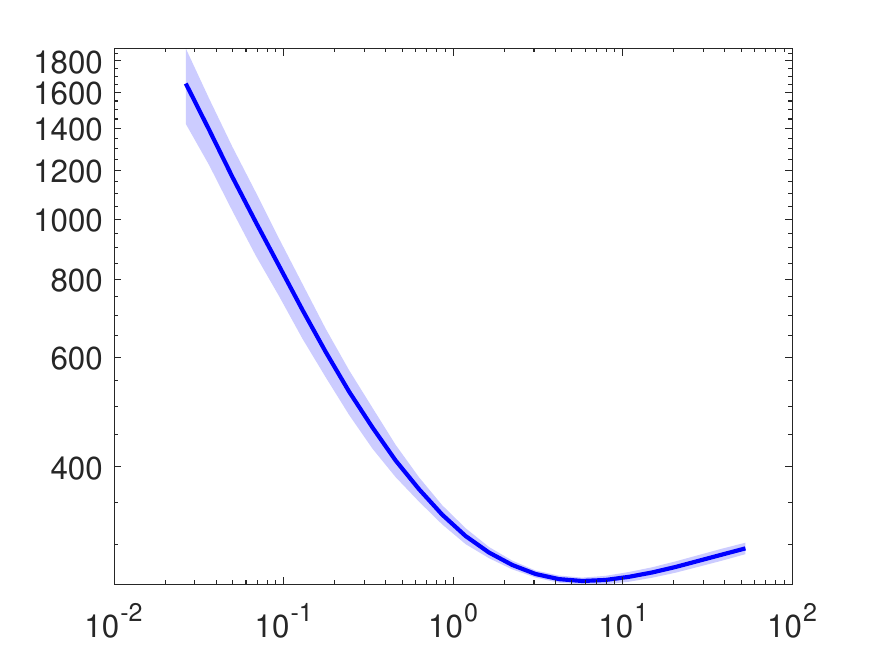} \\
     $\lambda_2$ & $\lambda_2$
\end{tabular}
        \caption{The validation loss for different values of $\lambda_2$, when $\lambda_0$ is fixed to its optimal value ($\text{CV}(\lambda_0^*,\lambda_2)$). The details are discussed in Appendix~\ref{app:l2reg}.}
        \label{fig:lambda}
\end{figure}

Next, we study the statistical effect of varying the Big-M value in Problem~\eqref{eqn:mio}. We follow an experimental setup similar to the one discussed above with $n=100$. In particular, we fix a value of the Big-M parameter $M$ in Problem~\eqref{eqn:mio}, and use our approximate solver to compute a path of solutions over a $4\times 4$ grid of $(\lambda_0,\lambda_2)$ values (we refer to Appendix~\ref{app:num-details} for more details). We let $\hat{\B\Theta}(M)$ be a solution from the computed path that has the lowest validation loss (across the tuning parameters $(\lambda_0,\lambda_2)$). We define $\widehat{\text{CV}}(M)$ to be the validation loss corresponding to $\hat{\B\Theta}(M)$: 
$$\widehat{\text{CV}}(M)=\sum_{i=1}^p\bigl(-\log(\hat\theta_{ii}(M))+\frac1{\hat\theta_{ii}(M)}\norm{\btX_{\text{val}}\hat{\bm\theta_{i}}(M)}^2\bigr).$$
We plot $\widehat{\text{CV}}(M)$ for different values of $M$ and $n=1000$ in Figure~\ref{fig:m}. We see that choosing $M$ to be too small is detrimental to the validation loss, while choosing $M$ to be too large does not improve the validation loss (over a moderate value of $M$). 

\begin{figure}[ht]

     \centering
     \begin{tabular}{cc}
     &  \\
  \rotatebox{90}{~~~~~~~~~~~~~~~~~~~~$\widehat{\text{CV}}(M)$} &   \includegraphics[width=0.5\linewidth,trim =.1cm 0cm .1cm 0cm, clip = true]{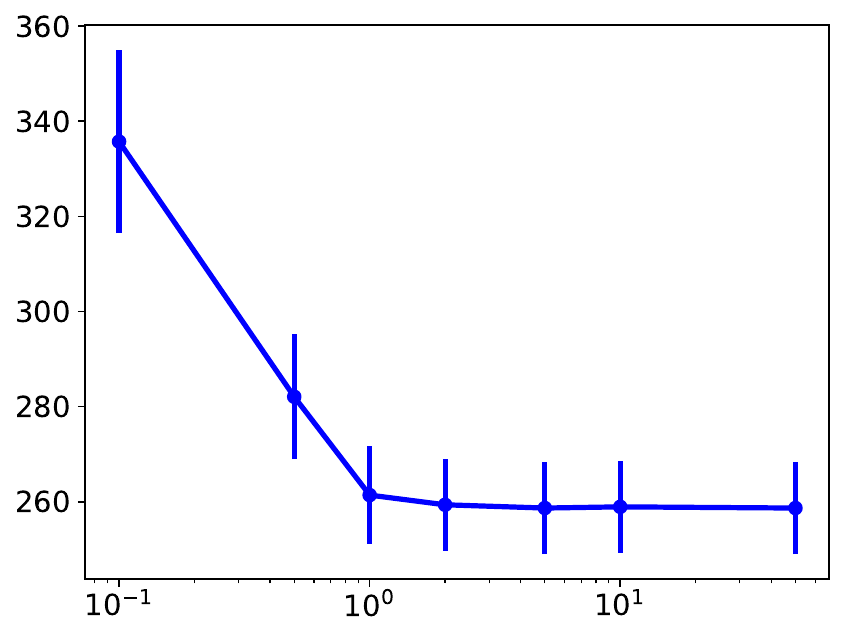} \\
   &  $M/\max_{i,j}|\theta^*_{ij}|$
\end{tabular}
        \caption{The validation loss of $\hat{\B\Theta}(M)$, denoted as $\widehat{\text{CV}}(M)$,
        for different values of $M$. The details are discussed in Appendix~\ref{app:l2reg}.}
        \label{fig:m}
\end{figure}

\subsection{Additional experiments on portfolio optimization}\label{appendix:portfolio}
We perform additional numerical experiments on the example from Section~\ref{financesec}. We consider the top-100 stocks from Section~\ref{financesec}. We split the 
dataset into two parts (based on even and odd time indices) called $D_1$ and $D_2$ (1250 samples per split). We use the first half of the data $D_1$ to compute solutions to the sparse GGM problem for a range of hyper-parameters. 
We consider 100 distinct hyper-parameters (corresponding to a $10\times 10$ grid over $\lambda_0,\lambda_2$) and obtain the estimated precision matrices $\{\hat{\B\Theta}_1,\cdots,\hat{\B\Theta}_{100}\}$ using our estimator \ourmethodnobnb.  
For each estimated precision matrix $\hat{\B\Theta}_i$, we solve the portfolio optimization problem~\eqref{portfolio} using $\B\Sigma_X = \hat{\B\Theta}_i^{-1}$ and the second half of the data (that is, $D_2$) in the minimum return constraint in~\eqref{portfolio}. Similarly, for \texttt{GLASSO}, \texttt{CONCORD} and \texttt{CLIME} we consider 100 distinct values for their corresponding hyper-parameters and obtain 100 estimated precision matrices which we then use to solve problem~\eqref{portfolio}.
We consider various values of $\bar{r}$ ranging from zero to 140 for each solution, and calculate the risk and return of each solution based on the data $D_2$. 

We conduct two sets of experiments. 
First, we obtain the return and risk for each solution (for every hyperparameter and method). For every method, we then compute the Pareto frontier of return/risk over different values of $\bar{r}$. We also include the results from a baseline method that estimates $\B{\Sigma}_X$ in~\eqref{portfolio} with the sample covariance matrix. The results are shown in Figure~\ref{fig:app-portfolio} [Left Panel]. We see that overall \ourmethodnobnb~and \texttt{GLASSO} appear to have the best performance---both improving upon the baseline.

Next we consider precision matrix estimators with a budget on their sparsity levels as such estimators might be more desirable in practice (due to their sparsity properties). In particular, among all solutions along the path of hyper-parameters, we only keep the ones that have at most $k$ nonzeros $\|\hat{\B\Theta}\|_0\leq k$. We then calculate the Pareto frontier as mentioned above. We show results for $k=2000$ in Figure~\ref{fig:app-portfolio} [Middle Panel] and for $k=1000$ in Figure~\ref{fig:app-portfolio} [Right Panel]. We observe that \ourmethodnobnb~seems to have the best performance under the minimum sparsity constraint for a wide range of return values. Interestingly, enforcing a sparsity level of $k=2000$ has a minimal effect on \ourmethodnobnb's Pareto frontier. This further demonstrates the potential usefulness of our $\ell_0$-based approach as a method for obtaining sparse precision matrices.

\begin{figure}[ht]
     \centering
     \begin{tabular}{cccc}
 & \small ~~~No Minimum Sparsity & \small   ~~$\|\hat{\B\Theta}\|_0\leq 2000$  & \small ~~$\|\hat{\B\Theta}\|_0\leq 1000$ \\
    \rotatebox{90}{~~~~~~~~~~~~~Risk} & \includegraphics[width=0.28\linewidth,trim =.1cm 0cm .1cm 0cm, clip = true] {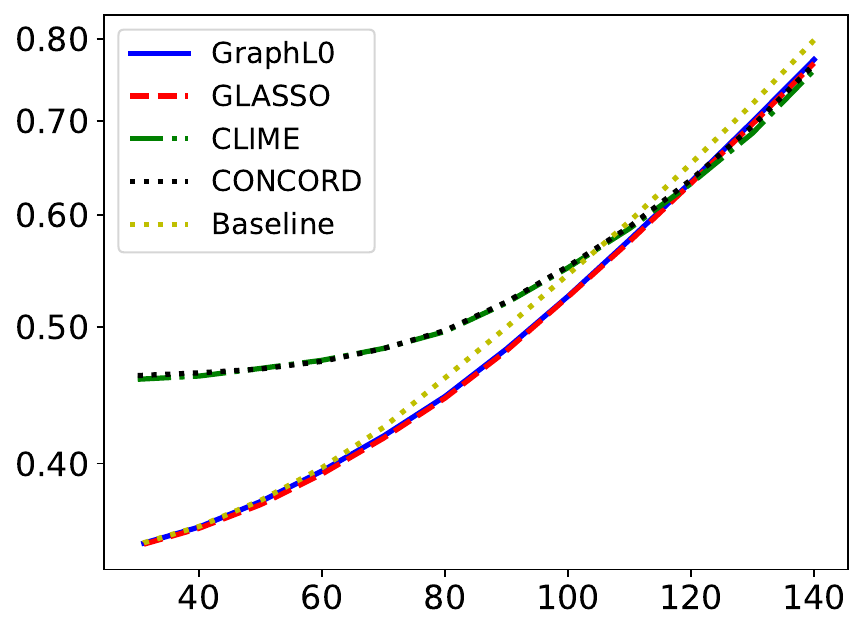}& \includegraphics[width=0.28\linewidth,trim =.1cm 0cm .1cm 0cm, clip = true] {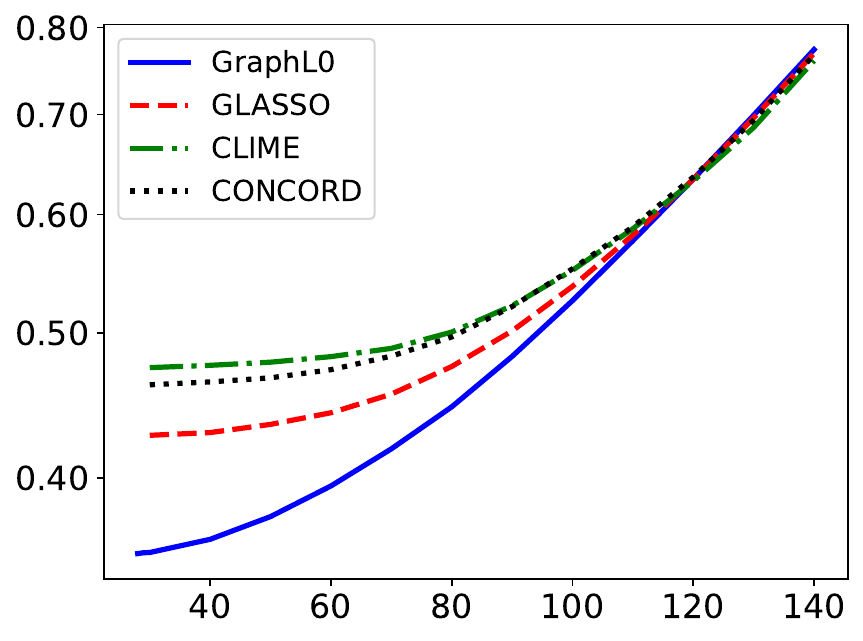} & \includegraphics[width=0.28\linewidth,trim =.1cm 0cm .1cm 0cm, clip = true] {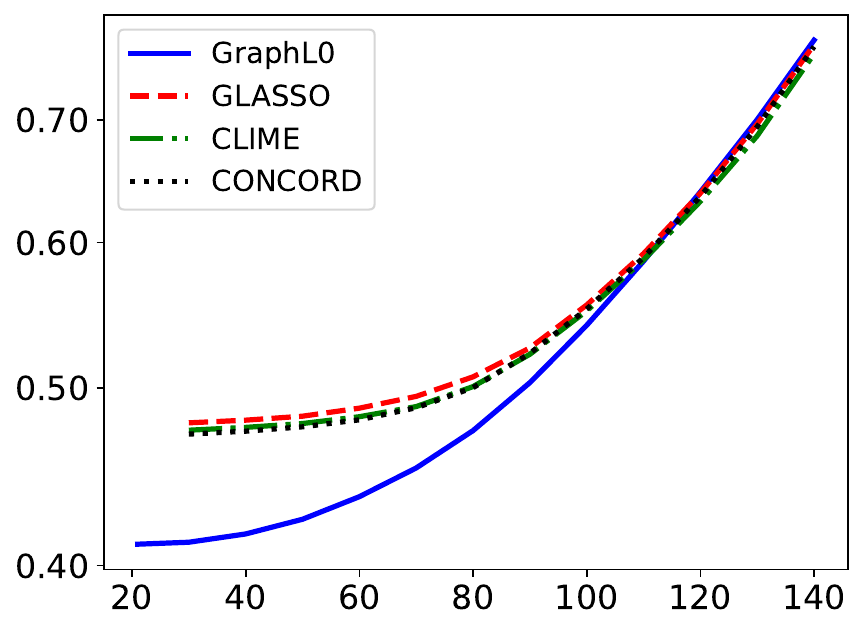} \\
    & ~Return & ~~~Return & ~~~Return

\end{tabular}
        \caption{ Comparison of different methods for the portfolio optimization downstream application, when we require a minimum return. The details are discussed in Appendix~\ref{appendix:portfolio}. [Left Panel]: We do not enforce any minimum sparsity requirement for different methods. [Middle Panel]: We only consider precision matrix estimators with at most 2000 nonzeros. [Right Panel]: We only consider precision matrix estimators with at most 1000 nonzeros.}
        \label{fig:app-portfolio}
\end{figure}

\section{Discussion on Theory}\label{sec:theory-discussion}
In this section, we investigate our assumptions from Section~\ref{sec:stat} in more details. In particular, we present an example which satisfies Assumptions~\ref{assum1-1} to~\ref{assum1-5}, and Assumptions~\ref{sigmaboundedassum} to~\ref{kappaassum}.

For some fixed $r\geq 1$, let $s_1,\cdots, s_r\in(0,1/2]$. Moreover, suppose $\B{u}_1\cdots,\B{u}_r\in\R^p$ are given such that:
\begin{enumerate}
    \item They have unit norm, for all $i\in[r]$ we have $\|\B u_i\|_2=1$.
    \item They are sparse, for $i\in[r]$ we have that $|\text{Supp}(\B u_i)| = k$ where $\text{Supp}(\B u)=\{i\in[p]:u_i\neq 0\}$.
        \item They have non-overlapping supports, for all $i\neq j\in[r]$, we have that $\text{Supp}(\B u_i)\cap \text{Supp}(\B u_j)=\emptyset$. 
    \item There exist $0<c_1<c_2$  such that for any $i\in[r]$, and any $j\in\text{Supp}(\B u_i)$, we have that $c_1/\sqrt{k}\leq|u_{ij}|\leq c_2/\sqrt{k}$.
\end{enumerate}
We define the underlying precision matrix as
\begin{equation}\label{appendix-spiked-theta}
    \B\Theta^* = \B I_p + \sum_{i=1}^r s_i \B u_i\B u_i^\top.
\end{equation}
By using Sherman-Woodbury matrix inversion formula, we can see that the underlying covariance matrix is given as
\begin{equation}
    \B\Sigma^* = (\B\Theta^*)^{-1} =  \B I_p - \sum_{i=1}^r \frac{s_i \B u_i\B u_i^\top}{1+s_i}.
\end{equation}
We note that $\B\Theta^*$ defined here encompasses a wide range of statistical models. As an example, if $\text{Supp}(\B u_1)=\{1,\cdots, k\}, \text{Supp}(\B u_2)=\{k+1,\cdots, 2k\},\cdots$, then $\B\Theta^*$ has a block diagonal (banded) sparsity structure, and if the supports of $\B u_i$ are chosen uniformly at random, then the sparsity pattern of $\B\Theta^*$ will be random and uniform as well. Moreover, the nonzero coordinates of $\B u_i$ determine what values the nonzero coordinates of $\B\Theta^*$ takes. For example, if all nonzero coordinates of $\B u_i$'s are equal to $1/\sqrt{k}$, the (off-diagonal) nonzeros of $\B\Theta^*$ are equal to $1/k$. 

For $\B\Theta^*$ given in~\eqref{appendix-spiked-theta}, we have that for $a\neq b\in[p]$: 
\begin{equation}\label{app-d-sigma-beta}
    (\sigma_a^*)^2 = \frac{1}{\theta^*_{aa}} =\frac{1}{1+\sum_{i=1}^r s_i u_{ia}^2},~~~ \beta_{ab}^*=-\frac{\theta^*_{ba}}{\theta^*_{bb}}= -\frac{\sum_{i=1}^r s_iu_{ia}u_{ib}}{1+\sum_{i=1}^r s_i u_{ib}^2}.
\end{equation}

In what follows we verify that the assumptions in our technical results hold.

\customparagraph{Verifying~\ref{assum1-1}} From~\eqref{app-d-sigma-beta}, for $a\in[p]$ we have that
$$\sqrt{\frac{2}{3}}\leq \sigma_a^*\leq 1$$
as $s_i\in(0,1/2]$ and $u_{ia}$ can be nonzero for at most one value of $i\in[r]$. Therefore, we can take $l_\sigma=\sqrt{2/3}$ and $u_\sigma=1$, satisfying~\ref{assum1-1}.

\customparagraph{Verifying~\ref{assum1-2}} From~\eqref{app-d-sigma-beta}, we have that for $a,b\in[p]$, 
$|\beta^*_{ab}|\leq 1$
as $s_i\in(0,1/2]$.

\customparagraph{Verifying~\ref{assum1-3}}
We have
$$l_\sigma^2 = \frac{2}{3}> \frac{6}{25}+\frac{2}{5}=\frac{6}{25}u_\sigma^4+\frac{2}{5}u_\sigma^2 $$
which verifies~\ref{assum1-3}.

\customparagraph{Verifying~\ref{assum1-4}}
If $a\in[p]$ is such that $a\notin \text{Supp}(\B u_i)$ for all $i\in[r]$, then we have $u_{ia}=0$ for all $i\in[r]$ and therefore, $\{b\in[p]:b\neq a, \beta_{ba}^*\neq 0\}=\emptyset$ from~\eqref{app-d-sigma-beta}. If $a\in\text{Supp}(\B u_i)$ for some $i\in[r]$, then we have that 
$$|\{b\in[p]:b\neq a, \beta_{ba}^*\neq 0\}|\leq |\text{Supp}(\B u_i)|=k$$
as $a\notin \text{Supp}(\B u_j)$ if $j\neq i$.

\customparagraph{Verifying~\ref{assum1-5}} Note that as $\B u_i$'s have non-overlapping supports, they are orthonormal. Therefore, the eigenvalues of $\B\Sigma^*$ are given as $1-s_i/(1+s_i)$ for $i\in[r]$ and 1, which can be lower bounded by $1/2$. Therefore, we can take $\kappa^2=1/2$ in~\ref{assum1-5}.

\customparagraph{Verifying~\ref{sigmaboundedassum}} As we discussed for~\ref{assum1-1}, we can take $l_\sigma=\sqrt{2/3}>1/5 = u_\sigma/5$.

\customparagraph{Verifying~\ref{beta-bound}} Suppose $a,b\in\text{Supp}(\B u_i)$ for some $i\in[r]$. Then, from~\eqref{app-d-sigma-beta} and the fact that $\B u_i$'s have non-overlapping supports,
$$|\beta^*_{ab}| =\frac{ s_i|u_{ia}||u_{ib}|}{1+ s_i u_{ib}^2}\leq |u_{ia}||u_{ib}| \leq \frac{c_2^2}{k} \leq \frac{1}{\sqrt{k}}$$
as long as $c_2^2\leq \sqrt{k}$.

\customparagraph{Verifying~\ref{assum2-max}} For $a\in[p]$, we have either that $(\B\Sigma^*)_{aa}=1$ or $(\B\Sigma^*)_{aa}=1-s_iu_{ia}^2/(1+s_i)$ for some $i\in[r]$. Then, $|(\B\Sigma^*)_{aa}|\leq 1$. For $a\neq b\in[p]$, we either have $(\B\Sigma^*)_{ab}=0$ or $|(\B\Sigma^*)_{ab}|=s_i|u_{ia}u_{ib}|/(1+s_i)\leq 1$ for some $i\in[r]$. Therefore, for all $a,b\in[p]$, $|(\B\Sigma^*)_{ab}|\leq 1$. Moreover, from the discussion for~\ref{assum1-1},
$$\frac{(\B{\Sigma}^*)_{aa}}{(\sigma_a^*)^2}\leq \frac{1}{2/3}<\frac{400}{7}.$$

\customparagraph{Verifying~\ref{betaminassum}} From~\eqref{app-d-sigma-beta}, if $\beta_{ab}^*\neq 0$ for some $a\neq b\in[p]$, we have that for some $i\in[r]$
$$|\beta_{ab}^*|=\frac{ s_i|u_{ia}u_{ib}|}{1+ s_i u_{ib}^2}\geq \frac{c_1^2 s_i}{2k}\geq \beta_{\min}=\sqrt{\frac{\eta\log p}{n}}$$
if $n$ is sufficiently large.

\customparagraph{Verifying~\ref{assum2degree}} This is the same as~\ref{assum1-4}.

\customparagraph{Verifying~\ref{kappaassum}} As we discussed for~\ref{assum1-5}, the eigenvalues of $\B\Sigma^*$ are between 1 and 1/2, verifying this assumption.

Our discussion above shows that the precision matrix model discussed in this section satisfies all our theoretical assumptions, and therefore, our statistical guarantees from Section~\ref{sec:stat} hold for this model. As we discussed, by choosing the sparsity pattern of $\B u_i$'s, and the values of their nonzero coordinates, we can simulate a wide range of statistical setups using the model we introduced here. This further shows the usefulness of our theoretical guarantees in practice.

\end{document}